\documentclass[11pt]{article}

\usepackage[final]{acl}

\usepackage{times}
\usepackage{latexsym}
\usepackage{comment}
\usepackage{tabularx}
\usepackage{setspace}
\usepackage{titletoc} 
\usepackage{balance}

\usepackage[T1]{fontenc}
\usepackage[utf8]{inputenc}

\usepackage{microtype}

\usepackage{inconsolata}

\usepackage[dvipsnames]{xcolor}
\usepackage{graphicx}
\usepackage{subcaption}
\usepackage{multirow}
\usepackage{booktabs}
\usepackage{colortbl}
\usepackage[ruled,vlined,linesnumbered]{algorithm2e}
\usepackage{placeins}
\usepackage{xspace}
\usepackage{amsmath}
\usepackage{amssymb}
\usepackage{mathtools}
\usepackage{amsthm}
\usepackage{cleveref}
\usepackage{contour}
\usepackage{enumitem}
\usepackage{adjustbox}
\usepackage[most]{tcolorbox}
\usepackage{circledsteps}
\usepackage[normalem]{ulem}

\newcommand{\ulc}[2]{%
  {\color{#2}\uline{\phantom{#1}}}%
  \llap{\contour{white}{#1}}%
}

\definecolor{turchese}{HTML}{1f77b4}
\definecolor{foresta}{HTML}{2ca02c} 
\definecolor{g1color}{rgb}{0.68, 0.87, 1}
\definecolor{g2color}{RGB}{232, 156, 111}
\definecolor{g3color}{RGB}{213, 153, 217}
\newcommand{\Guno}{\Circled[outer color=g1color,fill color=g1color]{\textbf{G1}}\xspace}
\newcommand{\Gdue}{\Circled[outer color=g2color,fill color=g2color]{\textbf{G2}}\xspace}
\newcommand{\Gtre}{\Circled[outer color=g3color,fill color=g3color]{\textbf{G3}}\xspace}

\definecolor{deltapos}{RGB}{210,240,210} % verde pastello
\definecolor{deltaneg}{RGB}{245,210,210}

\newcommand{\posdelta}[1]{\cellcolor{deltapos}{#1}}
\newcommand{\negdelta}[1]{\cellcolor{deltaneg}{#1}}

\newtcbtheorem[auto counter, number within=section]{samplebox}{Calibration Samples}%
{
    colback=orange!5!white,
    colframe=orange!75!black,
    enhanced,
    boxrule=0.5pt,
    top=2mm, bottom=2mm,
    fonttitle=\bfseries
}{box}
\crefname{samplebox}{box}{boxes}
\Crefname{samplebox}{Box}{Boxes}

\theoremstyle{plain}
\newtheorem{theorem}{Theorem}[section]

\newtheorem{lemma}[theorem]{Lemma}

\theoremstyle{definition}

\theoremstyle{remark}

\newcommand{\Let}[2]{#1 $\gets$ #2}

\renewcommand{\epsilon}{\varepsilon}

\newcommand{\methodname}{{\texttt{\textbf{ZipCal}}}\xspace}

\title{Frequency Matters: Fast Model-Agnostic Data Curation for Pruning and Quantization}

\author{
Francesco Pio Monaco \quad
Elia Cunegatti \quad
Flavio Vella \quad
Giovanni Iacca \\
University of Trento \\
\texttt{\{francescopio.monaco, elia.cunegatti, flavio.vella, giovanni.iacca\}@unitn.it}
}
\begin{document}
\maketitle
\begin{abstract}
Post-training model compression is essential for enhancing the portability of Large Language Models (LLMs) while preserving their performance. While several compression approaches have been proposed, less emphasis has been placed on selecting the most suitable set of data (the so-called \emph{calibration data}) for finding the compressed model configuration. The choice of calibration data is a critical step in preserving model capabilities both intra- and inter-tasks. In this work, we address the challenge of identifying high-performance calibration sets for both pruning and quantization by analyzing intrinsic data properties rather than model-specific signals. We introduce \methodname, a model-agnostic data curation strategy that maximizes lexical diversity based on Zipfian power laws. Experiments demonstrate that our method outperforms standard uniform random sampling across various pruning benchmarks. Notably, it also performs on par, in terms of downstream performance, with a state-of-the-art method that relies on model perplexity. The latter becomes prohibitively expensive for large-scale models and datasets, while \methodname is on average $\sim$240$\times$ faster due to its tractable linear complexity\footnote{We make the code and the experiments available at 
\url{https://github.com/FrancescoMonaco/ZipCal}.}.
%\url{https://anonymous.4open.science/r/zipcal-71CD}.}.
\end{abstract}

\section{Introduction}

In recent years, considerable effort has been put into alleviating the significant computational requirements of Large Language Models (LLMs) \cite{xia2024sheared,muralidharan2024compact,lee2025littlebit}. To facilitate the deployment of such models, researchers have developed various post-training compression techniques that reduce memory and compute overhead while striving to preserve the original model's capabilities.
In this work, we mainly focus on two such compression techniques, which are pruning and quantization. The former compresses the model by removing part of its parameters, with earlier approaches requiring expensive retraining \cite{han2015learning,frankle2018the}, and more modern approaches simply removing weights in a gradient-free manner \cite{frantar_sparsegpt_2023,sun_simple_2024}, or enforcing hardware-friendly sparsity patterns \cite{sun2026learning} without the need for parameter updates \cite{zhang2024dynamic,cunegatti2025zerothorder}. Quantization, on the other hand, compresses the model by reducing the numerical precision of weights and activations \cite{frantar_gptq_2023,lin_awq_2024,huang2025slim}. The latest approaches are based on ternary quantized models \cite{wang_bitnet_2023,ma_era_2024} or efficient binary MatMul techniques \cite{dehghankar_efficient_2025}. 

Beyond the choice of the compression approach, recent works \cite{bandari_is_2024,williams_impact_2024,oh_beyond_2025} explored how the selection of the data used for gathering model statistics during the compression process, called \textit{calibration data}, can influence the process and, as a result, the final compressed model capabilities. 
While the majority of compression techniques \cite{frantar_sparsegpt_2023,lin_awq_2024} rely on general-purpose datasets, such as C4 \cite{raffel_exploring_2020} or Pile \cite{gao_pile_2020}, it has been investigated that domain-specific calibration has significant impact on the downstream performances of the model  \cite{bandari_is_2024,zhang_pruning_2024,williams_compressing_2026}. However, a compressed model can be deployed across multiple, and at times unforeseen, downstream tasks, raising the question of whether a calibration strategy can be optimal on average across a spread of domains rather than tailored to any single one. A question that, to our knowledge, remains unaddressed. We investigate this via a Multi-Domain extension of our method that aggregates candidate calibration data across domains without requiring prior knowledge of the specific deployment task. We further explore an orthogonal and underexplored axis of calibration data selection: language. While prior work considers task or data alignment between calibration and deployment data, whether alignment in language between calibration and downstream data is similarly necessary has not been studied.

To analyze the impact of selected data for compression, we leverage several studies on statistical linguistics that show that human languages show a structure in the frequency distribution of words \cite{zipf_frontmatter_2013,piantadosi_zipfs_2014}.
Specifically, we investigate whether this statistical observation can be exploited for data curation of calibration sets for model compression. We hypothesize that a set of data samples maximizing lexical diversity, capturing the sparse \textit{tail} of the Zipfian distribution, can provide much richer and more representative information for compression algorithms. By focusing on these intrinsic linguistic properties, we propose a sampling strategy that identifies high-utility, \emph{model-agnostic} data with negligible computational overhead. Our results show that this linguistically informed approach performs on par with more expensive, model-dependent curation methods, offering a scalable and robust solution for billion-parameter model compression that generalizes across diverse downstream tasks.

\paragraph{Data Curation Goals}
An ideal technique for data curation must be \emph{scalable} \Guno, i.e., capable of processing massive corpora with minimal computational overhead; \emph{model-agnostic} \Gdue, i.e., capable of identifying the most informative examples from a corpus without relying on expensive model passes; and address \emph{{inter-domain generalization}} \Gtre, i.e., being capable of synthesizing both Single-Domain and Multi-Domain corpora settings by design.
Existing data curation techniques only fulfill a subset of these goals, and, to our best knowledge, none have yet explored \Gtre.

\paragraph{Core Contributions}
We propose a sampling strategy rooted in Zipfian statistics that identifies high-utility calibration samples by maximizing lexical diversity. This approach sidesteps the need for model-dependent metrics (e.g., perplexity or gradient information), achieving goals \Guno and \Gdue. Moreover, we provide a framework for extracting representative samples from Multi-Domain datasets, ensuring the calibration set is balanced, meeting goal \Gtre. We empirically evaluate our proposed method, called \methodname, against SoTA techniques for data curation for compression algorithms. We conduct a comprehensive analysis against baselines based on alternative data properties, showing that capturing lexical diversity is a stable and scalable proxy for data curation.
%%%
%%% RELATED WORK %%%
%%%
\section{Related Work}
\paragraph{Model Compression}
Pruning methods aim to remove parameters to reduce the size of the models. 
The first approaches to pruning involved estimating the contribution of each neuron to the final loss using the magnitude of the weights and gradient information \cite{molchanov_importance_2019,ma2023llm}.
More recently, different gradient-free no-retraining pruning modalities have been explored. Unstructured approaches zero-out individual weights \cite{frantar_sparsegpt_2023, sun_simple_2024,yang_wanda_2025}, while semi-structured methods enforce specific sparsity patterns to ensure hardware compatibility \cite{zhou2021learning}. Structured approaches, instead, remove entire architectural components, such as attention heads or layer rows/columns \cite{ashkboos2024slicegpt, sandri_2ssp_2025,guo_slimllm_2025}.

On the other hand, quantization methods reduce the numerical precision of weights and/or activations to reduce the memory footprint and accelerate inference. These methods range from Round-To-Nearest (RTN) \cite{nagel_up_2020} to more sophisticated error-minimization strategies \cite{nahshan_loss_2021,chen_efficientqat_2025}. Other works focus on reducing degradation in extreme quantization scenarios \cite{dettmers_llmint8_nodate}, employ optimized kernels like LUT-GEMM \cite{park_lut-gemm_2024}, or activation-aware scaling \cite{lin_awq_2024}.
\paragraph{Calibration Data}
Almost every compression algorithm relies on a small, representative \textit{calibration set} to estimate the information flow through the network. These statistics guide the compression process, determining quantization thresholds \cite{lin_awq_2024} or pruning scores \cite{sun_simple_2024}.
Recent studies have shown that the choice of the calibration source significantly impacts the performance of the compressed model \cite{williams_impact_2024}, revealing that general-purpose corpora like C4 \cite{raffel_exploring_2020} are not the optimal choice for downstream tasks \cite{bandari_is_2024}. This suggests that calibration data should mirror the target domain to prevent activation distribution shifts, which can lead to suboptimal quantization thresholds or pruning masks.

Researchers have proposed more sophisticated calibration data curation strategies. Marion et al. \cite{marion_when_2023} demonstrate that perplexity serves as a robust metric to rank and select the most impactful samples for pruning, essentially using the model's own likelihood as a proxy for quality. Extending this logic, COLA \cite{he_preserving_2025} introduces a hybrid approach that selects samples by balancing the magnitude of model activations with intrinsic data statistics. While effective, these methods are inherently model-dependent and computationally intensive. Moreover, these techniques are typically evaluated on single-source datasets and do not address the challenge of heterogeneous composability \Gtre, while it would be desirable to have stable, cross-task performance post-compression.
\section{Zipf Sampling}

\begin{figure*}
\centering
    \includegraphics[width=\linewidth]{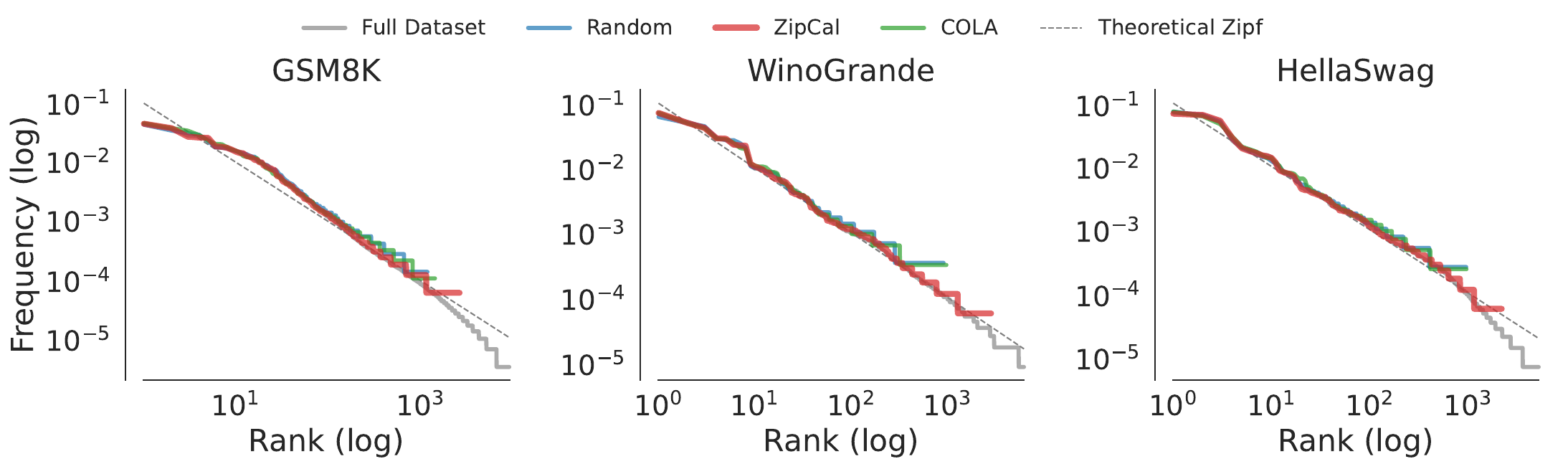}
    \caption{Token frequency distribution of the original datasets and the random, COLA, and \methodname calibration sets.}
    \label{fig:zipf_tokens}
\end{figure*}

%\subsection{Sampling Data}
Natural languages are characterized by a Zipfian distribution \cite{piantadosi_zipfs_2014}, where a small number of words appear with high frequency, while most of the vocabulary resides in an increasingly sparse long tail. Not acknowledging this sparsity implies potentially omitting the rare tokens and diverse semantic contexts that trigger critical activation outliers in LLMs.

\subsection{Single-Domain Sampling}
\label{sec:single_task}
We propose to sample calibration data from a dataset by maximizing the \textit{lexical diversity} of the calibration set within a constrained number of samples.
Specifically, we apply a sanitization pass to the dataset where tokens are converted to lowercase to form a vocabulary $V$ and special tokens (i.e., EOS) are removed. We then employ a randomized greedy selection heuristic to iteratively populate the calibration set. In each iteration, we select a sample $s$ from a candidate pool $P$ that maximizes the gain of the sample's vocabulary:
\begin{equation}
s^* = \arg\max_{s \in P} | V(s) \setminus \mathcal{V}_{covered} |
\end{equation}
where $\mathcal{V}_{covered}$ is the vocabulary of sanitized tokens already present in previously selected samples. In the event of a tie, we prioritize the sample with the highest total number of unique tokens to maximize information density. This approach ensures that the resulting calibration data provides a high-fidelity approximation of the full dataset's vocabulary manifold. We present the procedure, named \methodname, in \Cref{algo:sampling}.

\begin{lemma}
\label{lemma:sampling}
    When \methodname is used to extract a calibration set of $k$ samples on dataset $\mathcal{D}$ of $n$ elements, it completes the procedure in $O(nk)$ time.
\end{lemma}
%\subsubsection{Time Complexity Analysis}
\begin{proof}
The selection process, lines~\ref{ln:selection} to~\ref{ln:selectionend} of \Cref{algo:sampling}, consists of $k$ iterations to construct the calibration set. At any iteration $i\in(1,k)$, we have to carry out $T_i=N-i+1$ evaluations. To find out the total number of evaluations required to construct a calibration set of $k$ samples, we have:
\begin{align}
    T_{tot} = \sum_{i=1}^k T_i= k\cdot n - \frac{k(k+1)}{2}+k
\end{align}
During each evaluation, we compute the set difference between the candidate's vocabulary $V_s$ and the cumulative vocabulary $V_{covered}$. This cost scales with the number of unique elements in a sample, which is bounded by the context window size $w$. Since $w$ is fixed a priori, its contribution to the asymptotic complexity is constant. Consequently, summing these contributions, the overall complexity is $O(kn)$.
\end{proof}

\begin{algorithm}[t]
    \KwIn{Dataset $D$; Number of samples $k$}
    \KwOut{Set of calibration samples $S$}
    \caption{\methodname} \label{algo:sampling}
    \Let{$S$,$V_{covered}$}{$\emptyset$}\;

    \tcp{Precalculate the full vocabulary}
    \ForEach{sample $s \in \mathcal{D}$}{
        $V_s \gets \{ \text{sanitize}(t) \mid t \in s, t \notin \text{SpecialTokens} \}$ \;
    }
    
    \For{$i = 1$ \KwTo $k$ \label{ln:selection}}{
        \tcp{Select sample with maximum marginal vocabulary gain}
        ${s^*} \gets \arg\max_{s \in \mathcal{D} \setminus S} | V_s \setminus \mathcal{V}_{covered} |$ \;
        
        $S \gets S \cup \{s^*\}$ \;
        $V_{covered} \gets V_{covered} \cup V_{s^*}$ \label{ln:selectionend}\;
    }

    \Return $S$ \;
\end{algorithm}

\subsection{Multi-Domain Sampling}
\label{sec:multi_task}
When calibrating models for general-purpose use or Multi-Domain applications, a single source of data is often insufficient. To address goal \Gtre, we extend our approach to support heterogeneous multi-domain settings. Simply concatenating datasets and applying \methodname over the joint dataset is suboptimal, as a single large or linguistically dense corpus might dominate the selection process. To address this, we propose a hierarchical selection strategy. First, we apply \methodname to each dataset $D_i \in \mathbb{D}$, to extract a local representative pool $P_i$ of size $k$. This ensures each domain's unique vocabulary is captured. Second, we consolidate these pools into a candidate set $\mathcal{P} = \bigcup P_i$ and apply a greedy $k$-centers selection algorithm. By representing each sample using a lightweight embedding, the $k$-centers objective selects a final set $S$ that maximizes the distance between samples, ensuring the calibration data is semantically spread across all provided domains.

\begin{lemma}
\label{lemma:multi}
The Multi-Domain \methodname procedure extracts $k$ samples from $m$ datasets, each of length $n_m$, in $O(mNk)$ time, where $N=\max_m(n_m)$.
\end{lemma}

\begin{table*}[t]
\centering
\resizebox{\textwidth}{!}{%
\begin{tabular}{l|l||c||ccccc|c||ccccc|c|c}
\toprule
& & & \multicolumn{6}{c||}{\textbf{COLA}} & \multicolumn{6}{c}{{\methodname}} \\
\cmidrule(lr){4-9} \cmidrule(lr){10-15}
&&&\multicolumn{5}{c}{\textbf{Calibration Category}} && \multicolumn{5}{c}{\textbf{Calibration Category}}&&
\\
\textbf{Model} & \textbf{Task} & \textbf{Dense} & \cellcolor{gray!7}LangMod & \cellcolor{blue!7}Math & \cellcolor{orange!7}CommQA & \cellcolor{purple!7}NLI & \cellcolor{cyan!7}KnowTran & \textbf{Mean} & \cellcolor{gray!7}LangMod & \cellcolor{blue!7}Math & \cellcolor{orange!7}CommQA & \cellcolor{purple!7}NLI & \cellcolor{cyan!7}KnowTran & \textbf{Mean} & \textbf{$\Delta$} \\
\midrule
& \cellcolor{blue!7}MMLU-M & 24.44 & 23.33 & \cellcolor{blue!7}\underline{24.44} & 23.01 & 21.85 & 23.33 & \textbf{23.19} & 23.01 & \cellcolor{blue!7}\underline{26.29} & \underline{21.24} & 22.02 & 23.01 & 23.11 & \negdelta{-0.08} \\
& \cellcolor{blue!7}GSM8k & 78.09 & \underline{76.55} & \cellcolor{blue!7}75.70 & 74.40 & 76.54 & 75.89 & 75.82 & \underline{76.88} & \cellcolor{blue!7}76.27 & 75.82 & 75.70 & 76.61 & \textbf{76.25} & \posdelta{\textbf{+0.43}} \\
& \cellcolor{orange!7}HellaSwag & 71.71 & 72.03 & \underline{72.09} & \cellcolor{orange!7}71.76 & 71.67 & 71.86 & \textbf{71.88} & 71.78 & \underline{71.90} & \cellcolor{orange!7}71.83 & 71.59 & 71.80 & 71.78 & \negdelta{-0.10} \\
& \cellcolor{orange!7}WinoGr. & 69.46 & \underline{68.93} & 68.51 & \cellcolor{orange!7}67.93 & 68.03 & 68.07 & \textbf{68.30} & 67.88 & 68.43 & \cellcolor{orange!7}\underline{68.82} & 68.15 & 68.03 & 68.26 & \negdelta{-0.04} \\
& \cellcolor{orange!7}OBQA & 47.80 & 46.67 & \underline{47.30} & \cellcolor{orange!7}46.67 & 46.20 & 47.10 & 46.79 & 47.90 & 47.50 & \cellcolor{orange!7}\underline{48.10} & 47.30 & 47.80 & \textbf{47.72} & \posdelta{\textbf{+0.93}} \\
& \cellcolor{orange!7}BoolQ & 84.46 & \underline{84.99} & 84.74 & \cellcolor{orange!7}84.41 & 84.85 & 84.91 & \textbf{84.78} & 84.50 & 84.45 & \cellcolor{orange!7}\underline{84.83} & 84.60 & 84.48 & 84.57 & \negdelta{-0.21} \\
& \cellcolor{purple!7}RTE & 74.73 & 72.32 & 72.20 & 71.96 & \cellcolor{purple!7}\underline{72.38} & \underline{72.38} & 72.25 & 72.20 & 72.38 & 72.74 & \cellcolor{purple!7}72.74 & \underline{73.10} & \textbf{72.64} & \posdelta{\textbf{+0.39}} \\
& \cellcolor{purple!7}ANLI & 58.40 & 56.37 & 59.00 & \underline{59.33} & \cellcolor{purple!7}57.10 & 57.80 & 57.92 & 61.50 & 58.50 & 59.80 & \cellcolor{purple!7}59.65 & \underline{61.95} & \textbf{60.28} & \posdelta{\textbf{+2.36}} \\
& \cellcolor{cyan!7}ARC-C & 51.71 & \underline{51.54} & 50.30 & 50.28 & 51.24 & \cellcolor{cyan!7}50.77 & 50.82 & 51.45 & 50.98 & \underline{51.58} & 51.32 & \cellcolor{cyan!7}51.11 & \textbf{51.29} & \posdelta{\textbf{+0.47}} \\
& \cellcolor{cyan!7}ARC-E & 73.86 & 73.27 & 73.46 & 73.13 & 73.21 & \cellcolor{cyan!7}\underline{73.74} & 73.36 & 73.86 & 74.05 & 73.74 & 73.84 & \cellcolor{cyan!7}\underline{73.93} & \textbf{73.88} & \posdelta{\textbf{+0.52}} \\
& \cellcolor{cyan!7}MMLU-K & 62.28 & 62.08 & \underline{62.51} & 62.04 & 61.94 &\cellcolor{cyan!7} 61.56 & 62.02 & 62.23 & 62.68 & \underline{62.70} & 62.44 &\cellcolor{cyan!7} 62.58 & \textbf{62.53} & \posdelta{\textbf{+0.51}} \\

\cmidrule{2-16}
& \textbf{Mean} & 63.36 & 62.55 & 62.75 & 62.27 & 62.27 & 62.49 & 62.47 & 63.02 & 63.04 & 62.84 & 62.67 & 63.13 & \textbf{62.94} & \posdelta{\textbf{+0.47}} \\
& & & & & & & & $\pm 0.19$ & & & & & & $\pm 0.17$\\
\cmidrule{2-16}
\rowcolor{yellow!30} \cellcolor{white}\multirow{-12}{*}{\rotatebox[origin=c]{90}{\texttt{Llama-3.1-8B-Instruct}}} &\textbf{Runtime} & & 5400s & 36s & 3240s & 2160s & 1380s & 2443s & 15.2s & 2.3s & 12.3s & 9.3s & 14.5s & \textbf{10.7s} & {\textbf{228$\times$}} \\
\midrule
& \cellcolor{blue!7}MMLU-M & 21.48 & \underline{24.44} & \cellcolor{blue!7}23.01 & 22.02 & 21.24 & 21.85 & \textbf{22.51} & 27.40 & \cellcolor{blue!7}24.44 & 26.29 & \underline{28.79} & 24.44 & \textbf{26.27} & \posdelta{\textbf{+3.76}} \\
& \cellcolor{blue!7}GSM8k & 75.44 & 74.05 & \cellcolor{blue!7}74.34 & \underline{74.98} & 74.30 & 73.88 & 74.31 & 74.32 & \cellcolor{blue!7}74.41 & \underline{75.09} & 74.26 & 74.37 & \textbf{74.49} & \posdelta{\textbf{+0.18}} \\
& \cellcolor{orange!7}HellaSwag & 67.24 & 67.11 & \underline{67.23} & \cellcolor{orange!7}67.08 & 67.06 & 67.22 & 67.14 & 67.27 & 67.15 & \cellcolor{orange!7}67.23 & \underline{67.34} & \underline{67.34} & \textbf{67.27} & \posdelta{\textbf{+0.13}} \\
& \cellcolor{orange!7}WinoGr. & 70.48 & 69.24 & 69.06 & \cellcolor{orange!7}68.82 & 69.26 & \underline{69.65} & \textbf{69.21} & 69.11 & 68.63 & \cellcolor{orange!7}69.18 & \underline{69.34} & 68.82 & 69.02 & \negdelta{-0.19} \\
& \cellcolor{orange!7}OBQA & 45.40 & 46.07 & 46.10 & \cellcolor{orange!7}\underline{46.40} & 45.70 & 46.30 & \textbf{46.11} & 45.33 & \underline{45.60} & \cellcolor{orange!7}45.40 & 45.20 & 45.40 & 45.39 & \negdelta{-0.72} \\
& \cellcolor{orange!7}BoolQ & 88.59 & \underline{88.92} & 88.62 & \cellcolor{orange!7}88.65 & 88.81 & 88.64 & 88.73 & 88.80 & 88.79 & \cellcolor{orange!7}88.79 & 88.72 & \underline{88.98} & \textbf{88.81} & \posdelta{\textbf{+0.08}} \\
& \cellcolor{purple!7}RTE & 78.34 & 78.22 & \underline{78.88} & 77.80 & \cellcolor{purple!7}78.34 & 78.70 & 78.39 & 78.46 & 78.34 & \underline{78.52} & \cellcolor{purple!7}\underline{78.52} & 78.34 & \textbf{78.44} & \posdelta{\textbf{+0.05}} \\
& \cellcolor{purple!7}ANLI & 72.80 & 72.37 & 73.30 & 73.05 & \cellcolor{purple!7}\underline{72.40} & 72.15 & 72.65 & \underline{73.27} & 73.15 & 72.60 & \cellcolor{purple!7}73.05 & 73.00 & \textbf{73.01} & \posdelta{\textbf{+0.36}} \\
& \cellcolor{cyan!7}ARC-C & 51.79 & \underline{52.13} & 51.37 & 51.71 & 51.62 & \cellcolor{cyan!7}51.83 & \textbf{51.73} & 51.48 & 51.54 & 51.54 & \underline{51.79} & \cellcolor{cyan!7}51.32 & 51.53 & \negdelta{-0.20} \\
& \cellcolor{cyan!7}ARC-E & 66.79 & \underline{67.51} & 67.09 & 66.94 & 67.28 & \cellcolor{cyan!7}67.23 & \textbf{67.21} & 66.72 & 66.75 & 66.90 & \underline{66.96} & \cellcolor{cyan!7}66.90 & 66.85 & \negdelta{-0.36} \\
& \cellcolor{cyan!7}MMLU-K & 33.83 & 35.90 & 33.53 & 32.94 & \underline{36.80} & \cellcolor{cyan!7}33.94 & 34.62 & \underline{36.95} & 36.07 & 36.42 & 35.88 &\cellcolor{cyan!7} 35.63 & \textbf{36.19} & \posdelta{\textbf{+1.57}} \\

\cmidrule{2-16}
& \textbf{Mean} & 61.11 & 61.45 & 61.14 & 60.94 & 61.16 & 61.04 & 61.15 & 61.74 & 61.35 & 61.63 & 61.80 & 61.32 & \textbf{61.57} & \posdelta{\textbf{+0.42}} \\
& & & & & & & & $\pm 0.81$ & & & & & & $\pm 0.82$\\
\cmidrule{2-16}
\rowcolor{yellow!30} \cellcolor{white}\multirow{-12}{*}{\rotatebox[origin=c]{90}{\texttt{gemma-2-9b-it}}} & \textbf{Runtime (sec)} & & 6231s & 149s & 3400s & 2671s & 1500s & 2790s & 15.2s & 2.3s & 12.3s & 9.3s & 14.5s & \textbf{10.7s} & \textbf{260$\times$} \\
\bottomrule
\end{tabular}
}
\caption{Comparison of COLA vs.\ \methodname calibration data under Wanda unstructured pruning at $25\%$ sparsity.}
\label{tab:pruning_single}
\end{table*}

\section{Experiments}

We present below the details of the experimental setup and the experimental results.

\subsection{Experimental Setup}

\paragraph{Post-training Compression}
We validate \methodname across a number of post-training compression techniques that rely on calibration data.
For pruning, we consider Wanda \cite{sun_simple_2024}, an unstructured approach which scores weight importance via the product of magnitudes and input activation norms $|W_{ij}| \cdot \|X_j\|_2$, and 2SSP \cite{sandri_2ssp_2025}, a two-stage framework for structured pruning that balances width and depth reduction. For quantization, we evaluate GPTQ \cite{frantar_gptq_2023}, which minimizes layer-wise reconstruction error $ \|\mathbf{W} X - \mathbf{\hat{W}} X\|_2^2$, and AWQ \cite{lin_awq_2024}, which preserves salient weights critical to model performance by scaling them according to activation magnitude $|\mathbf{X}|$.

\paragraph{Experimental Environment}
We use two LLMs to perform our evaluation of downstream tasks: \texttt{Llama-3.1-8B-Instruct} \cite{grattafiori_llama_2024} and \texttt{Gemma-2-9B-it} \cite{team_gemma_2024}. For language modeling evaluation, we use two base models, namely \texttt{Llama-3.1-8B} and \texttt{Gemma-2-9B} \footnote{Due to the number of combinations, executing the experiments over these models requires $\approx1200$ GPU hours.}. Unless otherwise specified, we set the context length to $w=2048$ and the number of calibration samples $k=128$.

\paragraph{Baselines}
We benchmark \methodname firstly against random sampling, the approach used by almost any compression algorithm \cite{ashkboos2024slicegpt,sun_simple_2024,lin_awq_2024}.
Then, we evaluate whether our lightweight, model-agnostic approach can match or exceed the performance of a computationally expensive, model-dependent technique. Hence, we extensively compare against COLA \cite{he_preserving_2025}, a recent state-of-the-art data curation method for compression algorithms that relies on both activation influence and data diversity metrics. To further validate \methodname, we show the comparison with a self-generative approach \cite{ji2025beware} in \Cref{apx:generativebaseline}. We compare the performance evaluation of both intra- (i.e., when the evaluation tasks \emph{match} with the calibration domain) and inter- (i.e., when the evaluation tasks \emph{do not match} with the calibration domain) tasks, as well as the runtime between \methodname and COLA.

\paragraph{Evaluation}
To assess downstream performance post-compression, we use the LM-Evaluation-Harness framework \cite{eval-harness} across five functional domains and three languages (EN, ES, and ZH). Results are reported on the subset of datasets that support standardized zero- or few-shot metrics. We categorize the datasets in functional domains as follows.
\textbf{(i) Language Modeling}: zero-shot perplexity is measured on WikiText \cite{merity_pointer_2016}, C4 \cite{raffel_exploring_2020}, and Pile \cite{gao_pile_2020}, which are general datasets that have commonly been used for model compression. \textbf{(ii) Mathematical Reasoning}: evaluated via GSM8k (5-shot) \cite{cobbe_training_2021}, SVAMP \cite{patel_are_2021}, and MMLU-M the subset of math tasks in MMLU. \textbf{(iii) Commonsense Reasoning \& QA}: zero-shot assessed using WinoGrande \cite{sakaguchikeisuke_winogrande_2021}, XWinograd \cite{tikhonov_its_2021}, CommonsenseQA \cite{talmor_commonsenseqa_2019}, HellaSwag \cite{zellers_hellaswag_2019}, OpenBookQA \cite{mihaylov_can_2018} and XCOPA \cite{ponti_xcopa_2020}. \textbf{(iv) Natural Language Inference (NLI)}: tested in zero-shot on RTE \cite{wang_glue_2018} and the adversarial ANLI \cite{nie_adversarial_2020} benchmarks. \textbf{(v) Knowledge \& Translation}: general world knowledge and reasoning are measured via MMLU-K, MMLU-ES, and MMLU-ZH, the MMLU corpus with math tasks excluded \cite{hendrycks_measuring_2021,singh_global_2025}, XQuAD \cite{artetxe_cross-lingual_2020}, XNLI \cite{conneau_xnli_2018}, and ARC \cite{clark_think_2018}, while translation capabilities are tested on WMT14 \cite{bojar_findings_2014}.

\begin{table}[t]
\centering
\resizebox{\linewidth}{!}{
\begin{tabular}{l c |cc| cc| cc}
\toprule
& & \multicolumn{2}{c}{\textbf{Wanda (25\%)}} & \multicolumn{2}{c}{\textbf{GPTQ (W4A16)}} & \multicolumn{2}{c}{\textbf{AWQ (W4A16)}} \\
\cmidrule(lr){3-4} \cmidrule(lr){5-6} \cmidrule(lr){7-8}
Task & Dense & Rand. & \methodname & Rand. & \methodname & Rand. & \methodname \\
\midrule
\cellcolor{blue!7} MMLU-M & 24.44 & 23.33 & \textbf{24.14} & 23.33 & \textbf{24.07} & 23.12 & \textbf{24.07} \\
\cellcolor{blue!7} GSM8K & 78.08 & 74.98 & \textbf{77.17} & 70.12 & \textbf{70.63} & 36.13 & \textbf{38.46} \\
\cellcolor{orange!7} HellaSwag & 71.70 & \textbf{64.49} & 64.28 & \textbf{64.23} & 64.20 & \textbf{64.85} & \textbf{64.85} \\
\cellcolor{orange!7} WinoGr. & 73.56 & 67.86 & \textbf{68.07} & 67.21 & \textbf{67.39} & 67.48 & \textbf{68.71} \\
\cellcolor{orange!7} OBQA & 47.80 & 41.83 & \textbf{42.42} & \textbf{42.02} & 41.10 & \textbf{40.90} & 40.45 \\
\cellcolor{orange!7} BoolQ & 85.41 & \textbf{84.80} & 84.64 & \textbf{83.53} & 83.44 & 83.98 & \textbf{84.36} \\
\cellcolor{purple!7} RTE & 74.37 & 72.39 & \textbf{72.77} & 71.75 & \textbf{73.06} & \textbf{74.73} & 73.47 \\
\cellcolor{purple!7} ANLI & 58.40 & 57.90 & \textbf{60.64} & 46.50 & \textbf{53.90} & 51.10 & \textbf{53.20} \\
\cellcolor{cyan!7} ARC-C & 53.75 & 50.33 & \textbf{50.89} & 50.03 & \textbf{50.37} & 51.41 & \textbf{51.54} \\
\cellcolor{cyan!7} ARC-E & 82.24 & 76.48 & \textbf{76.67} & \textbf{76.27} & 76.22 & 76.47 & \textbf{76.49} \\
\cellcolor{cyan!7} MMLU-K & 62.28 & 62.24 & \textbf{62.46} & 57.73 & \textbf{58.08} & 56.51 & \textbf{57.49} \\
\midrule
Mean & 64.73 & 61.51 & \posdelta{\textbf{62.20}} & 59.34 & \posdelta{\textbf{60.22}} & 56.97 & \posdelta{\textbf{57.55}} \\
\bottomrule
\end{tabular}
}
\caption{Performance comparison against random sampling across different tasks and compression techniques for \texttt{Meta-Llama-3.1-8B-Instruct}.}
\label{tab:full-compression-comparison}
\end{table}

\subsection{Experimental Results}\label{sec:experiments}
\paragraph{Better than Random}
The purpose of this experiment is to verify the hypothesis that \methodname identifies higher-utility data compared to standard uniform random sampling. Random sampling is statistically prone to overrepresent high-frequency tokens while failing to capture the \textit{Zipfian tail} composed of tokens that are an integral part of the corpora, as shown in \Cref{fig:zipf_tokens}.
Moreover, as numerically shown in \Cref{tab:full-compression-comparison}, \methodname wins 24 of 33 comparisons against random selection across the different compression techniques. On the few occasions where random selection performs best, it does so by a small margin ($<1\%$).
Most notably, we observe significant gains in reasoning-intensive tasks: on ANLI, our method improves accuracy by $7.4\%$ for GPTQ and $2.7\%$ for Wanda. Similarly, in GSM8K, we achieve a $2.1\%$ boost in the pruning setting with Wanda. This increased lexical representation correlates directly with performance; by forcing the calibration set to cover a broader vocabulary manifold, Zipf Sampling provides the compression algorithms with a more representative set of activation outliers \cite{sun2024massive,an2025systematic}. This leads to more robust statistics, effectively bridging the gap between the original dense model and its compressed counterpart. 
%Further discussion on the utility of these samples compared to samples based on various distribution metrics is provided in \Cref{apx:ablation-extra}.%\pio{When we're worse its by little, on average we go much better by 0.something}

\paragraph{Evaluating Single-Domain Data Curation}
We now discuss the evaluation against COLA on models compressed using a Single-Domain calibration source (i.e, the calibration data are extracted from a single domain). 
\Cref{tab:pruning_single,tab:quantization_single} shows how our proposed approach achieves on-par performance w.r.t. the baselines. Considering only the mean across all possible $<$task,domain$>$ pairs, \methodname performs better than COLA in 3 out of 4 cases. More importantly, as stated at the beginning of the paper, it reduces the data curation bottleneck from over an hour to mere seconds, yielding a $228-260\times$ acceleration while preserving downstream capabilities on par with computationally expensive baselines.

Furthermore, the results in \Cref{tab:pruning_single,tab:quantization_single} demonstrate that Single-Domain compression is highly sensitive to the choice of calibration data. Contrary to intuition, matching the calibration source domain to the task domain (e.g., using Commonsense Reasoning \& QA for BoolQ) does not universally hold the best performance (best is underlined in \Cref{tab:pruning_single,tab:quantization_single}, while domain-task match cells are colored). The results are also in line with \cite{bandari_is_2024}, where the authors showed how Language Modeling corpora are not inherently the optimal calibration source for diverse downstream tasks. 
This domain sensitivity finding introduces a form of intra-task sub-optimality that limits the possibility of deploying a unique compressed model that can achieve reasonable performance on a predefined downstream domain task.

\begin{table*}[t]
\centering
\resizebox{\textwidth}{!}{%
\begin{tabular}{l|l||c||ccccc|c||ccccc|c|c}
\toprule
& & & \multicolumn{6}{c||}{\textbf{COLA}} & \multicolumn{6}{c}{{\methodname}} \\
\cmidrule(lr){4-9} \cmidrule(lr){10-15}
&&&\multicolumn{5}{c}{\textbf{Calibration Category}} && \multicolumn{5}{c}{\textbf{Calibration Category}}&&
\\
\textbf{Model} & \textbf{Task} & \textbf{Dense} & \cellcolor{gray!7}LangMod & \cellcolor{blue!7}Math & \cellcolor{orange!7}CommQA & \cellcolor{purple!7}NLI & \cellcolor{cyan!7}KnowTran & \textbf{Mean} & \cellcolor{gray!7}LangMod & \cellcolor{blue!7}Math & \cellcolor{orange!7}CommQA & \cellcolor{purple!7}NLI & \cellcolor{cyan!7}KnowTran & \textbf{Mean} & \textbf{$\Delta$} \\
\midrule
%%mathmmlu
& \cellcolor{blue!7}MMLU-M & 24.44 & 23.33 & \cellcolor{blue!7}\underline{24.44} & 23.01 & 21.85 & 23.33 & \textbf{23.19} & 23.01 & \cellcolor{blue!7}\underline{26.29} & 21.24 & 22.02 & 23.01 & 23.11 & \negdelta{-0.08} \\
%%%
& \cellcolor{blue!7}GSM8k & 78.09 & \underline{71.75} & \cellcolor{blue!7}68.46 & 68.84 & 63.99 & 65.01 & 67.61 & 71.99 & \cellcolor{blue!7}\underline{74.60} & 65.13 & 70.96 & 67.32 & \textbf{70.00} & \posdelta{\textbf{+2.39}} \\
& \cellcolor{orange!7}HellaSwag & 71.71 & 72.07 & \underline{72.14} & \cellcolor{orange!7}71.91 & 72.06 & 71.54 & 71.94 & 71.84 & \underline{72.71} & \cellcolor{orange!7}71.46 & 71.89 & 72.37 & \textbf{72.05} & \posdelta{\textbf{+0.11}} \\
& \cellcolor{orange!7}WinoGr. & 69.46 & 66.75 & \underline{68.15} & \cellcolor{orange!7}66.34 & 66.30 & 66.73 & 66.85 & 66.34 & \underline{67.40} & \cellcolor{orange!7}66.69 & 67.05 & 67.32 & \textbf{66.96} & \posdelta{\textbf{+0.11}} \\
& \cellcolor{orange!7}OBQA & 47.80 & 45.60 & \underline{46.40} & \cellcolor{orange!7}46.20 & 43.90 & 45.90 & 45.60 & 45.20 & 47.20 & \cellcolor{orange!7}45.60 & \underline{47.50} & 46.20 & \textbf{46.34} & \posdelta{\textbf{+0.74}} \\
& \cellcolor{orange!7}BoolQ & 84.46 & 84.70 & 84.08 & \cellcolor{orange!7}84.08 & \underline{84.74} & 83.53 & \textbf{84.23} & 82.63 & 82.68 & \cellcolor{orange!7}83.79 & \underline{83.91} & 83.35 & 83.27 & \negdelta{-0.96} \\
& \cellcolor{purple!7}RTE & 74.73 & 74.61 & \underline{75.63} & 71.30 & \cellcolor{purple!7}74.91 & 74.37 & \textbf{74.16} & 72.38 & 70.76 & 71.84 & \cellcolor{purple!7}\underline{73.10} & 72.92 & 72.20 & \negdelta{-1.96} \\
& \cellcolor{purple!7}ANLI & 58.40 & 57.47 & 52.85 & 53.55 & \cellcolor{purple!7}54.40 & \underline{58.25} & \textbf{55.30} & 53.45 & 51.90 & 55.20 & \cellcolor{purple!7}\underline{56.25} & 55.05 & 54.37 & \negdelta{-0.93} \\
& \cellcolor{cyan!7}ARC-C & 51.71 & 50.14 & 50.38 & 50.21 & \underline{50.77} & \cellcolor{cyan!7}50.64 & 50.43 & 50.21 & 50.34 & 50.17 & 49.66 & \cellcolor{cyan!7}\underline{51.96} & \textbf{50.47} & \posdelta{\textbf{+0.04}} \\
& \cellcolor{cyan!7}ARC-E & 73.86 & 73.55 & 73.76 & 72.62 & 73.30 & \cellcolor{cyan!7}\underline{73.88} & 73.42 & 72.73 & 74.07 & 74.33 & 72.90 & \cellcolor{cyan!7}\underline{75.48} & \textbf{73.90} & \posdelta{\textbf{+0.48}} \\
& \cellcolor{cyan!7}MMLU-K & 62.28 & 57.46 & \underline{60.45} & 57.48 & 58.45 & \cellcolor{cyan!7}58.35 & \textbf{58.44} & 54.42 & \underline{58.64} & 57.40 & 56.96 & \cellcolor{cyan!7}58.12 & 57.11 & \negdelta{-1.33} \\

\cmidrule{2-16}
& \textbf{Mean} & 63.36 & 61.58 & 61.52 & 60.50 & 60.42 & 61.05 & \textbf{61.02} & 60.38 & 61.51 & 60.26 & 61.11 & 61.19 & 60.89 & \negdelta{-0.13} \\
& & & & & & & & $\pm 0.67$ & & & & & & $\pm 0.63$\\
\cmidrule{2-16}
\rowcolor{yellow!30} \cellcolor{white} \multirow{-12}{*}{\rotatebox[origin=c]{90}{\texttt{Llama-3.1-8B-Instruct}}} & \textbf{Runtime} & & 5400s & 36s & 3240s & 2160s & 1380s & 2443s & 15.2s & 2.3s & 12.3s & 9.3s & 14.5s & \textbf{10.7s} & \textbf{228$\times$}\\
\midrule
& \cellcolor{blue!7}MMLU-M & 21.48 & \underline{24.44} & \cellcolor{blue!7} 23.01 & 22.02 & 21.24 & 21.85 & 22.44 & 27.40 & \cellcolor{blue!7} 24.44 & 26.29 & \underline{28.79} & 24.44 & \textbf{26.27} & \posdelta{\textbf{+3.83}} \\
& \cellcolor{blue!7}GSM8k & 75.44 & 74.12 & \cellcolor{blue!7}74.30 & \underline{74.41} & 74.18 & 73.81 & \textbf{74.16} & 73.67 & \cellcolor{blue!7}\underline{73.92} & 73.39 & 73.81 & 73.54 & 73.66 & \negdelta{-0.50} \\
& \cellcolor{orange!7}HellaSwag & 67.24 & 67.43 & 67.56 & \cellcolor{orange!7}\underline{68.23} & 67.43 & 67.48 & 67.63 & \underline{68.02} & 67.61 & \cellcolor{orange!7}67.81 & 67.36 & 67.37 & 67.63 & \cellcolor{gray!10}+0.00 \\
& \cellcolor{orange!7}WinoGr. & 70.48 & 69.93 & \underline{70.01} & \cellcolor{orange!7}69.69 & 69.34 & 69.61 & 69.72 & 69.69 & \underline{70.60} & \cellcolor{orange!7}70.38 & 70.01 & 69.02 & \textbf{69.92} & \posdelta{\textbf{+0.20}} \\
& \cellcolor{orange!7}OBQA & 45.40 & \underline{45.47} & 44.50 & \cellcolor{orange!7}44.20 & 44.10 & 45.40 & 44.73 & 45.07 & 44.90 & \cellcolor{orange!7}\underline{45.20} & 45.00 & 45.00 & \textbf{45.03} & \posdelta{\textbf{+0.30}} \\
& \cellcolor{orange!7}BoolQ & 88.59 & 88.66 & 88.62 & \cellcolor{orange!7}88.32 & 88.50 & \underline{88.76} & 88.57 & 88.64 & 88.62 & \cellcolor{orange!7}88.58 & \underline{89.07} & 88.49 & \textbf{88.68} & \posdelta{\textbf{+0.11}} \\
& \cellcolor{purple!7}RTE & 78.34 & 78.10 & 76.53 & \underline{78.34} & \cellcolor{purple!7}77.26 & 77.62 & 77.57 & 77.26 & 77.80 & 77.80 & \cellcolor{purple!7}\underline{78.16} & 78.52 & \textbf{77.91} & \posdelta{\textbf{+0.34}} \\
& \cellcolor{purple!7}ANLI & 72.80 & 71.07 & 71.75 & 71.80 & \cellcolor{purple!7}71.90 & \underline{72.00} & 71.70 & 72.20 & \underline{72.90} & 70.30 & \cellcolor{purple!7}72.55 & 71.95 & \textbf{71.98} & \posdelta{\textbf{+0.28}} \\
& \cellcolor{cyan!7}ARC-C & 51.79 & \underline{52.28} & 50.77 & 51.79 & 51.32 & \cellcolor{cyan!7}51.41 & 51.51 & 52.22 & 52.09 & 52.43 & 51.45 & \cellcolor{cyan!7}\underline{52.52} & \textbf{52.14} & \posdelta{\textbf{+0.63}} \\
& \cellcolor{cyan!7}ARC-E & 66.79 & 67.31 & 66.50 & 66.73 & 66.22 & \cellcolor{cyan!7}\underline{67.66} & 66.88 & \underline{68.83} & 67.82 & 68.81 & 66.41 & \cellcolor{cyan!7}68.22 & \textbf{68.02} & \posdelta{\textbf{+1.14}} \\
& \cellcolor{cyan!7}MMLU-K & 33.83 & 33.23 & 26.74 & 29.17 & 29.70 & \cellcolor{cyan!7} \underline{37.46} & 31.26 & 36.76 & 30.64 & 34.03 & 28.79 & \cellcolor{cyan!7} \underline{41.04} & \textbf{34.25} & \posdelta{\textbf{+2.99}} \\

\cmidrule{2-16}
& \textbf{Mean} & 61.11 & 61.09 & 60.03 & 60.43 & 60.11 & 61.19 & 60.56 & 61.80 & 61.03 & 61.37 & 61.04 & 61.83 & \textbf{61.41} & \posdelta{\textbf{+0.85}} \\
& & & & & & & & $\pm 0.86$ & & & & & & $\pm 0.85$\\
\cmidrule{2-16}
\rowcolor{yellow!30} \cellcolor{white}\multirow{-12}{*}{\rotatebox[origin=c]{90}{\texttt{gemma-2-9b-it}}} &\textbf{Runtime} & & 6231s & 149s & 3400s & 2671s & 1500s & 2790s & 15.2s & 2.3s & 12.3s & 9.3s & 14.5s & \textbf{10.7s} & \textbf{260$\times$}\\
\bottomrule
\end{tabular}
}
\caption{Comparison of COLA vs.\ \methodname calibration data under GPTQ W4A16 quantization.}
\label{tab:quantization_single}
\end{table*}

\paragraph{Evaluating Multi-Domain Data Curation}

We tested our Multi-Domain approach of \methodname, \Cref{sec:multi_task}, to see if it can effectively bypass the aforementioned intra-task sub-optimality, as well as the inter-task limitations of real-world model deployment, where the specific downstream task domain that will be used is unknown \emph{a priori}.
By aggregating Zipfian subsets from multiple domains, we produce a single general calibration set. We show in \Cref{tab:cola_vs_mix_summary_reordered} that models compressed using a Multi-Domain source perform better overall on average across different tasks.
Most notably, Multi-Domain calibration delivers a model that achieves a score of $64.48$ for \texttt{Llama-3.1-8B} with Wanda pruning, which is a higher score than any individual calibration domain within the Single-Domain compressed models within \methodname and COLA groups.
These results highlight that the Multi-Domain version of \methodname can solve both the intra-suboptimality issue as well as the inter-task deployment.

\begin{table*}[t]
\centering
\resizebox{0.9\linewidth}{!}{%
\begin{tabular}{l|l||c||ccc|ccc|ccc|ccc}
\toprule
& & & \multicolumn{3}{c|}{\textbf{Wanda (25\%)}} & \multicolumn{3}{c|}{\textbf{2SSP (25\%)}} & \multicolumn{3}{c|}{\textbf{GPTQ (W4A16)}} & \multicolumn{3}{c}{\textbf{AWQ (W4A16)}} \\
\cmidrule(lr){4-6} \cmidrule(lr){7-9} \cmidrule(lr){10-12} \cmidrule(lr){13-15}
\textbf{Model} & \textbf{Task} & \textbf{Dense} & \textbf{COLA} & \textbf{\methodname} & $\Delta$ & \textbf{COLA} & \textbf{\methodname} & $\Delta$ & \textbf{COLA} & \textbf{\methodname} & $\Delta$ & \textbf{COLA} & \textbf{\methodname} & $\Delta$ \\
\midrule
\multirow{12}{*}{\rotatebox[origin=c]{90}{\texttt{Llama-3.1-8B-Instruct}}}
& \cellcolor{blue!7}MMLU-M & 24.44 & {23.19} & \textbf{24.00} & \posdelta{+0.81} & {22.22} & \textbf{22.54} & \posdelta{+0.32} & \textbf{23.19} & {23.01} & \negdelta{-0.18} & 24.44 & \textbf{24.44} & \cellcolor{gray!30}{0.00} \\
& \cellcolor{blue!7}GSM8k & 78.09 & \textbf{75.82} & 75.80 & \negdelta{-0.02} & \textbf{4.88} & 4.13 & \negdelta{-0.75} & 67.61 & \textbf{74.07} & \posdelta{+6.46} & 68.98 & \textbf{70.50} & \posdelta{+1.52} \\
& \cellcolor{orange!7}HellaSwag & 71.71 & \textbf{71.88} & 71.86 & \negdelta{-0.02} & \textbf{59.41} & 59.16 & \negdelta{-0.25} & \textbf{71.94} & 71.73 & \negdelta{-0.21} & 71.68 & \textbf{71.95} & \posdelta{+0.27} \\
& \cellcolor{orange!7}WinoGr. & 69.46 & 68.30 & \textbf{68.82} & \posdelta{+0.52} & \textbf{60.03} & 59.68 & \negdelta{-0.35} & 66.85 & \textbf{67.95} & \posdelta{+1.10} & 66.75 & \textbf{68.10} & \posdelta{+1.35} \\
& \cellcolor{orange!7}OBQA & 47.80 & 46.79 & \textbf{47.90} & \posdelta{+1.11} & 40.25 & \textbf{40.92} & \posdelta{+0.67} & \textbf{45.60} & 44.40 & \negdelta{-1.20} & 45.90 & 45.90 & \cellcolor{gray!30} 0.00 \\
& \cellcolor{orange!7}BoolQ & 84.46 & 84.78 & \textbf{85.23} & \posdelta{+0.45} & 72.57 & \textbf{73.12} & \posdelta{+0.55} & 84.23 & \textbf{85.38} & \posdelta{+1.15} & 84.07 & \textbf{84.60} & \posdelta{+0.53} \\
& \cellcolor{purple!7}RTE & 74.73 & 72.25 & \textbf{72.82} & \posdelta{+0.57} & 67.04 & \textbf{68.02} & \posdelta{+0.98} & 74.16 & \textbf{74.72} & \posdelta{+0.56} & 73.10 & \textbf{74.50} & \posdelta{+1.40} \\
& \cellcolor{purple!7}ANLI & 58.40 & 57.92 & \textbf{61.50} & \posdelta{+3.58} & \textbf{44.50} & 43.87 & \negdelta{-0.63} & 55.30 & \textbf{56.90} & \posdelta{+1.60} & 53.56 & \textbf{55.00} & \posdelta{+1.44} \\
& \cellcolor{cyan!7}ARC-C & 51.71 & 50.82 & \textbf{51.11} & \posdelta{+0.29} & 39.71 & \textbf{39.75} & \posdelta{+0.04} & \textbf{50.43} & 49.31 & \negdelta{-1.12} & 50.21 & \textbf{50.30} & \posdelta{+0.09} \\
& \cellcolor{cyan!7}ARC-E & 73.86 & 73.36 & \textbf{73.86} & \posdelta{+0.50} & \textbf{63.21} & 62.47 & \negdelta{-0.74} & \textbf{73.42} & 72.97 & \negdelta{-0.45} & 72.55 & \textbf{72.80} & \posdelta{+0.25} \\
& \cellcolor{cyan!7}MMLU-K & 62.28 & 62.02 & \textbf{76.37} & \posdelta{+14.35} & 35.59 & \textbf{36.07} & \posdelta{+0.48} & \textbf{58.44} & 55.86 & \negdelta{-2.58} & 58.52 & \textbf{58.60} & \posdelta{+0.08} \\

\cmidrule{2-15}
& \textbf{Mean} & 63.36 & 62.47 & \textbf{64.48} & \posdelta{\textbf{+2.01}} & 46.31 & \textbf{46.34} & \posdelta{\textbf{+0.03}} & 61.02 & \textbf{61.48} & \posdelta{\textbf{+0.46}} & 60.89 & \textbf{61.52} & \posdelta{\textbf{+0.63}} \\
\midrule
\multirow{12}{*}{\rotatebox[origin=c]{90}{\texttt{gemma-2-9b-it}}}
& \cellcolor{blue!7}MMLU-M & 21.48 & {22.51} & \textbf{27.40} & \posdelta{+4.89} & {19.26} & \textbf{23.33} & \posdelta{+4.07} & {22.44} & \textbf{26.29} & \posdelta{+3.85} & \textbf{21.40} & {21.30} & \negdelta{-0.10} \\
& \cellcolor{blue!7}GSM8k & 75.44 & 74.31 & \textbf{74.75} & \posdelta{+0.44} & 3.30 & \textbf{3.51} & \posdelta{+0.21} & \textbf{74.16} & 73.61 & \negdelta{-0.55} & \textbf{74.22} & 74.10 & \negdelta{-0.12} \\
& \cellcolor{orange!7}HellaSwag & 67.24 & 67.14 & \textbf{67.25} & \posdelta{+0.11} & \textbf{55.48} & 54.73 & \negdelta{-0.75} & \textbf{67.63} & 67.55 & \negdelta{-0.08} & \textbf{67.58} & 67.50 & \negdelta{-0.08} \\
& \cellcolor{orange!7}WinoGr. & 70.48 & \textbf{69.21} & 68.98 & \negdelta{-0.23} & \textbf{59.06} & 58.66 & \negdelta{-0.40} & \textbf{69.72} & 69.61 & \negdelta{-0.11} & \textbf{69.58} & 69.40 & \negdelta{-0.18} \\
& \cellcolor{orange!7}OBQA & 45.40 & \textbf{46.11} & 45.60 & \negdelta{-0.51} & 38.14 & \textbf{38.64} & \posdelta{+0.50} & 44.73 & \textbf{45.40} & \posdelta{+0.67} & 45.17 & \textbf{46.10} & \posdelta{+0.93} \\
& \cellcolor{orange!7}BoolQ & 88.59 & 88.73 & \textbf{89.08} & \posdelta{+0.35} & 72.78 & \textbf{73.07} & \posdelta{+0.29} & \textbf{88.57} & 88.04 & \negdelta{-0.53} & 88.57 & \textbf{89.00} & \posdelta{+0.43} \\
& \cellcolor{purple!7}RTE & 78.34 & 78.39 & \textbf{79.06} & \posdelta{+0.67} & \textbf{66.00} & 65.64 & \negdelta{-0.36} & 77.57 & \textbf{79.06} & \posdelta{+1.49} & 77.58 & \textbf{78.50} & \posdelta{+0.92} \\
& \cellcolor{purple!7}ANLI & 72.80 & \textbf{72.65} & 72.60 & \negdelta{-0.05} & 51.90 & \textbf{52.28} & \posdelta{+0.38} & 71.70 & \textbf{72.90} & \posdelta{+1.20} & 71.76 & \textbf{72.10} & \posdelta{+0.34} \\
& \cellcolor{cyan!7}ARC-C & 51.79 & 51.73 & \textbf{52.21} & \posdelta{+0.48} & 39.80 & \textbf{40.62} & \posdelta{+0.82} & \textbf{51.51} & 51.11 & \negdelta{-0.40} & \textbf{51.28} & 51.20 & \negdelta{-0.08} \\
& \cellcolor{cyan!7}ARC-E & 66.79 & 67.21 & \textbf{67.84} & \posdelta{+0.63} & \textbf{59.22} & 58.43 & \negdelta{-0.79} & 66.88 & \textbf{68.22} & \posdelta{+1.34} & 66.66 & \textbf{66.80} & \posdelta{+0.14} \\
& \cellcolor{cyan!7}MMLU-K & 33.83 & 34.62 & \textbf{36.96} & \posdelta{+2.34} & \textbf{22.77} & 22.41 &\negdelta{-0.36} & 31.26 & \textbf{35.19} & \posdelta{+3.93} & 32.00 & \textbf{34.00} & \posdelta{+2.00} \\

\cmidrule{2-15}
& \textbf{Mean} & 61.11 & 61.15 & \textbf{61.98} & \posdelta{\textbf{+0.83}} & 44.34 & \textbf{44.67} & \posdelta{\textbf{+0.33}} & 60.56 & \textbf{61.54} & \posdelta{\textbf{+0.98}} & 60.53 & \textbf{60.91} & \posdelta{\textbf{+0.38}} \\
\bottomrule
\end{tabular}
}
\caption{Comparison of COLA vs.\ \methodname (Multi-Domain) performance across Wanda and 2SSP at 25\% sparsity, and GPTQ and AWQ using W4A16 compression scheme.}
\label{tab:cola_vs_mix_summary_reordered}
\end{table*}%

%\paragraph{Evaluating Multi-Lingual Data Curation}\label{apx:multilingual}
\paragraph{Evaluating Single-Domain Single-Language Data Curation}
\Cref{tab:multilingual_single,tab:multilingual_gptq,tab:amultilingual_awq,tab:multilingual_2ssp} show the results among all the compression techniques in our testbed. In 5 out of 8 model-compression configurations, models compressed with \methodname-curated data outperform those compressed using COLA-curated data. In the remaining cases, the performance degradation is marginal (-0.30 on average, -0.40 at most). Furthermore, we emphasize that \methodname achieves these results while retrieving the curation data in mere seconds, compared to the significantly longer processing time required by COLA (which averages 17 minutes per dataset). 

We previously discussed that there is little to no correlation between the calibration data domain and higher performance on downstream tasks of that same domain. The multilingual results allow us to further expand this claim: counterintuitively, utilizing calibration data in a specific language does not guarantee optimal compression for downstream tasks in that identical language. We observe that performance on MMLU-ES is actually higher when the model is compressed using Chinese calibration data. Similarly, MMLU-ZH and XNLI-ZH frequently reach peak performance when models are compressed using Spanish data. This pattern of mismatches repeats consistently across the different compression techniques.
We can confidently develop the observation from Bandari et al. \cite{bandari_is_2024} about the mismatch between the calibration data domain and the higher performance in downstream tasks of the same domain to languages. The optimal calibration data for a downstream task in a certain language does not necessarily coincide with the data from that same language.

\paragraph{Evaluating Multi-Domain Multi-Language Data Curation}
We have seen that language alignment between calibration data and evaluation tasks is neither a necessary nor sufficient condition for preserving performance. This mismatch adds another layer of complexity to the manual selection of calibration data, as the natural intuition on matching languages falls short.
We recall that \methodname was developed with property \Gtre in mind. 
We now test the Multi-Domain approach of \methodname under this more difficult setup, to see if the models compressed with aggregated Zipfian subsets can achieve both intra-task sub-optimality and cross-lingual sub-optimality.
The results, reported in  \Cref{tab:cola_vs_mix_multilingual_summary}, highlight that this aggregation across different languages and tasks yields compressed models that, on average, outperform their single-task, single-language counterparts. While our previous evaluations on strictly English tasks demonstrated relatively stable and uniform improvements, the multilingual, multi-domain setting exhibits a noticeably higher skewness in the performance deltas. 
%As shown in the table, \methodname triggers localized gains (e.g., a +8.75\% surge on MMLU-ZH under 2SSP for Gemma-2 and a +7.53\% boost on XQuAD-ZH for Llama-3.1), with occasional sharp regressions (such as a -15.02\% drop on XQuAD-ZH under GPTQ). Despite this heightened variance on individual tasks, \methodname's aggregation strategy effectively acts as a stabilizer, 
\methodname yields a positive mean $\Delta$ in 6 out of 8 compression scenarios and proves itself more robust than relying on naive language-matching.

\paragraph{Evaluation on instruction-tuned benchmarks}
To assess generation quality beyond likelihood-based benchmarks, we evaluate on AlpacaEval \cite{alpaca_eval}, reporting both win rate and the length-controlled (LC) win rate \cite{dubois2024length}, which also accounts for the length of the outputs. Due to computational cost limitations, we use Llama-3.1-8B-Instruct itself as a local judge rather than GPT-4, which introduces a self-evaluation bias: the judge may systematically favor generations that resemble its own output distribution, independent of true quality. \Cref{tab:alpaca_eval} shows that \methodname matches or slightly exceeds COLA across all four compression methods ($\Delta$ ranging from -0.1 to +1.6 points), these values indicate comparable generation quality. Because the same judge and prompt set are applied identically to both COLA and \methodname calibrated models, any self-evaluation bias is shared across conditions and should not favor one calibration method over the other.

\begin{table}[!ht]
\centering
\begin{adjustbox}{width=\columnwidth}
\begin{tabular}{c c c c c c}
\toprule
\textbf{Model} & \textbf{Method} & \textbf{COLA} & \methodname-SD & \methodname-MD & \textbf{$\Delta$} \\
\midrule
\multirow{5}{*}{\rotatebox[origin=c]{90}{\texttt{Llama-3.1-8B}}}
& Dense & \multicolumn{3}{c}{50.0/28.9} \\
\cmidrule{3-5}
& Wanda (25\%) & 41.3/27.4 & 40.8/27.1 & 41.3/27.4 & \cellcolor{gray!10}0.0 \\
& 2SSP (25\%) & 16.7/12.2 & 10.0/7.2 & 18.3/12.5 & \posdelta{+1.6} \\
& GPTQ (W4A16) & 48.8/32.2 & 48.6/32.0 & 49.1/32.6 & \posdelta{+0.3} \\
& AWQ (W4A16) & 49.1/30.6 & 48.8/32.4 & 49.0/30.4 & \negdelta{-0.1} \\
\bottomrule
\end{tabular}
\end{adjustbox}
\caption{AlpacaEval win rate / length-controlled win rate (\%) over 805 instructions, judged locally with Llama-3.1-8B-Instruct, comparing COLA and \methodname (SD/MD) calibration across compression methods on Llama-3.1-8B-Instruct.}
\label{tab:alpaca_eval}
\end{table}

\paragraph{Evaluating Scalability}
We evaluated the efficiency of \methodname against COLA across various dataset scales, to assess its scalability \Guno. 
While already in \Cref{tab:pruning_single,tab:quantization_single} the runtime superiority of \methodname is clear, we also tried gradually subsampling two datasets, namely ARC-C and WinoGrande, to better understand the computational complexity trade-off between \methodname and COLA. Results are reported in \Cref{fig:scalability}. On a small dataset like ARC-C, COLA needs from a few seconds with a small 3B model to some minutes using a large 70B model, with the forward pass becoming a bottleneck even on a dataset of this size.
We observe a similar trend with WinoGrande, which is an order of magnitude larger than ARC-C. COLA with a 70B model requires over 3 hours, while \methodname takes just 9 seconds. Overall, \methodname is 104$\times$ faster over the 3B model, and achieves up to 1330$\times$ speedup over the 70B model.

\begin{figure}[!t]
    \centering
    \includegraphics[width=\linewidth]{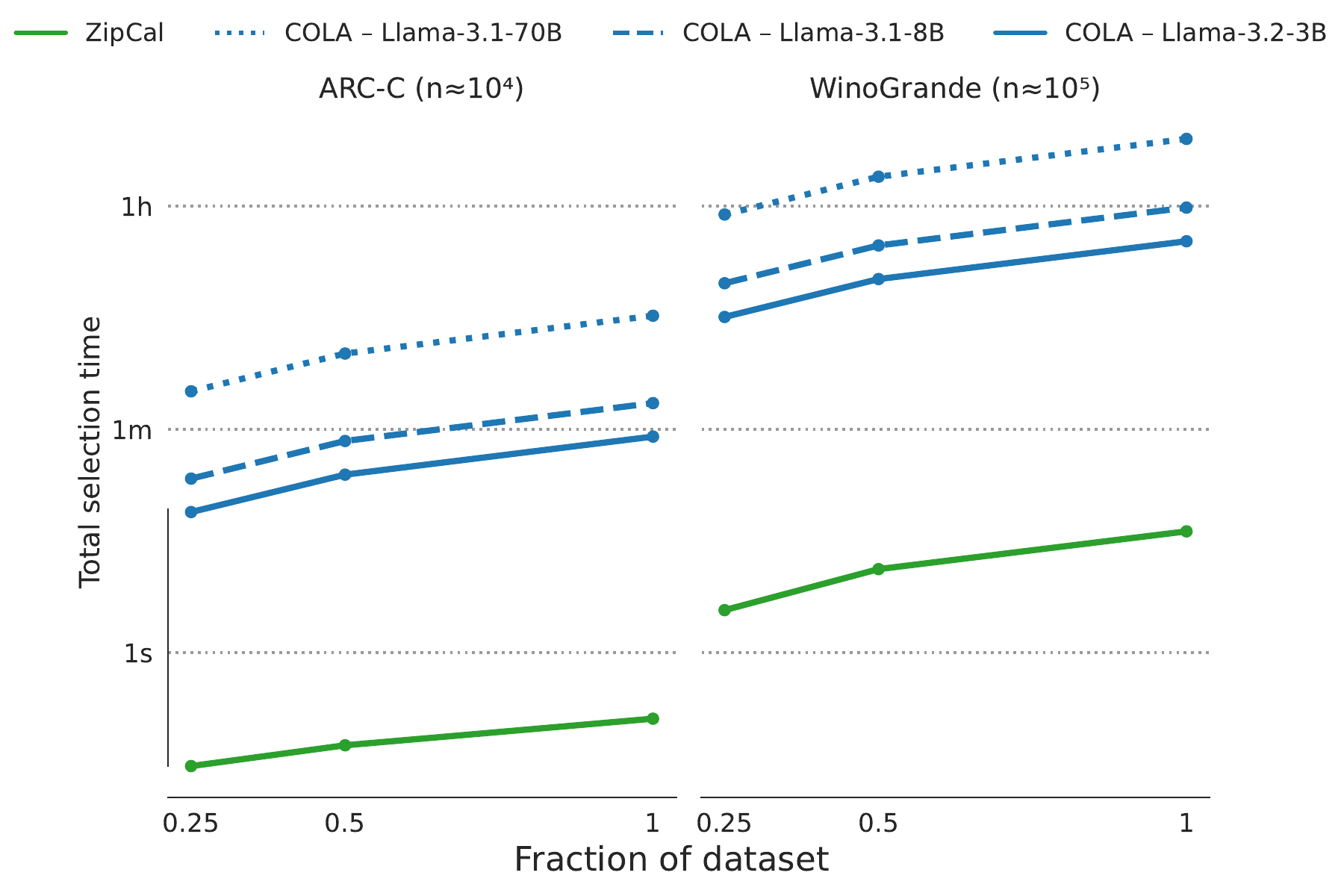}
    \caption{Running time (log-scale) for calibration data selection. The \protect\ulc{COLA}{turchese} baseline is run for models of different sizes; whereas \protect\ulc{\methodname}{foresta} is model-agnostic, thus we report the measurement of the only run.}
    \label{fig:scalability}
\end{figure}

\paragraph{Remarks}
To conclude, the results highlight that while model-based curation becomes computationally prohibitive as LLMs scale in size, our model-agnostic approach \Guno is a near-zero-overhead solution \Gdue that maintains high-fidelity calibration without requiring inference. Crucially, the efficiency gain is model-agnostic; once a calibration set is computed, it is fixed and reusable across different models and compression techniques.
Moreover, the Multi-Domain version of \methodname overcomes the necessity of a \textit{lucky} pick from a plethora of datasets and samples. The proposed sampling approach provides a compressed model that, on average, outperforms models compressed separately on single domains and single languages \Gtre. %(inter-task suboptimality) 
\section{Further Insights}

\paragraph{Ablation on Length-Controlled Samples}
\methodname's selected samples are longer on average than those chosen by COLA. For example, on HellaSwag, the median length of the samples picked by \methodname is 30 tokens, compared to 23 for those picked by COLA. This difference is exacerbated on more complex datasets, such as BoolQ, where \methodname's samples have a median length of 171 tokens, compared to 76 for COLA's samples.
One could think that length could be driving the quality of the results rather than lexical representation.
To isolate the effect of linguistic diversity from the effect of length, we truncate every sample in the source corpus $C$ to $l=64$ tokens, creating $C'$. We run both calibration methods over $C'$ to select the calibration samples $S$, then compression is performed as always. \Cref{tab:summary_trunc} reports the results, \methodname retains an advantage on 2SSP and AWQ even under this strict length control, and remains competitive for the other compression techniques. 

\begin{table}[!t]
\centering
\resizebox{\linewidth}{!}{%
\begin{tabular}{cc|ccc}
\toprule
\textbf{Model} &\textbf{Method} & \textbf{COLA} & \textbf{\methodname} & $\Delta$ \\
\midrule
\multirow{4}{*}{\rotatebox[origin=c]{90}{\texttt{Llama-3.1-8B}}} 
\\[-1em]
& Wanda (25\%) & \textbf{$62.96 \pm 0.15$} & $60.06 \pm 0.17$ & \negdelta{-2.90} \\
\\[-1em]
&2SSP (25\%) & \textbf{$40.39 \pm 0.12$} & $42.86 \pm 0.12$ &  \posdelta{2.47} \\
\\[-1em]
&GPTQ (W4A16) & \textbf{$58.40 \pm 0.15$} & $57.97 \pm 0.15$ & \negdelta{-0.43} \\
\\[-1em]
&AWQ (W4A16)& $58.01 \pm 0.15$ & \textbf{$60.57 \pm 0.15$} & \posdelta{2.56} \\
\\[-1em]

\bottomrule
\end{tabular}
}
\caption{Comparison of COLA vs. \methodname across compression methods on Llama-3.1-8B-Instruct using samples truncated at $l=64$ tokens.}
\label{tab:summary_trunc}
\end{table}

\paragraph{Mechanistic Analysis}
To understand why lexical diversity correlates with more robust compressed models, we recall that \methodname represents a larger vocabulary space (\Cref{fig:zipf_tokens}). This results in a more accurate representation of the original dataset's long tail, increasing entropy in the activations and gathering richer statistics for otherwise dormant parts of the network that would be compressed. This comes with a measurable trade-off: decomposing distributional fidelity into head- and tail-frequency regions, \methodname substantially improves tail-region KL divergence at the cost of increased head-region divergence, confirming that its benefit is concentrated in restoring exposure to the long tail rather than uniformly matching the frequency profile. We define the \textit{overlap} between two models as the ratio between the number of weights compressed in both models and the total number of weights. The overlap between random and \methodname-compressed models is around $90\%$, compared to $96\%$ between random and COLA, with overlap dropping further to $85$–$94\%$ in the attention output projection (\texttt{o\_proj}) and the feed-forward down-projection (\texttt{down\_proj}), the layers where each sub-block's output rejoins the stream, evidence that \methodname rescues weight pathways that random and COLA calibration leave underactivated. Beyond raw rare-token counts, which we find only weakly predictive, \textit{calibration entropy} is the dominant predictor of post-compression resilience, correlating strongly with performance retention ($\rho = 0.94$, $p \leq 0.005$). For an extended analysis, see \Cref{apx:mechanistic}.

\paragraph{Effects of Compression on Perplexity}
Along with the downstream evaluation over Instruct models, we also evaluate \methodname against COLA over Language Modeling tasks (i.e., perplexity). The results reported in \Cref{tab:perplexity_comparison_updated} show that Single-Domain and Multi-Domain \methodname performances are once again on par with COLA. We want to highlight that the calibration sets extracted by \methodname are shared across different models without adaptation, confirming that lexically diverse calibration sets preserve compression quality.

\begin{table}[!ht]
\centering
\begin{adjustbox}{width=\columnwidth}
\begin{tabular}{c c c c c c}
\toprule
\textbf{Model} & \textbf{Method} & \textbf{COLA} & \methodname-SD & \methodname-MD & \textbf{$\Delta$} \\
\midrule

\multirow{5}{*}{\rotatebox[origin=c]{90}{\texttt{Llama-3.1-8B}}}
& Dense & 7.15 & 7.15 & 7.15 \\
\cmidrule{3-5}
& Wanda (25\%) & 7.77 & 7.76 & \textbf{7.67} & \posdelta{-0.10} \\
& 2SSP (25\%) & 16.32 & \textbf{15.81} & 18.52 & \posdelta{-0.51} \\
& GPTQ (W4A16) & \textbf{8.09} & \textbf{8.09} & 8.67 & \cellcolor{gray!10}+0.00 \\
& AWQ (W4A16) & \textbf{7.35} & 7.40 & 7.40 & \negdelta{+0.05} \\
\midrule

\multirow{5}{*}{\rotatebox[origin=c]{90}{\texttt{gemma-2-9B}}}
& Dense & 7.93 & 7.93 & 7.93 \\
\cmidrule{3-5}
& Wanda (25\%) & \textbf{8.13} & 8.22 & \textbf{8.13} & \cellcolor{gray!10}0.00 \\
& 2SSP (25\%) & 12.74 & \textbf{12.41} & 14.18 & \posdelta{-0.33} \\
& GPTQ (W4A16) & \textbf{8.27} & \textbf{8.27} & 8.67 & \cellcolor{gray!10}0.00 \\
& AWQ (W4A16) & 15.84 & 15.75 & \textbf{15.73} & \posdelta{-0.11} \\
\bottomrule
\end{tabular}
\end{adjustbox}
\caption{Language modeling perplexity ($\downarrow$). Comparison between COLA, \methodname, and \methodname (Multi-Domain). $\Delta$ indicates the difference between our best and COLA on Avg. over C4, WikiText, and Pile.}
\label{tab:perplexity_comparison_updated}
\end{table}

\section{Conclusion}
In this paper, we highlighted the limitations of current calibration data selection strategies for compression algorithms, which are either based on pure random sampling or use expensive, model- and task- dependent calibration strategies, limiting performance across multiple models and tasks.
To address these issues, we proposed \methodname, a model-agnostic, computationally cheap data curation strategy for both pruning and quantization approaches, which selects calibration data by following a linguistics-principled Zipfian distribution. Results show that the proposed method performs better than random on majority of the evaluation benchmarks, and more importantly, on par with state-of-the-art data curation approaches while requiring minimal overhead. Furthermore, we also introduced a multi-domain version of \methodname, which applies Zipf sampling hierarchically across different calibration datasets. Results show that the resulting unified multi-domain calibration dataset allows for outperforming single-domain calibration, proving a solution to the aforementioned problem of inter-task sub-optimality.
%does not require any \emph{a priori} knowledge about the downstream tasks, but rather 

\FloatBarrier

\clearpage

%%% SAMPLE DATA
\section*{Limitations}
While we provide extended experimental insights on choosing calibration data based on linguistic diversity, we acknowledge the following limitations to the current study.
We mainly relied on English calibration data and evaluation tasks for the main paper, with an accompanying extension to Spanish and Chinese. Although Zipf's Law is a cross-linguistic phenomenon \cite{piantadosi_zipfs_2014}, the specific ``sanitization'' process (e.g., lowercase conversion and subword marker stripping) may require tuning for morphologically rich languages or non-alphabetic scripts.
Our experiments focus on three specific families of LLMs (Llama-3.1-8B, Gemma-2-9B, and Qwen3-8B) and four compression methods. It remains to be seen how lexical diversity as a curation proxy behaves for Mixture-of-Experts (MoE) models, where activation routing might necessitate a different balance of data to ensure all ``experts'' are adequately calibrated. Similarly, extending this approach to Multimodal Large Language Models (MLLMs) would require a multimodal analogue of lexical diversity that captures the distributional properties of non-textual tokens, such as visual, audio, or video modality.

\section*{Ethics Statement}
This work focuses on efficient data curation for model compression. All models, datasets, and benchmarks used in this research are publicly available with appropriate licenses. We credit original authors throughout the manuscript and acknowledge their contributions. We acknowledge that data selection processes can inadvertently amplify existing biases in the source data, and we have reported detailed performance statistics across our benchmarks. We also recognize that efficient techniques for deploying LLMs accelerate AI adoption, potentially leading to misuse. However, our focus remains on the foundational technical aspects of data curation rather than specific downstream applications of AI.

\section*{Acknowledgments}
%``ARCHitectures based on unconventional accelerators for dependable/energY efficienT AI Systems'' 
This work was supported in part by the UE GA n. 101167870 "ARCHitectures based on unconventional accelerators for dependable/energY efficienT AI Systems" - ARCHYTAS. Views and opinions expressed are however those of the authors only and do not necessarily reflect those of the European Union or the European Commission. Neither the European Union nor the granting authority can be held responsible for them. 

\bibliography{library}

\clearpage
\appendix
\startlist{toc}
\label{sec:appendix}
\printlist{toc}{}{\bfseries\large Appendix Table of Contents\vskip 5ex}
\vskip 2ex

\section{Detailed Algorithms and Proofs}\label{apx:extra_algo}
\setcounter{table}{0}
\renewcommand{\thetable}{A\arabic{table}}
\setcounter{figure}{0}
\renewcommand{\thefigure}{A\arabic{figure}}
\paragraph{Multi-Domain Sampling}
We present here the pseudocode for the Multi-Domain selection procedure, \Cref{algo:multi}, and prove \Cref{lemma:multi}, the time complexity of \methodname in this setting. For the lightweight embedding used for all our experiments, we relied on \texttt{all-MiniLM-L6-v2} \cite{wang_minilm_2020}.
\begin{proof}
Given $m$ datasets $D_i:i\in(1,m)$.
In the first stage, Single-Domain \methodname is executed for each dataset. Let $n_i$ be the size of dataset $\mathcal{D}_i$. From \Cref{lemma:sampling}, the complexity for each source is $O(n_i k)$. An upper bound to this operation is $mkN$, where $N = \max_{D_i} (n_i) $
The second stage involves the $k$-centers greedy selection on a pool of size $mk$. At each of the $k$ iterations, we calculate the distance between the remaining candidates and the current set $S$. The cost of this stage is $O(k \cdot (mk) \cdot d)$, where $d$ is the dimensionality of the sample representation. Since $m$ and $k$ are typically $<<N$, the term $mkd$ is negligible. Thus, the overall complexity is dominated by the initial \methodname pass, yielding $O(mkN)$.
\end{proof}
\begin{algorithm}[t]
\KwIn{Collection of datasets $\mathbb{D} = \{\mathcal{D}_1, \dots, \mathcal{D}_m\}$; Final budget $k$}\KwOut{Multi-Domain calibration set $S$}\caption{Multi-Domain \methodname \label{algo:multi}}
$\mathcal{P} \gets \emptyset$\;
\ForEach{$\mathcal{D}_i \in \mathbb{D}$}{
    \tcp{Step 1: Single-Domain \methodname}
    $P_i \gets \text{\methodname}(\mathcal{D}_i, k)$ \;
    $\mathcal{P} \gets \mathcal{P} \cup P_i$ \;
}

\tcp{Step 2: Global K-Centers Greedy Selection}
$\forall s \in P \text{ create a lightweight embedding}$\;
$s_1 \gets \text{select random } s \in \mathcal{P}$\;
$S \gets \{s_1\}$\;
\For{$j = 2$ \KwTo $k$}{
    $s^* \gets \arg\max_{s \in \mathcal{P} \setminus S} \left( \min_{z \in S} \text{dist}(s, z) \right)$ \;
    $S \gets S \cup \{s^*\}$ \;
}
\Return $S$ \;
\end{algorithm}

%%%%OTHER SECTION
\section{Complementary Results to the Main Matter}\label{apx:results}
We report here additional tables that were omitted from the main paper due to length constraints. In particular, \Cref{tab:awq_words_dataset_updated,tab:2ssp_words_dataset} show the comparison between COLA and Single-Domain \methodname under AWQ and 2SSP compression, respectively. Standard deviations on all tables are derived from 3 seeded runs.

\begin{figure}[t]
    \centering
    \includegraphics[width=\linewidth]{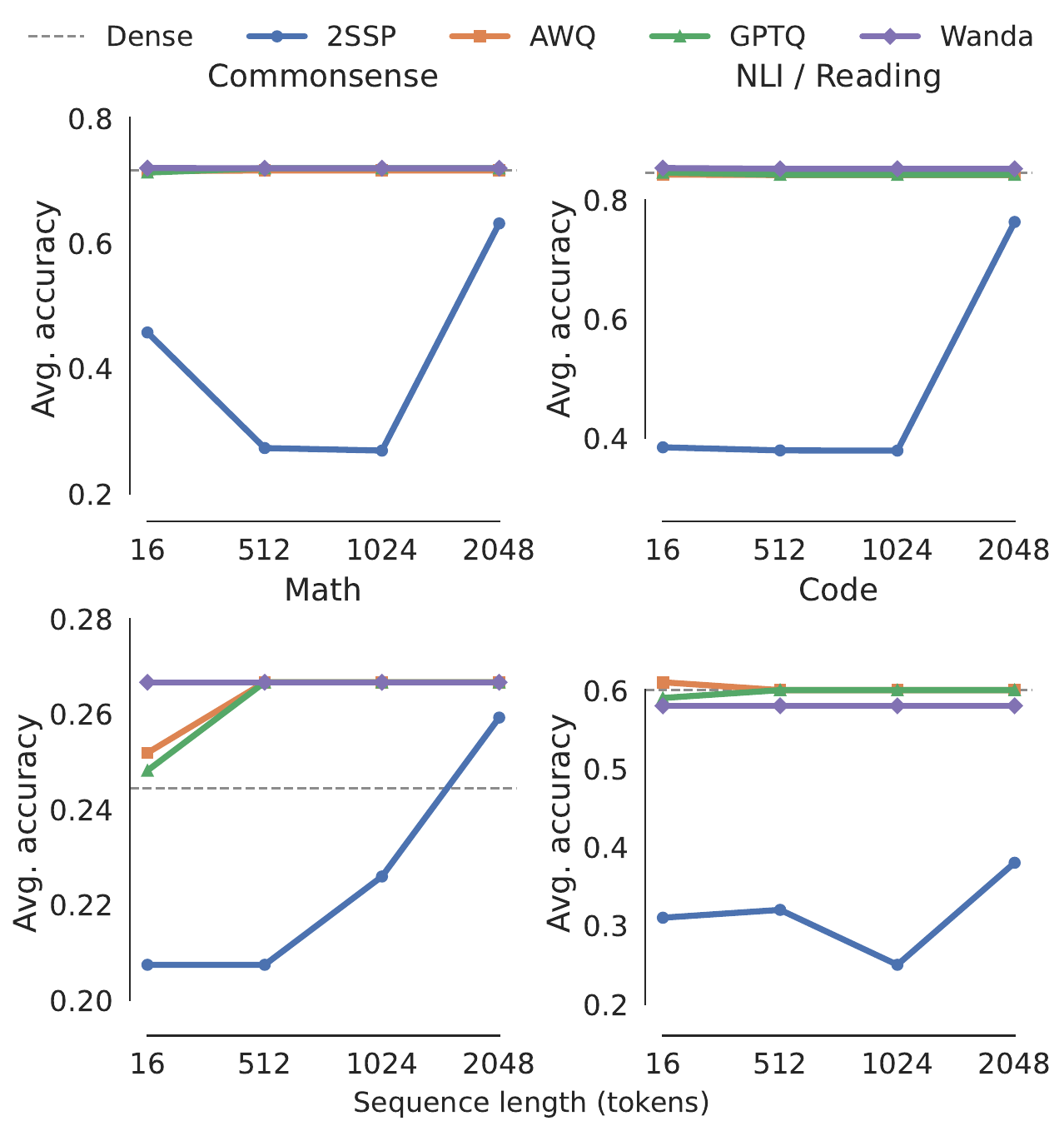}
    \caption{Effect of calibration data context length on model capabilities across compression techniques for \texttt{LLaMA-3.1-8B-Instruct}.}
    \label{fig:context-len}
\end{figure}

\begin{figure}[t]
    \centering
    \includegraphics[width=\linewidth]{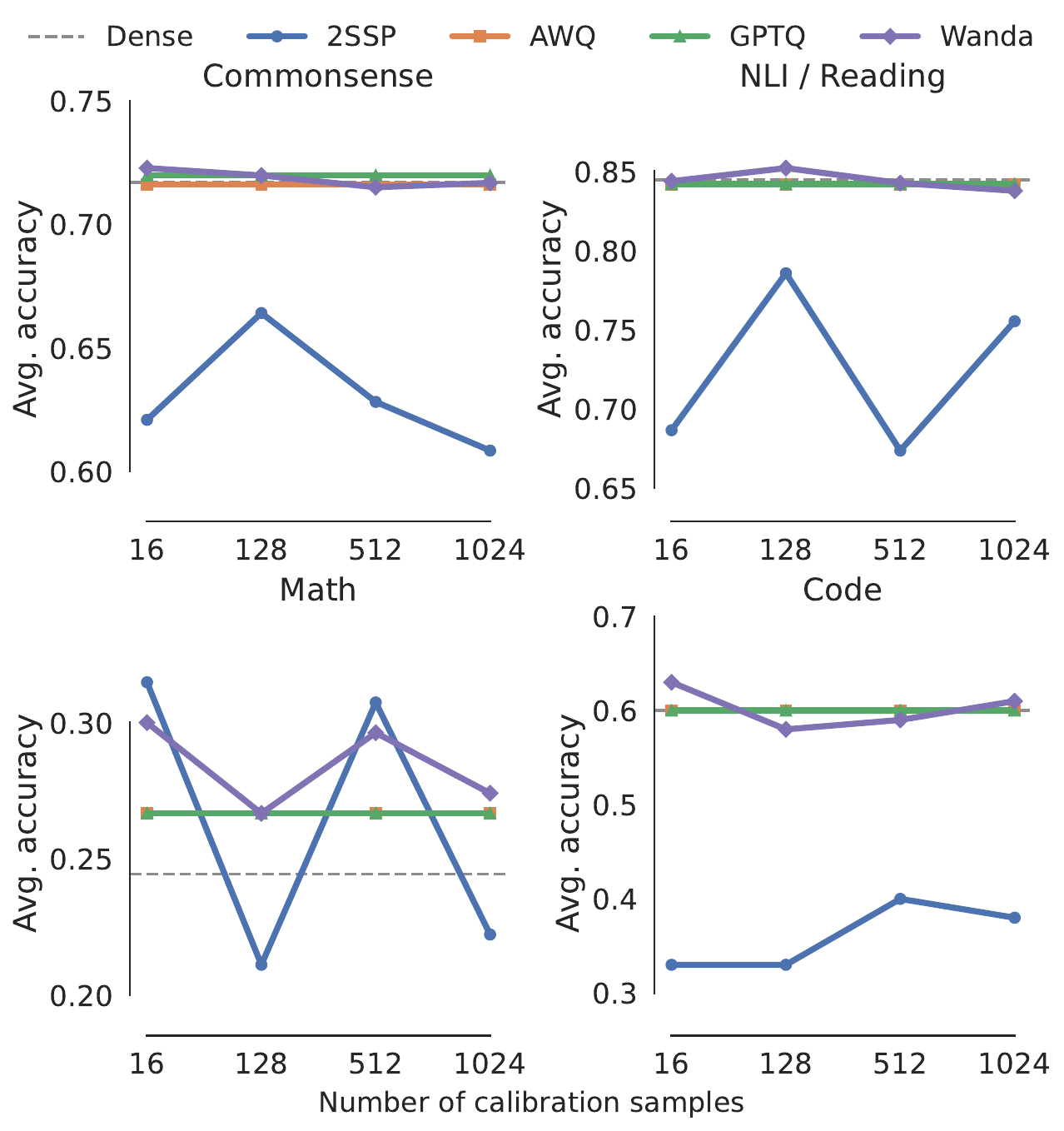}
    \caption{Effect of the number of calibration data samples on model capabilities across compression techniques for \texttt{LLaMA-3.1-8B-Instruct}.}
    \label{fig:sample-size}
\end{figure}

\paragraph{Expanding the Model-Agnostic Claim}
We extend our single-domain evaluation to Qwen3~\cite{yang_qwen3_2025}, an architecturally distinct third family (QK-norm, unlike Llama's pre-norm-only~\cite{grattafiori_llama_2024} or Gemma~2's pre+post-norm design~\cite{team_gemma_2024}). Results in \Cref{tab:qwen_results} confirm that our main claims hold: \methodname matches or exceeds COLA across compression methods, despite using calibration data selected in a fraction of the time.

\paragraph{Effects of Context Length}
We investigate the effect of calibration sequence length by testing across a set of values $w\in[16,2048]$ and setting the number of samples $k=128$. As shown in \Cref{fig:context-len}, we find that for most compression techniques, no significant improvement is observed by increasing the context window. Performance is flat even at $w=16$, which suggests that lexically diverse calibration sets work even at small scales. This trend seems to be characterized by the compression technique itself rather than the length of the calibration set. For instance, 2SSP shows extreme sensitivity and strictly requires the longest context length to recover performance. The reason is that 2SSP removes entire attention submodules, which requires capturing long-range relationships; at small context length, this information is removed and leads to biased pruning.
Our results are in line with the comprehensive analysis of \cite{oh_beyond_2025}.

\paragraph{Effects of Calibration Sample Size}
To further highlight the efficiency of \methodname, we analyze the impact of the number of samples in the calibration set by testing $k\in[16,2048]$ with context length $w=2048$. The results in \Cref{fig:sample-size} show great stability among capabilities and the number of samples. Once again, 2SSP shows greater sensitivity to sample count. Surprisingly, Code and Math capabilities actually benefit from a smaller number of samples (38\% vs. 31\% using 16 vs. 1024 samples, respectively, for Math capabilities after 2SSP), finding a negligible decrease. Compared to baselines such as COLA \cite{he_preserving_2025}, which exhibit some capability shifts at larger sample sizes, \methodname provides a much more stable and predictable function.

\paragraph{Scalability to Larger Language Models}
To investigate if \methodname data curation is effective on larger models, we repeat the set of experiments about Single-Domain and Multi-Domain data curation on Llama-3.1-70B-Instruct. For Single-Domain data curation, we can see in \Cref{tab:70Bresults} that the patterns seen on smaller models hold even in this case, with performances that are on par with COLA but require much less time to extract the curation set of samples \Guno, increasing the  $188\times$ speedup of the 8B model to $404\times$. On Multi-Domain curation, the advantages of \methodname are even more pronounced, \Cref{tab:cola_vs_mix_70B}. For the retention of mathematical capabilities, \methodname is able to protect the information flow with performance that is very close to the dense model. On the other hand, COLA completely fumbles these tasks, leading to a net delta on the overall quality of the results that averages $+3.38$.

\begin{figure*}[t]
    \centering
    \includegraphics[width=\linewidth]{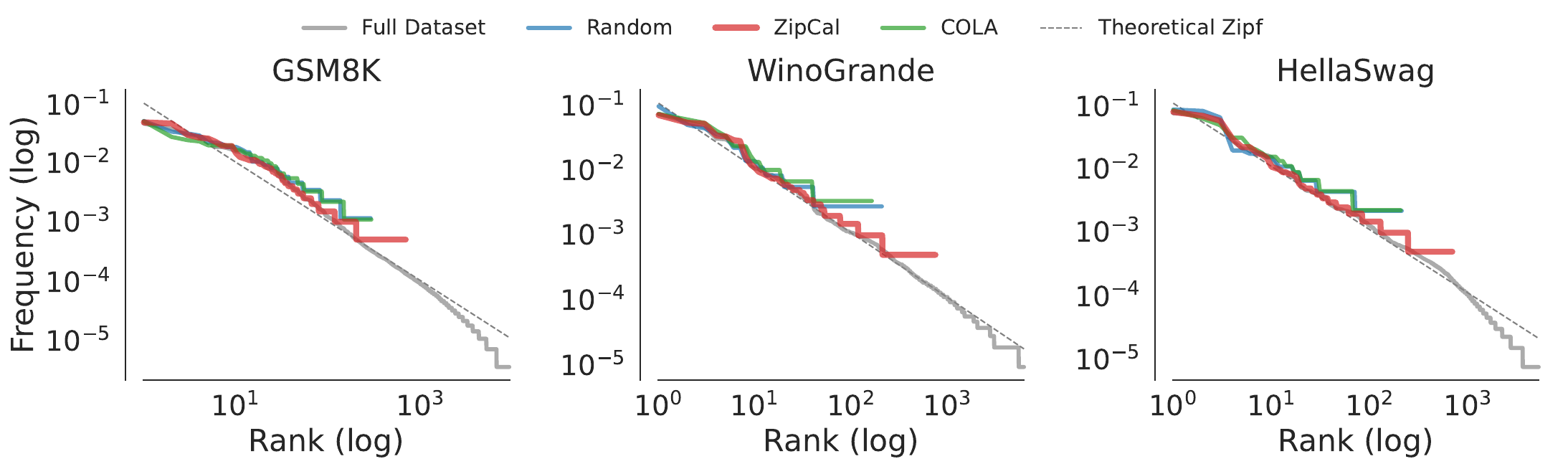}
    \caption{Token frequency distribution of the original datasets and the random, COLA, and \methodname sampling calibration sets using  16 samples.}
    \label{fig:token_tinysample}
\end{figure*}

\begin{figure*}[t]
    \centering
    \includegraphics[width=\linewidth]{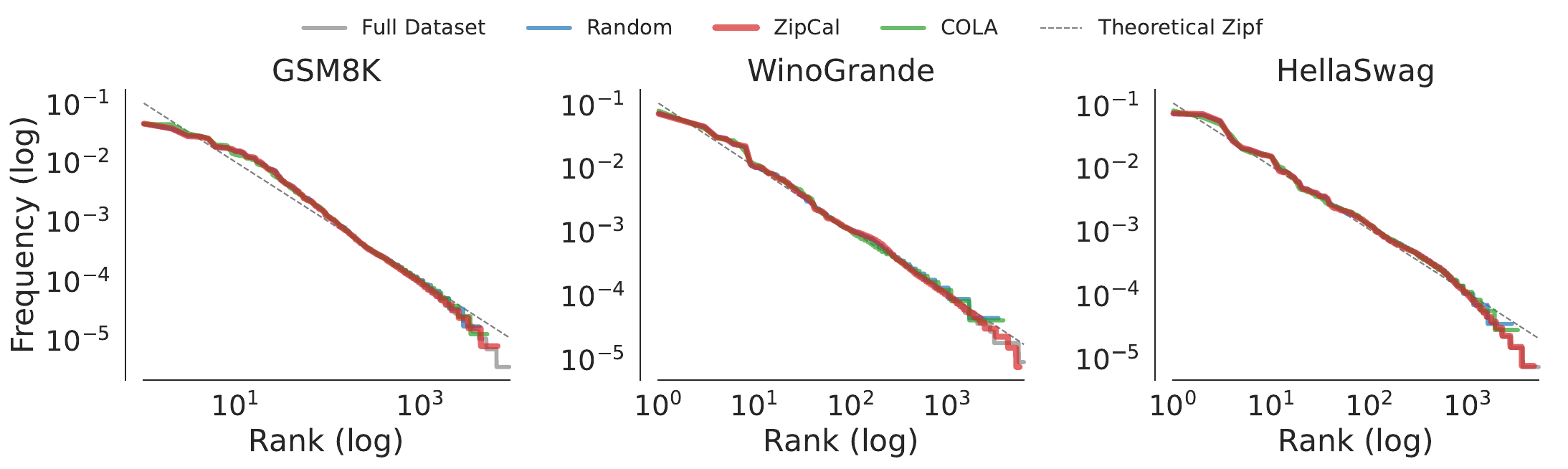}
    \caption{Token frequency distribution of the original datasets and the random, COLA, and \methodname sampling calibration sets using  1024 samples.}
    \label{fig:token_bigsample}
\end{figure*}

\section{Assessing Zipf Coverage}
As illustrated in \Cref{fig:zipf_tokens}, \methodname achieves a vocabulary coverage that better reflects the original dataset.
When restricted to a standard budget of 128 samples, \methodname identifies by construction the subset that yields the longest possible Zipfian tail. On the other hand, random sampling would need a significantly higher number of samples in order to cover the same vocabulary space. COLA's selection criterion leads to a similar distribution to random sampling. 
This advantage in coverage is especially noticeable in lexically rich corpora like WinoGrande and HellaSwag, where the \textit{gap} between \methodname and competing baselines reaches up to half an order of magnitude. At extremely tiny samples (i.e., 16), \Cref{fig:token_tinysample}, the difference is even more pronounced, with \methodname being the only one capable of covering more than half of the original vocabulary. For large amounts of samples, \Cref{fig:token_bigsample}, \methodname actually covers the whole vocabulary space compared to the other baselines.
A qualitative analysis on the selected samples reveals that \methodname selects samples with significantly higher median lengths and Type-Token Ratios (TTR). This advantage is especially noticeable in lexically rich corpora like WinoGrande and HellaSwag, where the gap between \methodname and competing baselines reaches half an order of magnitude, capturing the encyclopedic context that others ignore. We provide some examples in Boxes~\ref{box:hellaswag}~and~\ref{box:boolq}.

\begin{samplebox}[float=ht!,fontupper=\small,fontlower=\small]{HellaSwag}{hellaswag}
\textbf{\methodname} \textit{Median Length: 30.0, Unique Words: 1455}
\begin{itemize}[leftmargin=*] \itemsep0em
    \item A man has shaving cream on his face and reads a label on a pressurized can and discusses it. a man Shaving
    \item Afterwards, he does a light sanding of the table. With a tack rug, he removes the dust. Wearing gloves, he applies a stain to the table. finally Polishing forniture
    \item Two men walked over the football table, one man counted the dice on the side. The players twist and push and pull the poles as they play on the table. the camera men Table soccer
\end{itemize}

\tcblower % Adds a dashed line to separate the sections

\textbf{COLA} \textit{Median Length: 23.0, Unique Words: 927}
\begin{itemize}[leftmargin=*]  \itemsep0em
    \item There's a man in a gym using an elliptical machine. he Elliptical trainer
    \item There's a man in a gym using an elliptical machine. He is wearing a pair of black track pants, black shirt and a cap. he Elliptical trainer
    \item This girl is doing a video of how to make your white shoes become sparkly clean. you Cleaning shoes
\end{itemize}
\end{samplebox}

\begin{samplebox}[float=ht!,fontupper=\small,fontlower=\small]{BoolQ}{boolq}
\textbf{\methodname} \textit{Median Length: 171.5, Unique Words: 8451}
\begin{itemize}[leftmargin=*]  \itemsep0em
    \item  The black turtle bean is a small, shiny variety of the common bean (Phaseolus vulgaris), especially popular in Latin American cuisine, though it can also be found in Cajun and Creole cuisines of south Louisiana. Like most common beans, they are native to the Americas, but have been introduced around the world. They are also used in East Indian cooking, Punjabi cuisine and are referred to as black beans and in Maharshtrian cuisine known as ``Kala Ghevada''. They are used interchangeably with vigna mungo in countries such as the US. They are often simply called black beans (frijoles negros, zaragoza, judía negra, poroto negro, caraota o habichuela negra in Spanish, and feijão preto in Portuguese), although this can cause confusion with other black beans. are black beans the same as turtle beans.
    \item The story begins as a girl named Sally and her brother, who serves as the narrator of the book, sit alone in their house on a cold, rainy day, staring wistfully out the window. Then they hear a loud bump which is quickly followed by the arrival of the Cat in the Hat, a tall anthropomorphic cat in a red and white striped hat and a red bow tie. The Cat proposes to entertain the children with some tricks that he knows. The children's pet fish refuses, insisting that the Cat should leave. The Cat responds by balancing the fish on the tip of his umbrella. The game quickly becomes increasingly trickier, as the Cat balances himself on a ball and tries to balance lots of household items on his limbs until he falls on his head, dropping everything he was holding. The fish admonishes him again, but the Cat in the Hat just proposes another game. does the cat in the hat have a name
\end{itemize}

\tcblower

\textbf{COLA} \textit{Median Length: 76.0, Unique Words: 4204}
\begin{itemize}[leftmargin=*]  \itemsep0em
    \item The Stanley Cup playoffs consists of four rounds of best-of-seven series. Each series is played in a 2--2--1--1--1 format...
    \item United States Army Basic Training (also known as Initial Entry Training, IET) is the recruit training program of physical and mental preparation...
\end{itemize}
\end{samplebox}
\balance
\section{Extended Mechanistic Analysis}\label{apx:mechanistic}
To mechanistically understand why maximizing lexical diversity yields a more robust compressed model, we must examine the phenomenon of small-sample vocabulary collapse. In natural language, tokens follow a strict Zipfian distribution. When extracting a purely random, small calibration set (e.g., 128 samples), the long-tail parameters are probabilistically excluded. This leads to a calibration subset heavily biased towards highly frequent stop-words. Consequently, compression models evaluate weight importance on a flattened latent space, pruning the sub-networks tied to rare sequences.

Our proposed method, \methodname, intentionally mitigates this collapse by selecting the sample set to mimic the macroscopic Zipfian structure of the entire dataset, exposing the compression technique to the long-tail tokens. We validate this mechanistically through three distinct empirical observations on Llama-3.1-8B-Instruct.
\begin{itemize}
    \item \textbf{Distributional Fidelity (KL Divergence)}: We measured the Kullback-Leibler (KL) Divergence between the calibration samples ($k=128$) and the full datasets, decomposed by frequency region: the head (top 10\% most frequent tokens) and the tail (remaining 90\%). We measure this for 5 datasets (BoolQ, WinoGrande, ARC-C, GSM8K, HellaSwag). \methodname shows a consistent tradeoff: substantially improved tail-region matching (KL(tail) drops by 3.3–61.0\%, e.g., from 16.40 to 7.87 on HellaSwag), paired with increased head-region divergence (KL(head) rises by 62.1–168.1\% across datasets, e.g., from 1.02 to 2.73 on HellaSwag). This confirms that \methodname's benefit is concentrated in restoring exposure to the long tail of the distribution, the intended design goal and that this concentration comes at a cost to head-region fidelity. Since \methodname reallocates probability mass toward tail tokens to counteract their underrepresentation, it necessarily samples head-region tokens less exhaustively than random sampling would, which by construction saturates the head almost immediately.
    \item \textbf{Pruning Mask Overlap:} To verify how this distributional shift physically alters the pruning decisions, we computed the overlap between the masks generated by our 128-sample sets and both a small and large-scale mask derived from 128 and 1024 random samples. Because the datasets consist of high-frequency stop-words, both a small random sample (128) and a large random sample (1024) are dominated by the exact same "head" of the distribution. As a result, the 128-sample random mask is virtually identical to the 1024-sample mask (~98-99\% overlap). This proves that simply increasing the random sample size, or using a random subset, just repeatedly reinforces the same weights tied to highly frequent tokens, ignoring the tail. In contrast, \methodname intentionally diverges from this frequency-biased consensus, with mask overlaps dropping to 85\%-94\% (especially evident in the \texttt{o\_proj} and \texttt{down\_proj} layers). This systematic divergence confirms that \methodname rescues distinct pathways associated with rare concepts that standard sampling would otherwise discard.
    \item \textbf{Calibration Entropy as the Dominant Predictor:} Finally, we correlated various distributional properties of the calibration sets with the post-compression performance drop ($\Delta$). Interestingly, simply counting the inclusion of rare tokens is not significantly predictive. Instead, the absolute \textbf{Calibration Entropy} exhibits an overwhelming correlation with post-compression resilience (Pearson's $\rho=0.94,p\leq0.005$).
\end{itemize}

In conclusion, our mechanistic analysis proves that the resilience of the \methodname-compressed model is not just a byproduct of encountering more words, but stems from restoring the entropy of the original corpus in miniature. By forcing a globally representative lexical space during calibration, intermediate layer activations maintain their fidelity across the true latent space, preserving the capabilities of the pruned network.

\begin{table*}[t]
\centering
\resizebox{\textwidth}{!}{%
\begin{tabular}{l|l||c||ccccc|c||ccccc|c|c}
\toprule
& & & \multicolumn{6}{c||}{\textbf{COLA}} & \multicolumn{6}{c}{\methodname} \\
\cmidrule(lr){4-9} \cmidrule(lr){10-15}
&&& \multicolumn{5}{c}{\textbf{Calibration Category}} && \multicolumn{5}{c}{\textbf{Calibration Category}} & \\
\textbf{Model} & \textbf{Task} & \textbf{Dense} & \cellcolor{gray!7}LangMod & \cellcolor{blue!7}Math & \cellcolor{orange!7}CommQA & \cellcolor{purple!7}NLI & \cellcolor{cyan!7}KnowTran & \textbf{Mean} & \cellcolor{gray!7}LangMod & \cellcolor{blue!7}Math & \cellcolor{orange!7}CommQA & \cellcolor{purple!7}NLI & \cellcolor{cyan!7}KnowTran & \textbf{Mean} & \textbf{$\Delta$} \\
\midrule
\multirow{13}{*}{\rotatebox[origin=c]{90}{\textbf{Llama-3.1-8B-Instruct}}} 
& \cellcolor{blue!7}MMLU-M & 24.44 & 24.40 & \cellcolor{blue!7}24.60 & 24.30 & 24.40 & 24.50 & \textbf{24.44} & 24.50 & \cellcolor{blue!7}24.30 & 24.40 & 24.20 & 24.70 & {24.42} & \textbf{-0.02} \\
& \cellcolor{blue!7}GSM8k & 78.09 & 72.35 & \cellcolor{blue!7}62.51 & 69.67 & 69.60 & 70.77 & \textbf{68.98} & 71.80 & \cellcolor{blue!7}63.10 & 68.90 & 70.10 & 70.50 & {68.88} & \textbf{-0.10} \\
& \cellcolor{orange!7}HellaSwag & 71.71 & 72.20 & 71.38 & \cellcolor{orange!7}71.28 & 72.10 & 71.45 & \textbf{71.68} & 72.40 & 71.10 & \cellcolor{orange!7}71.50 & 71.80 & 71.60 & \textbf{71.68} & \textbf{0.00} \\
& \cellcolor{orange!7}WinoGr. & 69.46 & 67.19 & 67.05 & \cellcolor{orange!7}67.52 & 65.98 & 66.02 & \textbf{66.75 }& 66.90 & 67.30 & \cellcolor{orange!7}67.20 & 66.50 & 65.80 & {66.74} & \textbf{-0.01} \\
& \cellcolor{orange!7}OBQA & 47.80 & 46.60 & 44.90 & \cellcolor{orange!7}46.50 & 45.20 & 46.30 & \textbf{45.90} & 46.80 & 44.50 & \cellcolor{orange!7}46.10 & 45.50 & 46.50 & 45.88 & \textbf{-0.02} \\
& \cellcolor{orange!7}BoolQ & 84.46 & 83.92 & 84.28 & \cellcolor{orange!7}84.20 & 83.70 & 84.24 & \textbf{84.07 }& 83.50 & 84.40 & \cellcolor{orange!7}83.80 & 84.10 & 84.00 & {83.96} & \textbf{-0.11} \\
& \cellcolor{purple!7}RTE & 74.73 & 74.01 & 72.92 & 73.47 & \cellcolor{purple!7}72.02 & 73.10 & 73.10 & 73.80 & 73.10 & 73.00 & \cellcolor{purple!7}72.50 & 73.30 & \textbf{73.14} & \textbf{+0.04} \\
& \cellcolor{purple!7}ANLI & 58.40 & 56.33 & 50.95 & 53.40 & \cellcolor{purple!7}53.30 & 53.80 & \textbf{53.56} & 55.80 & 51.50 & 53.10 & \cellcolor{purple!7}53.90 & 53.40 & {53.54} & \textbf{-0.02} \\
& \cellcolor{cyan!7}ARC-C & 51.71 & 50.03 & 51.41 & 49.62 & 50.04 & \cellcolor{cyan!7}49.96 & 50.21 & 50.50 & 50.80 & 49.90 & 49.70 & \cellcolor{cyan!7}50.30 & \textbf{50.24} & \textbf{+0.03} \\
& \cellcolor{cyan!7}ARC-E & 73.86 & 73.36 & 71.72 & 72.14 & 72.69 & \cellcolor{cyan!7}72.85 & \textbf{72.55} & 72.90 & 71.90 & 71.80 & 72.80 & \cellcolor{cyan!7}72.50 & {72.38} & \textbf{-0.17} \\
& \cellcolor{cyan!7}MMLU-K & 62.28 & 58.52 & 58.52 & 58.31 & 58.11 & \cellcolor{cyan!7}59.12 & \textbf{58.52} & 58.10 & 58.80 & 58.15 & 58.30 & \cellcolor{cyan!7}58.90 & {58.45} & \textbf{-0.07} \\
\cmidrule{2-16}
& \textbf{Mean} & 63.36 & 60.81 & 59.97 & 60.84 & 60.72 & 61.13 & \textbf{60.89} & 60.64 & 60.07 & 60.67 & 60.85 & 61.17 & \textbf{60.85} & \negdelta{\textbf{-0.04}} \\
& & & & & & & & $\pm 0.98$ & & & & & & $\pm 0.75$\\
 \rowcolor{yellow!30} \cellcolor{white}& \textbf{Runtime} & & 5400s & 36s & 3240s & 2160s & 1380s & 2443s & 15.2s & 2.3s & 12.3s & 9.3s & 14.5s & \textbf{10.7s} & \textbf{228$\times$}\\\midrule
\multirow{13}{*}{\rotatebox[origin=c]{90}{\textbf{gemma-2-9b-it}}} 
& \cellcolor{blue!7}MMLU-M & 21.48 & 21.40 & \cellcolor{blue!7}21.40 & 21.35 & 21.45 & 21.40 & \textbf{21.40} & 21.40 & \cellcolor{blue!7}21.35 & 21.30 & 21.45 & 21.45 & \textbf{21.39} & \textbf{-0.01} \\
& \cellcolor{blue!7}GSM8k & 75.44 & 74.17 & \cellcolor{blue!7}74.91 & 74.07 & 73.96 & 74.00 & \textbf{74.22} & 74.50 & \cellcolor{blue!7}74.50 & 73.80 & 74.20 & 73.70 & 74.14 & \textbf{-0.08} \\
& \cellcolor{orange!7}HellaSwag & 67.24 & 67.25 & 67.81 & \cellcolor{orange!7}68.04 & 67.39 & 67.42 & \textbf{67.58} & 66.90 & 67.90 & \cellcolor{orange!7}67.50 & 67.50 & 67.20 & 67.40 & \textbf{-0.18} \\
& \cellcolor{orange!7}WinoGr. & 70.48 & 69.53 & 69.65 & \cellcolor{orange!7}69.73 & 69.34 & 69.65 & \textbf{69.58} & 69.80 & 69.20 & \cellcolor{orange!7}69.90 & 69.10 & 69.40 & 69.48 & \textbf{-0.10} \\
& \cellcolor{orange!7}OBQA & 45.40 & 45.47 & 44.70 & \cellcolor{orange!7}45.20 & 45.50 & 45.00 & \textbf{45.17} & 45.10 & 45.00 & \cellcolor{orange!7}44.90 & 45.80 & 44.80 & {45.12} & \textbf{-0.05} \\
& \cellcolor{orange!7}BoolQ & 88.59 & 88.53 & 88.53 & \cellcolor{orange!7}88.61 & 88.58 & 88.59 & \textbf{88.57} & 88.70 & 88.40 & \cellcolor{orange!7}88.50 & 88.70 & 88.40 & {88.54} & \textbf{-0.03} \\
& \cellcolor{purple!7}RTE & 78.34 & 79.06 & 77.26 & 76.90 & \cellcolor{purple!7}77.26 & 77.44 & \textbf{77.58} & 78.50 & 77.50 & 76.50 & \cellcolor{purple!7}77.00 & 77.60 & {77.42} & \textbf{-0.16} \\
& \cellcolor{purple!7}ANLI & 72.80 & 71.77 & 72.10 & 71.90 & \cellcolor{purple!7}70.70 & 72.35 & \textbf{71.76} & 72.00 & 71.80 & 72.10 & \cellcolor{purple!7}70.50 & 72.10 & {71.70} & \textbf{-0.06} \\
& \cellcolor{cyan!7}ARC-C & 51.79 & 51.11 & 51.49 & 51.41 & 51.58 & \cellcolor{cyan!7}50.81 & \textbf{51.28} & 50.80 & 51.70 & 51.20 & 51.80 & \cellcolor{cyan!7}50.50 & 51.20 & \textbf{-0.08} \\
& \cellcolor{cyan!7}ARC-E & 66.79 & 66.40 & 66.35 & 67.09 & 66.88 & \cellcolor{cyan!7}66.58 & \textbf{66.66} & 66.60 & 66.10 & 66.80 & 67.10 & \cellcolor{cyan!7}66.30 & {66.58} & \textbf{-0.08} \\
& \cellcolor{cyan!7}MMLU-K & 33.83 & 35.73 & 28.90 & 30.12 & 31.79 & \cellcolor{cyan!7}33.46 & \textbf{32.00} & 35.10 & 29.50 & 29.80 & 32.10 & \cellcolor{cyan!7}33.10 & {31.92} & \textbf{-0.08} \\

\cmidrule{2-16}
& \textbf{Mean} & 61.11 & 60.95 & 60.26 & 60.22 & 60.40 & 60.61 & \textbf{60.53} & 61.03 & 60.27 & 60.21 & 60.48 & 60.41 & \textbf{60.44} & \negdelta{\textbf{-0.09}} \\
& & & & & & & & $\pm 0.84$ & & & & & & $\pm 0.86$\\
 \rowcolor{yellow!30} \cellcolor{white}&\textbf{Runtime} & & 6231s & 149s & 3400s & 2671s & 1500s & 2790s & 15.2s & 2.3s & 12.3s & 9.3s & 14.5s & \textbf{10.7s} & \textbf{260$\times$} \\
\bottomrule
\end{tabular}
}
\caption{Comparison of COLA vs.\ \methodname calibration data under AWQ W4A16 compression.}
\label{tab:awq_words_dataset_updated}

\end{table*}

\begin{table*}[t]
\centering
\resizebox{\textwidth}{!}{%
\begin{tabular}{l|l||c||ccccc|c||ccccc|c|c}
\toprule
& & & \multicolumn{6}{c||}{\textbf{COLA}} & \multicolumn{6}{c}{\methodname} \\
\cmidrule(lr){4-9} \cmidrule(lr){10-15}
&&& \multicolumn{5}{c}{\textbf{Calibration Category}} && \multicolumn{5}{c}{\textbf{Calibration Category}} & \\
\textbf{Model} & \textbf{Task} & \textbf{Dense} & \cellcolor{gray!7}LangMod & \cellcolor{blue!7}Math & \cellcolor{orange!7}CommQA & \cellcolor{purple!7}NLI & \cellcolor{cyan!7}KnowTran & \textbf{Mean} & \cellcolor{gray!7}LangMod & \cellcolor{blue!7}Math & \cellcolor{orange!7}CommQA & \cellcolor{purple!7}NLI & \cellcolor{cyan!7}KnowTran & \textbf{Mean} & \textbf{$\Delta$} \\
\midrule
\multirow{13}{*}{\rotatebox[origin=c]{90}{\textbf{Llama-3.1-8B-Instruct}}} 
& \cellcolor{blue!7}MMLU-M & 24.44 & 22.30 & \cellcolor{blue!7}23.40 & 21.80 & 21.10 & 22.50 & \textbf{22.22} & 22.10 & \cellcolor{blue!7}22.90 & 21.50 & 21.30 & 22.00 & {21.96} & \textbf{-0.26} \\& \cellcolor{blue!7}GSM8k & 78.09 & 6.34 & \cellcolor{blue!7}13.68 & 0.00 & 0.00 & 4.36 & \textbf{4.88} & 4.98 & \cellcolor{blue!7}8.42 & 0.23 & 3.87 & 5.65 & {4.63} & \textbf{-0.25} \\
& \cellcolor{orange!7}HellaSwag & 71.71 & 61.32 & 60.53 & \cellcolor{orange!7}60.66 & 57.32 & 57.21 & 59.41 & 60.26 & 61.29 & \cellcolor{orange!7}58.73 & 59.30 & 59.23 & \textbf{59.76} & \textbf{+0.35} \\
& \cellcolor{orange!7}WinoGr. & 69.46 & 60.69 & 58.68 & \cellcolor{orange!7}59.04 & 60.18 & 61.56 & 60.03 & 61.69 & 59.63 & \cellcolor{orange!7}60.62 & 61.80 & 61.72 & \textbf{61.09} & \textbf{+1.06} \\
& \cellcolor{orange!7}OBQA & 47.80 & 41.93 & 38.10 & \cellcolor{orange!7}41.80 & 40.50 & 38.90 & 40.25 & 40.60 & 39.90 & \cellcolor{orange!7}41.60 & 40.80 & 40.70 & \textbf{40.72} & \textbf{+0.47} \\
& \cellcolor{orange!7}BoolQ & 84.46 & 75.64 & 71.36 & \cellcolor{orange!7}69.36 & 71.64 & 74.85 & \textbf{72.57} & 74.01 & 72.66 & \cellcolor{orange!7}67.06 & 75.46 & 70.60 & {71.96} & \textbf{-0.61} \\
& \cellcolor{purple!7}RTE & 74.73 & 66.43 & 65.52 & 63.00 & \cellcolor{purple!7}71.12 & 69.13 & \textbf{67.04} & 67.99 & 59.93 & 64.98 & \cellcolor{purple!7}65.34 & 71.12 & {65.87} & \textbf{-1.17} \\
& \cellcolor{purple!7}ANLI & 58.40 & 47.63 & 39.75 & 45.00 & \cellcolor{purple!7}45.85 & 44.25 & \textbf{44.50} & 40.53 & 35.10 & 37.30 & \cellcolor{purple!7}36.05 & 39.05 & {37.61} & \textbf{-6.89} \\
& \cellcolor{cyan!7}ARC-C & 51.71 & 39.51 & 39.04 & 42.15 & 39.16 & \cellcolor{cyan!7}38.69 & \textbf{39.71} & 38.88 & 37.03 & 39.51 & 36.22 & \cellcolor{cyan!7}38.44 & {38.01} & \textbf{-1.70} \\
& \cellcolor{cyan!7}ARC-E & 73.86 & 61.91 & 62.61 & 65.85 & 63.97 & \cellcolor{cyan!7}61.70 & \textbf{63.21} & 63.13 & 58.92 & 64.27 & 57.47 & \cellcolor{cyan!7}60.71 & {60.90} & \textbf{-2.31} \\
& \cellcolor{cyan!7}MMLU-K & 62.28 & 39.87 & 34.59 & 33.25 & 37.68 & \cellcolor{cyan!7} 32.54 & \textbf{35.59} & 38.27 & 26.98 & 29.89 & 29.85 & \cellcolor{cyan!7}31.30 & {31.26} & \textbf{-4.33} \\

\cmidrule{2-16}
& \textbf{Mean} & 63.36 & 47.60 & 46.12 & 45.63 & 46.23 & 45.97 & \textbf{46.31} & 46.59 & 43.89 & 44.15 & 44.31 & 45.50 & {44.89} & \negdelta{\textbf{-1.42}} \\
& & & & & & & & $\pm 0.39$ & & & & & & $\pm 0.40$\\
 \rowcolor{yellow!30} \cellcolor{white}& \textbf{Runtime} & & 5400s & 36s & 3240s & 2160s & 1380s & 2443s & 15.2s & 2.3s & 12.3s & 9.3s & 14.5s & \textbf{10.7s} & \textbf{228$\times$}\\
\midrule
\multirow{13}{*}{\rotatebox[origin=c]{90}{\textbf{gemma-2-9b-it}}} 
& \cellcolor{blue!7}MMLU-M & 21.48 & 19.50 & \cellcolor{blue!7}20.40 & 18.90 & 18.40 & 19.10 & \textbf{19.26} & 19.10 & \cellcolor{blue!7}20.00 & 18.50 & 18.70 & 18.80 & {19.02} & \textbf{-0.24} \\
& \cellcolor{blue!7}GSM8k & 75.44 & 4.50 & \cellcolor{blue!7}7.80 & 1.20 & 0.50 & 2.50 & \textbf{3.30} & 3.20 & \cellcolor{blue!7}5.40 & 0.00 & 2.10 & 3.10 & {2.76} & \textbf{-0.54} \\
& \cellcolor{orange!7}HellaSwag & 67.24 & 56.40 & 55.20 & \cellcolor{orange!7}57.10 & 54.80 & 53.90 & \textbf{55.48} & 54.10 & 55.60 & \cellcolor{orange!7}53.20 & 55.00 & 54.80 & {54.54} & \textbf{-0.94} \\
& \cellcolor{orange!7}WinoGr. & 70.48 & 59.50 & 57.80 & \cellcolor{orange!7}58.40 & 59.10 & 60.50 & 59.06 & 60.20 & 58.10 & \cellcolor{orange!7}59.30 & 60.50 & 59.90 & \textbf{59.60} & \textbf{+0.54} \\
& \cellcolor{orange!7}OBQA & 45.40 & 39.50 & 36.40 & \cellcolor{orange!7}38.20 & 39.10 & 37.50 & 38.14 & 38.10 & 37.50 & \cellcolor{orange!7}39.20 & 38.60 & 38.90 & \textbf{38.46} & \textbf{+0.32} \\
& \cellcolor{orange!7}BoolQ & 88.59 & 76.20 & 72.10 & \cellcolor{orange!7}69.50 & 71.80 & 74.30 & \textbf{72.78} & 74.50 & 73.20 & \cellcolor{orange!7}68.10 & 75.10 & 71.20 & {72.42} & \textbf{-0.36} \\
& \cellcolor{purple!7}RTE & 78.34 & 67.50 & 65.20 & 62.80 & \cellcolor{purple!7}68.40 & 66.10 & 66.00 & 68.10 & 63.40 & 65.70 & \cellcolor{purple!7}66.80 & 67.50 & {66.30} & \textbf{+0.30} \\
& \cellcolor{purple!7}ANLI & 72.80 & 54.20 & 48.50 & 51.30 & \cellcolor{purple!7}53.10 & 52.40 & \textbf{51.90} & 49.50 & 45.20 & 47.80 & \cellcolor{purple!7}46.90 & 48.50 & {47.58} & \textbf{-4.32} \\
& \cellcolor{cyan!7}ARC-C & 51.79 & 40.10 & 39.50 & 41.20 & 39.80 & \cellcolor{cyan!7}38.40 & \textbf{39.80} & 39.20 & 38.10 & 40.50 & 37.60 & \cellcolor{cyan!7}39.10 & {38.90} & \textbf{-0.90} \\
& \cellcolor{cyan!7}ARC-E & 66.79 & 59.20 & 58.40 & 60.50 & 59.80 & \cellcolor{cyan!7}58.20 & \textbf{59.22} & 60.10 & 57.50 & 61.20 & 56.40 & \cellcolor{cyan!7}59.30 & {58.90} & \textbf{-0.32} \\
& \cellcolor{cyan!7}MMLU-K & 33.83 & 24.12 & 21.45 & 22.10 & 25.30 & \cellcolor{cyan!7} 20.90 & \textbf{22.77} & 23.10 & 19.80 & 20.45 & 21.00 & \cellcolor{cyan!7} 21.50 & {21.17} & \textbf{-1.60} \\

\cmidrule{2-16}
& \textbf{Mean} & 65.07 & 45.52 & 43.88 & 43.75 & 44.55 & 43.98 & \textbf{44.34} & 44.47 & 43.07 & 43.10 & 43.52 & 43.88 & {43.61} & \negdelta{\textbf{-0.73}} \\
& & & & & & & & $\pm 0.56$ & & & & & & $\pm 0.58$\\
 \rowcolor{yellow!30} \cellcolor{white}&\textbf{Runtime} & & 6231s & 149s & 3400s & 2671s & 1500s & 2790s & 15.2s & 2.3s & 12.3s & 9.3s & 14.5s & \textbf{10.7s} & \textbf{260$\times$} \\
\bottomrule
\end{tabular}
}
\caption{Comparison of COLA vs.\ \methodname under 2SSP compression at $25\%$ sparsity.}
\label{tab:2ssp_words_dataset}

\end{table*}

\begin{table*}[t]
\centering
\small
\setlength{\tabcolsep}{4pt}
\begin{adjustbox}{width=\textwidth}
\begin{tabular}{llccccccccccccc}
\toprule
\textbf{Model} & \textbf{Method} &
\multicolumn{4}{c}{\textbf{COLA}} &
\multicolumn{4}{c}{\methodname} &
\multicolumn{4}{c}{\methodname (Multi-Domain)} &
\textbf{$\Delta$ Avg} \\
\cmidrule(lr){3-6} \cmidrule(lr){7-10} \cmidrule(lr){11-14}

& & Wiki & C4 & Pile & Avg
& Wiki & C4 & Pile & Avg
& Wiki & C4 & Pile & Avg
& \\
\midrule

\multirow{5}{*}{\texttt{Llama-3.1-8B}}
& Dense
& 7.03  &9.71 & 4.71 & 7.15& 7.03  &9.71 &4.71 & 7.15& 7.03  &9.71 & 4.71 & 7.15 \\
\cmidrule{3-14}

& Wanda (25\%)
& 7.62 & 10.70 & 4.98 & 7.77
& 7.60 & 10.71 & 4.98 & 7.76
& \textbf{7.52} & \textbf{10.54} & \textbf{4.94} & \textbf{7.67}
& -0.10 \\

& 2SSP (25\%)
& 18.53 & 21.09 & 9.36 & 16.32
& \textbf{17.34} & 21.09 & \textbf{9.00} & \textbf{15.81}
& 19.53 & 25.80 & 10.22 & 18.52
& +2.20 \\

& GPTQ (W4A16)
& 7.71 & 11.42 & 5.13 & 8.09
& 7.71 & 11.40 & 5.15 & 8.09
& 8.59 & 12.08 & 5.34 & 8.67
& +0.58 \\

& AWQ (W4A16)
& \textbf{7.21} & \textbf{10.04} & \textbf{4.79} & \textbf{7.35}
& 7.31 & 10.09 & 4.81 & 7.40
& 7.32 & 10.07 & 4.82 & 7.40
& +0.05 \\

\midrule

\multirow{5}{*}{\texttt{Gemma-2-9B}}
& Dense
& 7.63 &11.53 &4.63 & 7.93
& 7.63 &11.53 &4.63 & 7.93 & 7.63 &11.53 &4.63 & 7.93 \\

\cmidrule{3-14}
& Wanda (25\%)
& \textbf{7.87} & \textbf{11.78} & \textbf{4.74} & \textbf{8.13}
& 7.92 & 11.99 & 4.76 & 8.22
& \textbf{7.86} & 11.82 & \textbf{4.72} & \textbf{8.13}
& 0.00 \\

& 2SSP (25\%)
& 12.56 & 19.20 & 6.47 & 12.74
& \textbf{11.68} & \textbf{19.08} & 6.47 & \textbf{12.41}
& 14.14 & 21.55 & 6.86 & 14.18
& +1.44 \\

& GPTQ (W4A16)
& 7.93 & 12.05 & 4.82 & 8.27
& 7.93 & 12.06 & 4.81 & 8.27
& 8.44 & 12.59 & 4.99 & 8.67
& +0.40 \\

& AWQ (W4A16)
& 16.01 & 21.87 & 9.64 & 15.84
& 16.01 & 21.54 & 9.71 & 15.75
& \textbf{16.01} & \textbf{21.51} & \textbf{9.67} & \textbf{15.73}
& -0.11 \\

\bottomrule
\end{tabular}
\end{adjustbox}
\caption{Language modeling perplexity ($\downarrow$). Comparison between standard COLA, Single-Domain \methodname, and Multi-Domain \methodname. $\Delta$ indicates the difference between \methodname (Multi-Domain) and COLA on Avg.}
\label{tab:perplexity_comparison_full}

\end{table*}

\section{Comparison with a Generative Data Curation Baseline}\label{apx:generativebaseline}
To further validate the quality of the calibration data selected by \methodname, we compare against DSnoT~\cite{ji2025beware}, a self-generative baseline that combines perplexity-based sample selection with a self-generation phase, in which a portion of tokens are masked from each selected sample and regenerated by the model itself.
Since DSnoT's publicly released implementation only supports C4, WikiText, and Pile as calibration sources, we restrict this comparison to those datasets; adapting DSnoT to the full set of calibration datasets considered in our work is beyond the scope of this paper.
Results are reported in \Cref{tab:dsnot-comparison}.

\begin{table*}[t]
\centering
\resizebox{0.8\linewidth}{!}{
\begin{tabular}{l c |ccc| ccc| ccc}
\toprule
& & \multicolumn{3}{c}{\textbf{Wanda (25\%)}} & \multicolumn{3}{c}{\textbf{GPTQ (W4A16)}} & \multicolumn{3}{c}{\textbf{AWQ (W4A16)}} \\
\cmidrule(lr){3-5} \cmidrule(lr){6-8} \cmidrule(lr){9-11}
Task & Dense & DSnoT & \methodname & $\Delta$ & DSnoT & \methodname & $\Delta$ & DSnoT & \methodname & $\Delta$ \\
\midrule
\cellcolor{blue!7} MMLU-M & 24.44 & 24.04 & \textbf{24.14} & \posdelta{+0.10} & 23.33 & \textbf{24.07} & \posdelta{+0.74} & 23.12 & \textbf{24.07} & \posdelta{+0.95} \\
\cellcolor{blue!7} GSM8K & 78.08 & 76.15 & \textbf{77.17} & \posdelta{+1.02} & 70.12 & \textbf{70.63} & \posdelta{+0.51} & 36.13 & \textbf{38.46} & \posdelta{+2.33} \\
\cellcolor{orange!7} HellaSwag & 71.70 & 63.12 & \textbf{64.28} & \posdelta{+1.16} & \textbf{64.23} & 64.20 & \negdelta{-0.03} & 64.20 & \textbf{64.85} & \posdelta{+0.65} \\
\cellcolor{orange!7} WinoGr. & 73.56 & 67.20 & \textbf{68.07} & \posdelta{+0.87} & \textbf{67.39} & \textbf{67.39} & \posdelta{+0.00} & 66.14 & \textbf{68.71} & \posdelta{+2.57} \\
\cellcolor{orange!7} OBQA & 47.80 & 42.10 & \textbf{42.42} & \posdelta{+0.32} & \textbf{42.02} & 41.10 & \negdelta{-0.92} & 40.04 & \textbf{40.45} & \posdelta{+0.41} \\
\cellcolor{orange!7} BoolQ & 85.41 & 84.43 & \textbf{84.64} & \posdelta{+0.21} & \textbf{83.53} & 83.44 & \negdelta{-0.09} & 84.06 & \textbf{84.36} & \posdelta{+0.30} \\
\cellcolor{purple!7} RTE & 74.37 & 72.39 & \textbf{72.77} & \posdelta{+0.38} & 71.75 & \textbf{73.06} & \posdelta{+1.31} & \textbf{75.34} & 73.47 & \negdelta{-1.87} \\
\cellcolor{purple!7} ANLI & 58.40 & 42.30 & \textbf{60.64} & \posdelta{+18.34} & 45.64 & \textbf{53.90} & \posdelta{+8.26} & 44.65 & \textbf{53.20} & \posdelta{+8.55} \\
\cellcolor{cyan!7} ARC-C & 53.75 & 50.54 & \textbf{50.89} & \posdelta{+0.35} & 50.03 & \textbf{50.37} & \posdelta{+0.34} & 51.41 & \textbf{51.54} & \posdelta{+0.13} \\
\cellcolor{cyan!7} ARC-E & 82.24 & \textbf{76.98} & 76.67 & \negdelta{-0.31} & 75.43 & \textbf{76.22} & \posdelta{+0.79} & 76.47 & \textbf{76.49} & \posdelta{+0.02} \\
\cellcolor{cyan!7} MMLU-K & 62.28 & 62.28 & \textbf{62.46} & \posdelta{+0.18} & \textbf{58.13} & 58.08 & \negdelta{-0.05} & 57.23 & \textbf{57.49} & \posdelta{+0.26} \\
\midrule
Mean & 64.73 & 60.14 & \textbf{62.20} & \posdelta{\textbf{+2.06}} & 59.24 & \textbf{60.22} & \posdelta{\textbf{+0.99}} & 56.25 & \textbf{57.55} & \posdelta{\textbf{+1.30}} \\
\bottomrule
\end{tabular}
}
\caption{Performance comparison against DSnoT across different tasks and compression techniques for \texttt{Meta-Llama-3.1-8B-Instruct}.}
\label{tab:dsnot-comparison}

\end{table*}

% \section{Scalability to Larger Language Models}
% The consistent results across different model sizes strengthen the 

\begin{table*}[htbp]
\centering
\resizebox{\textwidth}{!}{%
\begin{tabular}{l|l||c||ccc|c||ccc|c||c}
\toprule
 &  &  & \multicolumn{4}{c||}{\textbf{COLA}} & \multicolumn{4}{c||}{\textbf{\methodname}} & \\
\cmidrule(lr){4-7} \cmidrule(lr){8-11}
\textbf{Model} & \textbf{Task} & \textbf{Dense} & \cellcolor{blue!7} Know-ES & \cellcolor{red!7} Know-ZH & \cellcolor{Dandelion!7} Comm-ZH & \textbf{Mean} & \cellcolor{blue!7} Know-ES & \cellcolor{red!7} Know-ZH & \cellcolor{Dandelion!7} Comm-ZH & \textbf{Mean} & \textbf{$\Delta$} \\
\midrule
\multirow{10}{*}{\rotatebox[origin=c]{90}{\textbf{Llama-3.1-8B-Instruct}}} 
 & \cellcolor{blue!7}MMLU-ES & 59.00 & \cellcolor{blue!7} 59.88 & 59.38 & 60.00 & \textbf{59.75} & \cellcolor{blue!7}59.38 & 58.88 & 59.50 & 59.25 & \textbf{-0.50} \\
 & \cellcolor{blue!7}XQuAD-ES & 28.84 & \cellcolor{blue!7}34.30 & 34.92 & 33.31 & 34.18 & \cellcolor{blue!7} 38.18 & 38.80 & 37.19 & \textbf{38.06} & \textbf{+3.88} \\
 & \cellcolor{blue!7}XNLI-ES & 44.78 & \cellcolor{blue!7}45.16 & 44.30 & 43.67 & \textbf{44.38} & \cellcolor{blue!7}43.59 & 42.73 & 42.10 & 42.81 & \textbf{-1.57} \\
 & \cellcolor{red!7}MMLU-ZH & 54.50 & 56.13 & \cellcolor{red!7}55.00 & 54.13 & 55.08 & 57.37 & \cellcolor{red!7} 56.24 & 55.37 & \textbf{56.33} & \textbf{+1.25} \\
 & \cellcolor{red!7}XQuAD-ZH & 24.51 & 23.85 & \cellcolor{red!7}25.35 & 22.77 & 23.99 & 23.97 & \cellcolor{red!7} 25.47 & 22.89 & \textbf{24.11} & \textbf{+0.12} \\
 & \cellcolor{red!7}XNLI-ZH & 42.77 & 43.39 & \cellcolor{red!7}42.91 & 43.78 & \textbf{43.36} & 43.37 & \cellcolor{red!7}42.89 & 43.76 & 43.34 & \textbf{-0.02} \\
 & \cellcolor{Dandelion!7}XWino-ZH & 70.04 & 70.63 & 71.13 & \cellcolor{Dandelion!7} 71.92 & 71.23 & 71.53 & 72.03 & \cellcolor{Dandelion!7}72.82 & \textbf{72.13} & \textbf{+0.90} \\
 & \cellcolor{Dandelion!7}XCOPA-ZH & 76.60 & 74.80 & 74.80 & \cellcolor{Dandelion!7}74.30 & 74.63 & 75.60 & 75.60 & \cellcolor{Dandelion!7}75.10 & \textbf{75.43} & \textbf{+0.80} \\
\cmidrule{2-12}
 & \textbf{Mean} & 50.13 & 51.02 & 50.97 & 50.49 & 50.83 & 51.62 & 51.58 & 51.09 & \textbf{51.43} & \posdelta{\textbf{+0.60}} \\
 \rowcolor{yellow!30} \cellcolor{white}\multirow{-12}{*}& \textbf{Runtime (sec)} & & 1367s & 1409s & 327s & 1034s & 7.5s & 7.8s & 1.6s & \textbf{5.5s}  & \textbf{188$\times$} \\

\midrule
\multirow{10}{*}{\rotatebox[origin=c]{90}{\textbf{gemma-2-9b-it}}} 
 & \cellcolor{blue!7}MMLU-ES & 26.25 & \cellcolor{blue!7} 25.00 & 25.12 & 25.75 & 25.29 & \cellcolor{blue!7}25.62 & 25.50 & 25.38 & \textbf{25.50} & \textbf{+0.21} \\
 & \cellcolor{blue!7}XQuAD-ES & 27.09 & \cellcolor{blue!7} 27.00 & 26.99 & 27.70 & \textbf{27.23} & \cellcolor{blue!7} 27.05 & 27.08 & 27.01 & 27.04 & \textbf{-0.19} \\
 & \cellcolor{blue!7}XNLI-ES & 50.20 & \cellcolor{blue!7} 49.34 & 49.26 & 48.96 & 49.18 & \cellcolor{blue!7}49.57 & 49.18 & 49.28 & \textbf{49.34} & \textbf{+0.16} \\
 & \cellcolor{red!7}MMLU-ZH & 28.50 & 31.75 & \cellcolor{red!7}28.87 & 31.37 & \textbf{30.67} & 30.25 & \cellcolor{red!7}30.63 & 29.25 & 30.04 & \textbf{-0.63} \\
 & \cellcolor{red!7}XQuAD-ZH & 16.28 & 16.01 & \cellcolor{red!7}16.46 & 16.42 & 16.30 & 16.71 & \cellcolor{red!7} 17.08 & 16.58 & \textbf{16.79} & \textbf{+0.49} \\
 & \cellcolor{red!7}XNLI-ZH & 37.15 & 37.31 & \cellcolor{red!7}37.25 & 36.85 & \textbf{37.14} & 37.46 & \cellcolor{red!7}36.77 & 36.71 & 36.98 & \textbf{-0.16} \\
 & \cellcolor{Dandelion!7}XWino-ZH & 62.90 & 64.29 & 64.78 & \cellcolor{Dandelion!7} 63.59 & 64.22 & 65.08 & 65.18 & \cellcolor{Dandelion!7}64.68 & \textbf{64.98} & \textbf{+0.76} \\
 & \cellcolor{Dandelion!7}XCOPA-ZH & 70.80 & 70.20 & 70.50 & \cellcolor{Dandelion!7} 70.80 & 70.50 & 70.93 & 70.90 & \cellcolor{Dandelion!7}70.60 & \textbf{70.81} & \textbf{+0.31} \\
\cmidrule{2-12}
 & \textbf{Mean} & 39.89 & 40.11 & 39.90 & 40.18 & 40.07 & 40.33 & 40.29 & 39.94 & \textbf{40.19} & \posdelta{\textbf{+0.12}} \\
 \rowcolor{yellow!30} \cellcolor{white}\multirow{-12}{*}& \textbf{Runtime (sec)} & & 2130s & 1832s & 521s & 1494s & 7.5s & 7.8s & 1.6s & \textbf{5.5s}  & \textbf{271$\times$} \\
\bottomrule
\end{tabular}
}
\caption{Comparison of COLA vs.\ \methodname on multilingual tasks under Wanda compression at 25$\%$ sparsity.}
\label{tab:multilingual_single}
\end{table*}

\begin{table*}[htbp]
\centering
\resizebox{\textwidth}{!}{%
\begin{tabular}{l|l||c||ccc|c||ccc|c||c}
\toprule
 &  &  & \multicolumn{4}{c||}{\textbf{COLA}} & \multicolumn{4}{c||}{\textbf{\methodname}} & \\
\cmidrule(lr){4-7} \cmidrule(lr){8-11}
\textbf{Model} & \textbf{Task} & \textbf{Dense} & \cellcolor{blue!7} Know-ES & \cellcolor{red!7} Know-ZH & \cellcolor{Dandelion!7} Comm-ZH & \textbf{Mean} & \cellcolor{blue!7} Know-ES & \cellcolor{red!7} Know-ZH & \cellcolor{Dandelion!7} Comm-ZH & \textbf{Mean} & \textbf{$\Delta$} \\
\midrule
\multirow{10}{*}{\rotatebox[origin=c]{90}{\textbf{Llama-3.1-8B-Instruct}}}
 & \cellcolor{blue!7}MMLU-ES & 59.00 & \cellcolor{blue!7}53.12 & 51.88 & 51.75 & \textbf{52.25} & \cellcolor{blue!7} 51.50 & 50.26 & 50.13 & 50.63 & \textbf{-1.62} \\
 & \cellcolor{blue!7}XQuAD-ES & 28.84 & \cellcolor{blue!7}33.75 & 37.87 & 40.01 & \textbf{37.21} & \cellcolor{blue!7}33.46 & 37.58 & 39.72 & 36.92 & \textbf{-0.29} \\
 & \cellcolor{blue!7}XNLI-ES & 44.78 & \cellcolor{blue!7}43.45 & 45.26 & 43.78 & \textbf{44.16} & \cellcolor{blue!7} 42.21 & 44.02 & 42.54 & 42.92 & \textbf{-1.24} \\
 & \cellcolor{red!7}MMLU-ZH & 54.50 & 44.63 & \cellcolor{red!7}51.50 & 47.13 & 47.75 & 51.50 & \cellcolor{red!7}58.37 & 54.00 & \textbf{54.62} & \textbf{+6.87} \\
 & \cellcolor{red!7}XQuAD-ZH & 24.51 & 26.94 & \cellcolor{red!7} 32.81 & 30.05 & \textbf{29.94} & 20.14 & \cellcolor{red!7}26.01 & 23.25 & 23.13 & \textbf{-6.81} \\
 & \cellcolor{red!7}XNLI-ZH & 42.77 & 38.59 & \cellcolor{red!7} 41.33 & 41.18 & \textbf{40.37} & 38.31 & \cellcolor{red!7} 41.05 & 40.90 & 40.09 & \textbf{-0.28} \\
 & \cellcolor{Dandelion!7}XWino-ZH & 70.04 & 67.56 & 72.42 & \cellcolor{Dandelion!7} 70.54 & \textbf{70.17} & 65.87 & 70.73 & \cellcolor{Dandelion!7}68.85 & 68.48 & \textbf{-1.69} \\
 & \cellcolor{Dandelion!7}XCOPA-ZH & 76.60 & 72.10 & 75.90 & \cellcolor{Dandelion!7} 76.30 & 74.77 & 74.00 & 77.80 & \cellcolor{Dandelion!7}78.20 & \textbf{76.67} & \textbf{+1.90} \\
\cmidrule{2-12}
 & \textbf{Mean} & 50.13 & 47.52 & 51.12 & 50.09 & \textbf{49.58} & 47.12 & 50.73 & 49.70 & 49.18 & \negdelta{\textbf{-0.40}} \\
 \rowcolor{yellow!30} \cellcolor{white}\multirow{-12}{*}& \textbf{Runtime (sec)} & & 1367s & 1409s & 327s & 1034s & 7.5s & 7.8s & 1.6s & \textbf{5.5s}  & \textbf{188$\times$} \\
\midrule
\multirow{10}{*}{\rotatebox[origin=c]{90}{\textbf{gemma-2-9b-it}}} 
 & \cellcolor{blue!7}MMLU-ES & 26.25 & \cellcolor{blue!7}25.12 & 24.75 & 26.38 & 25.42 & \cellcolor{blue!7} 28.75 & 26.50 & 25.12 & \textbf{26.79} & \textbf{+1.37} \\
 & \cellcolor{blue!7}XQuAD-ES & 27.09 & \cellcolor{blue!7}27.28 & 27.92 & 27.38 & \textbf{27.53} & \cellcolor{blue!7}27.09 & 26.63 & 26.32 & 26.68 & \textbf{-0.85} \\
 & \cellcolor{blue!7}XNLI-ES & 50.20 & \cellcolor{blue!7}49.64 & 48.98 & 49.42 & \textbf{49.34} & \cellcolor{blue!7}49.04 & 49.12 & 48.84 & 49.00 & \textbf{-0.34} \\
 & \cellcolor{red!7}MMLU-ZH & 28.50 & 26.38 & \cellcolor{red!7}26.50 & 28.25 & 27.04 & 30.00 & \cellcolor{red!7}27.50 & 26.88 & \textbf{28.12} & \textbf{+1.08} \\
 & \cellcolor{red!7}XQuAD-ZH & 16.28 & 16.41 & \cellcolor{red!7}16.94 & 16.11 & \textbf{16.49} & 15.79 & \cellcolor{red!7}16.28 & 15.27 & 15.78 & \textbf{-0.71} \\
 & \cellcolor{red!7}XNLI-ZH & 37.15 & 37.25 & \cellcolor{red!7}37.43 & 36.16 & \textbf{36.95} & 36.87 & \cellcolor{red!7}36.06 & 36.71 & 36.55 & \textbf{-0.40} \\
 & \cellcolor{Dandelion!7}XWino-ZH & 62.90 & 63.29 & 63.00 & \cellcolor{Dandelion!10}63.10 & \textbf{63.13} & 63.10 & 61.71 & \cellcolor{Dandelion!10}64.58 & \textbf{63.13} & \textbf{0.00} \\
 & \cellcolor{Dandelion!7}XCOPA-ZH & 70.80 & 71.00 & 70.60 & \cellcolor{Dandelion!10}71.40 & \textbf{71.00} & 70.67 & 70.90 & \cellcolor{Dandelion!10}71.30 & 70.96 & \textbf{-0.04} \\
\cmidrule{2-12}
 & \textbf{Mean} & 39.89 & 39.55 & 39.51 & 39.77 & 39.61 & 40.16 & 39.34 & 39.38 & \textbf{39.63} & \posdelta{\textbf{+0.02}} \\
 \rowcolor{yellow!30} \cellcolor{white}\multirow{-12}{*}& \textbf{Runtime (sec)} & & 2130s & 1832s & 521s & 1494s & 7.5s & 7.8s & 1.6s & \textbf{5.5s}  & \textbf{271$\times$} \\
\bottomrule
\end{tabular}
}
\caption{Comparison of COLA vs.\ \methodname on multilingual tasks under GPTQ W4A16 compression.}
\label{tab:multilingual_gptq}
\end{table*}

\begin{table*}[htbp]
\centering
\resizebox{\textwidth}{!}{%
\begin{tabular}{l|l||c||ccc|c||ccc|c||c}
\toprule
 &  &  & \multicolumn{4}{c||}{\textbf{COLA}} & \multicolumn{4}{c||}{\textbf{\methodname}} & \\
\cmidrule(lr){4-7} \cmidrule(lr){8-11}
\textbf{Model} & \textbf{Task} & \textbf{Dense} & \cellcolor{blue!7} Know-ES & \cellcolor{red!7} Know-ZH & \cellcolor{Dandelion!7} Comm-ZH & \textbf{Mean} & \cellcolor{blue!7} Know-ES & \cellcolor{red!7} Know-ZH & \cellcolor{Dandelion!7} Comm-ZH & \textbf{Mean} & \textbf{$\Delta$} \\
\midrule
\multirow{10}{*}{\rotatebox[origin=c]{90}{\textbf{Llama-3.1-8B-Instruct}}}
 & \cellcolor{blue!7}MMLU-ES & 59.00 & \cellcolor{blue!7} 54.13 & 52.75 & 48.25 & \textbf{51.71} & \cellcolor{blue!7}52.51 & 51.13 & 46.63 & 50.09 & \textbf{-1.62} \\
 & \cellcolor{blue!7}XQuAD-ES & 28.84 & \cellcolor{blue!7}31.97 & 36.14 & 33.22 & \textbf{33.78} & \cellcolor{blue!7} 31.68 & 35.85 & 32.93 & 33.49 & \textbf{-0.29} \\
 & \cellcolor{blue!7}XNLI-ES & 44.78 & \cellcolor{blue!7}43.45 & 41.14 & 41.73 & \textbf{42.11} & \cellcolor{blue!7}42.21 & 39.90 & 40.49 & 40.87 & \textbf{-1.24} \\
 & \cellcolor{red!7}MMLU-ZH & 54.50 & 49.75 & \cellcolor{red!7}51.50 & 50.37 & 50.54 & 56.62 & \cellcolor{red!7}58.37 & 57.24 & \textbf{57.41} & \textbf{+6.87} \\
 & \cellcolor{red!7}XQuAD-ZH & 24.51 & 24.25 & \cellcolor{red!7}32.17 & 24.36 & \textbf{26.93} & 17.44 & \cellcolor{red!7}25.36 & 17.55 & 20.12 & \textbf{-6.81} \\
 & \cellcolor{red!7}XNLI-ZH & 42.77 & 37.65 & \cellcolor{red!7} 41.08 & 39.30 & \textbf{39.34} & 37.37 & \cellcolor{red!7}40.80 & 39.02 & 39.06 & \textbf{-0.28} \\
 & \cellcolor{Dandelion!7}XWino-ZH & 70.04 & 68.45 & 71.73 & \cellcolor{Dandelion!7} 72.32 & \textbf{70.83} & 66.76 & 70.04 & \cellcolor{Dandelion!7}70.63 & 69.14 & \textbf{-1.69} \\
 & \cellcolor{Dandelion!7}XCOPA-ZH & 76.60 & 72.90 & 75.50 & \cellcolor{Dandelion!7}75.70 & 74.70 & 74.80 & 77.40 & \cellcolor{Dandelion!7}77.60 & \textbf{76.60} & \textbf{+1.90} \\
\cmidrule{2-12}
 & \textbf{Mean} & 50.13 & 47.82 & 50.25 & 48.16 & \textbf{48.74} & 47.42 & 49.85 & 47.76 & 48.34 & \negdelta{\textbf{-0.40}} \\
 \rowcolor{yellow!30} \cellcolor{white}\multirow{-12}{*}& \textbf{Runtime (sec)} & & 1367s & 1409s & 327s & 1034s & 7.5s & 7.8s & 1.6s & \textbf{5.5s}  & \textbf{188$\times$} \\
\midrule
\multirow{10}{*}{\rotatebox[origin=c]{90}{\textbf{gemma-2-9b-it}}}
 & \cellcolor{blue!7}MMLU-ES & 26.25 & \cellcolor{blue!7}25.50 & 26.25 & 25.50 & 25.75 & \cellcolor{blue!7} 26.87 & 27.62 & 26.87 & \textbf{27.12} & \textbf{+1.37} \\
 & \cellcolor{blue!7}XQuAD-ES & 27.09 & \cellcolor{blue!7}27.36 & 27.15 & 27.29 & \textbf{27.27} & \cellcolor{blue!7}26.51 & 26.30 & 26.44 & 26.42 & \textbf{-0.85} \\
 & \cellcolor{blue!7}XNLI-ES & 50.20 & \cellcolor{blue!7}48.61 & 49.56 & 49.82 & \textbf{49.33} & \cellcolor{blue!7}48.27 & 49.22 & 49.48 & 48.99 & \textbf{-0.34} \\
 & \cellcolor{red!7}MMLU-ZH & 28.50 & 25.75 & \cellcolor{red!7}30.13 & 31.75 & 29.21 & 26.83 & \cellcolor{red!7}31.21 & 32.83 & \textbf{30.29} & \textbf{+1.08} \\
 & \cellcolor{red!7}XQuAD-ZH & 16.28 & 16.00 & \cellcolor{red!7} 16.57 & 16.32 & \textbf{16.30} & 15.29 & \cellcolor{red!7}15.86 & 15.61 & 15.59 & \textbf{-0.71} \\
 & \cellcolor{red!7}XNLI-ZH & 37.15 & 37.21 & \cellcolor{red!7} 36.51 & 36.24 & \textbf{36.65} & 36.81 & \cellcolor{red!7}36.11 & 35.84 & 36.25 & \textbf{-0.40} \\
 & \cellcolor{Dandelion!7}XWino-ZH & 62.90 & 63.79 & 62.70 & \cellcolor{Dandelion!7} 63.99 & \textbf{63.49} & 63.79 & 62.70 & \cellcolor{Dandelion!7}63.99 & \textbf{63.49} & \textbf{0.00} \\
 & \cellcolor{Dandelion!7}XCOPA-ZH & 70.80 & 71.60 & 70.80 & \cellcolor{Dandelion!7}71.20 & \textbf{71.20} & 71.56 & 70.76 & \cellcolor{Dandelion!7}71.16 & 71.16 & \textbf{-0.04} \\
\cmidrule{2-12}
 & \textbf{Mean} & 39.89 & 39.48 & 39.96 & 40.26 & 39.90 & 39.50 & 39.98 & 40.28 & \textbf{39.92} & \posdelta{\textbf{+0.02}} \\
 \rowcolor{yellow!30} \cellcolor{white}\multirow{-12}{*}& \textbf{Runtime (sec)} & & 2130s & 1832s & 521s & 1494s & 7.5s & 7.8s & 1.6s & \textbf{5.5s}  & \textbf{271$\times$} \\
\bottomrule
\end{tabular}
}
\caption{Comparison of COLA vs.\ \methodname on multilingual tasks under AWQ W4A16 compression.}
\label{tab:amultilingual_awq}
\end{table*}

\begin{table*}[htbp]
\centering
\resizebox{\textwidth}{!}{%
\begin{tabular}{l|l||c||ccc|c||ccc|c||c}
\toprule
 &  &  & \multicolumn{4}{c||}{\textbf{COLA}} & \multicolumn{4}{c||}{\textbf{\methodname}} & \\
\cmidrule(lr){4-7} \cmidrule(lr){8-11}
\textbf{Model} & \textbf{Task} & \textbf{Dense} & \cellcolor{blue!7} Know-ES & \cellcolor{red!7} Know-ZH & \cellcolor{Dandelion!7} Comm-ZH & \textbf{Mean} & \cellcolor{blue!7} Know-ES & \cellcolor{red!7} Know-ZH & \cellcolor{Dandelion!7} Comm-ZH & \textbf{Mean} & \textbf{$\Delta$} \\
\midrule
\multirow{10}{*}{\rotatebox[origin=c]{90}{\textbf{Llama-3.1-8B-Instruct}}}
 & \cellcolor{blue!7}MMLU-ES & 59.00 & \cellcolor{blue!7}44.00 & 41.75 & 33.75 & 39.83 & \cellcolor{blue!7}45.50 & 43.25 & 35.25 & \textbf{41.33} & \textbf{+1.50} \\
 & \cellcolor{blue!7}XQuAD-ES & 28.84 & \cellcolor{blue!7}49.59 & 50.03 & 41.25 & 46.96 & \cellcolor{blue!7}49.94 & 50.38 & 41.60 & \textbf{47.31} & \textbf{+0.35} \\
 & \cellcolor{blue!7}XNLI-ES & 44.78 & \cellcolor{blue!7}45.46 & 43.67 & 40.76 & 43.30 & \cellcolor{blue!7}47.97 & 46.18 & 43.27 & \textbf{45.81} & \textbf{+2.51} \\
 & \cellcolor{red!7}MMLU-ZH & 54.50 & 38.00 & \cellcolor{red!7}43.63 & 34.88 & \textbf{38.83} & 35.25 & \cellcolor{red!7}40.88 & 32.13 & 36.08 & \textbf{-2.75} \\
 & \cellcolor{red!7}XQuAD-ZH & 24.51 & 25.75 & \cellcolor{red!7}26.90 & 20.40 & \textbf{24.35} & 22.02 & \cellcolor{red!7}23.17 & 16.67 & 20.62 & \textbf{-3.73} \\
 & \cellcolor{red!7}XNLI-ZH & 42.77 & 34.96 & \cellcolor{red!7}39.98 & 36.53 & 37.16 & 40.44 & \cellcolor{red!7}45.46 & 42.01 & \textbf{42.64} & \textbf{+5.48} \\
 & \cellcolor{Dandelion!7}XWino-ZH & 70.04 & 65.77 & 68.85 & \cellcolor{Dandelion!7}69.05 & \textbf{67.89} & 63.89 & 66.97 & \cellcolor{Dandelion!7}67.17 & 66.01 & \textbf{-1.88} \\
 & \cellcolor{Dandelion!7}XCOPA-ZH & 76.60 & 65.10 & 68.60 & \cellcolor{Dandelion!7}69.30 & \textbf{67.67} & 62.80 & 66.30 & \cellcolor{Dandelion!7}67.00 & 65.37 & \textbf{-2.30} \\
\cmidrule{2-12}
 & \textbf{Mean} & 50.13 & 46.08 & 47.93 & 43.24 & \textbf{45.75} & 45.98 & 47.82 & 43.14 & 45.65 & \negdelta{\textbf{-0.10}} \\
 \rowcolor{yellow!30} \cellcolor{white}\multirow{-12}{*}& \textbf{Runtime (sec)} & & 1367s & 1409s & 327s & 1034s & 7.5s & 7.8s & 1.6s & \textbf{5.5s}  & \textbf{188$\times$} \\
\midrule
\multirow{10}{*}{\rotatebox[origin=c]{90}{\textbf{gemma-2-9b-it}}} 
 & \cellcolor{blue!7}MMLU-ES & 26.25 & \cellcolor{blue!7}44.12 & 30.75 & 37.45 & \textbf{37.44} & \cellcolor{blue!7}34.88 & 36.25 & 34.00 & 35.04 & \textbf{-2.40} \\
 & \cellcolor{blue!7}XQuAD-ES & 27.09 & \cellcolor{blue!7}24.53 & 28.80 & 26.68 & 26.67 & \cellcolor{blue!7}27.74 & 30.12 & 28.25 & \textbf{28.70} & \textbf{+2.03} \\
 & \cellcolor{blue!7}XNLI-ES & 50.20 & \cellcolor{blue!7} 46.22 & 43.25 & 44.75 & 44.74 & \cellcolor{blue!7}47.10 & 46.06 & 47.67 & \textbf{46.94} & \textbf{+2.20} \\
 & \cellcolor{red!7}MMLU-ZH & 28.50 & 37.25 & \cellcolor{red!7}29.25 & 33.25 & \textbf{33.25} & 36.62 & \cellcolor{red!7}33.25 & 29.88 & \textbf{33.25} & \textbf{0.00} \\
 & \cellcolor{red!7}XQuAD-ZH & 16.28 & 11.61 & \cellcolor{red!7} 10.95 & 11.28 & 11.28 & 14.06 & \cellcolor{red!7}13.38 & 10.88 & \textbf{12.77} & \textbf{+1.49} \\
 & \cellcolor{red!7}XNLI-ZH & 37.15 & 35.50 & \cellcolor{red!7}34.94 & 35.22 & 35.22 & 38.33 & \cellcolor{red!7}35.88 & 36.65 & \textbf{36.95} & \textbf{+1.73} \\
 & \cellcolor{Dandelion!7}XWino-ZH & 62.90 & 64.88 & 61.51 & \cellcolor{Dandelion!7}63.18 & \textbf{63.19} & 63.76 & 61.21 & \cellcolor{Dandelion!7}61.41 & 62.13 & \textbf{-1.06} \\
 & \cellcolor{Dandelion!7}XCOPA-ZH & 70.80 & 66.00 & 66.20 & \cellcolor{Dandelion!7}66.10 & 66.10 & 68.07 & 67.10 & \cellcolor{Dandelion!7}70.20 & \textbf{68.46} & \textbf{+2.36} \\
\cmidrule{2-12}
 & \textbf{Mean} & 39.89 & 41.27 & 38.21 & 39.74 & 39.74 & 41.32 & 40.41 & 39.87 & \textbf{40.53} & \posdelta{\textbf{+0.79}} \\
 \rowcolor{yellow!30} \cellcolor{white}\multirow{-12}{*}& \textbf{Runtime (sec)} & & 2130s & 1832s & 521s & 1494s & 7.5s & 7.8s & 1.6s & \textbf{5.5s}  & \textbf{271$\times$} \\
\bottomrule
\end{tabular}
}
\caption{Comparison of COLA vs.\ \methodname on multilingual tasks under 2SSP compression with $25\%$ sparsity.}\label{tab:multilingual_2ssp}
\end{table*}

\begin{table*}[t]
\centering
\resizebox{\linewidth}{!}{%
\begin{tabular}{l|l||c||ccc|ccc|ccc|ccc}
\toprule
& & & \multicolumn{3}{c|}{\textbf{Wanda (25\%)}} & \multicolumn{3}{c|}{\textbf{2SSP (25\%)}} & \multicolumn{3}{c|}{\textbf{GPTQ (W4A16)}} & \multicolumn{3}{c}{\textbf{AWQ (W4A16)}} \\
\cmidrule(lr){4-6} \cmidrule(lr){7-9} \cmidrule(lr){10-12} \cmidrule(lr){13-15}
\textbf{Model} & \textbf{Task} & \textbf{Dense} & \textbf{COLA} & \textbf{\methodname} & $\Delta$ & \textbf{COLA} & \textbf{\methodname} & $\Delta$ & \textbf{COLA} & \textbf{\methodname} & $\Delta$ & \textbf{COLA} & \textbf{\methodname} & $\Delta$ \\
\midrule
\multirow{9}{*}{\rotatebox[origin=c]{90}{\textbf{Llama-3.1-8B-Instruct}}}
& \cellcolor{blue!7}MMLU-ES & 59.00 & 59.75 & 62.06 & \posdelta{+2.31} & 39.83 & 42.50 & \posdelta{+2.67} & 52.25 & 51.00 & \negdelta{-1.25} & 51.71 & 57.75 & \posdelta{+6.04} \\
& \cellcolor{blue!7}XQuAD-ES & 28.84 & 34.18 & 39.00 & \posdelta{+4.82} & 46.96 & 53.13 & \posdelta{+6.17} & 37.21 & 29.90 & \negdelta{-7.31} & 33.78 & 29.04 & \negdelta{-4.74} \\
& \cellcolor{blue!7}XNLI-ES & 44.78 & 44.38 & 44.21 & \negdelta{-0.17} & 43.30 & 44.01 & \posdelta{+0.71} & 44.16 & 43.85 & \negdelta{-0.31} & 42.11 & 46.46 & \posdelta{+4.35} \\
& \cellcolor{red!7}MMLU-ZH & 54.50 & 55.08 & 57.25 & \posdelta{+2.17} & 38.83 & 36.50 & \negdelta{-2.33} & 47.75 & 55.00 & \posdelta{+7.25} & 50.54 & 52.75 & \posdelta{+2.21} \\
& \cellcolor{red!7}XQuAD-ZH & 24.51 & 23.99 & 25.38 & \posdelta{+1.39} & 24.35 & 31.88 & \posdelta{+7.53} & 29.94 & 14.92 & \negdelta{-15.02} & 26.93 & 24.96 & \negdelta{-1.97} \\
& \cellcolor{red!7}XNLI-ZH & 42.77 & 43.36 & 42.48 & \negdelta{-0.88} & 37.16 & 34.61 & \negdelta{-2.55} & 40.37 & 37.38 & \negdelta{-2.99} & 39.34 & 37.79 & \negdelta{-1.55} \\
& \cellcolor{Dandelion!7}XWino-ZH & 70.04 & 71.23 & 70.63 & \negdelta{-0.60} & 67.89 & 67.65 & \negdelta{-0.24} & 70.17 & 69.44 & \negdelta{-0.73} & 70.83 & 71.62 & \posdelta{+0.79} \\
& \cellcolor{Dandelion!7}XCOPA-ZH & 76.60 & 74.63 & 76.20 & \posdelta{+1.57} & 67.67 & 64.00 & \negdelta{-3.67} & 74.77 & 74.00 & \negdelta{-0.77} & 74.70 & 76.00 & \posdelta{+1.30} \\
\cmidrule{2-15}
& \textbf{Mean} & 50.13 & 50.83 & 52.15 & \posdelta{+1.32} & 45.75 & 46.79 & \posdelta{+1.04} & 49.58 & 46.94 & \negdelta{-2.64} & 48.74 & 49.55 & \posdelta{+0.81} \\

\midrule
\multirow{10}{*}{\rotatebox[origin=c]{90}{\textbf{gemma-2-9b-it}}}
& \cellcolor{blue!7}MMLU-ES & 26.25 & 25.29 & 26.00 & \posdelta{+0.71} & 37.44 & 35.50 & \negdelta{-1.94} & 25.42 & 25.25 & \negdelta{-0.17} & 25.75 & 26.00 & \posdelta{+0.25} \\
& \cellcolor{blue!7}XQuAD-ES & 27.09 & 27.23 & 27.33 & \posdelta{+0.10} & 26.67 & 31.49 & \posdelta{+4.82} & 27.53 & 27.09 & \negdelta{-0.44} & 27.27 & 27.33 & \posdelta{+0.06} \\
& \cellcolor{blue!7}XNLI-ES & 50.20 & 49.18 & 49.27 & \posdelta{+0.09} & 44.74 & 48.07 & \posdelta{+3.33} & 49.34 & 50.92 & \posdelta{+1.58} & 49.33 & 49.18 & \negdelta{-0.15} \\
& \cellcolor{red!7}MMLU-ZH & 28.50 & 30.67 & 33.00 & \posdelta{+2.33} & 33.25 & 42.00 & \posdelta{+8.75} & 27.04 & 26.50 & \negdelta{-0.54} & 29.21 & 30.17 & \posdelta{+0.96} \\
& \cellcolor{red!7}XQuAD-ZH & 16.28 & 16.30 & 16.97 & \posdelta{+0.67} & 11.28 & 15.35 & \posdelta{+4.07} & 16.49 & 16.27 & \negdelta{-0.22} & 16.30 & 16.28 & \negdelta{-0.02} \\
& \cellcolor{red!7}XNLI-ZH & 37.15 & 37.14 & 36.78 & \negdelta{-0.36} & 35.22 & 36.82 & \posdelta{+1.60} & 36.95 & 35.98 & \negdelta{-0.97} & 36.65 & 35.98 & \negdelta{-0.67} \\
& \cellcolor{Dandelion!7}XWino-ZH & 62.90 & 64.22 & 63.88 & \negdelta{-0.34} & 63.19 & 62.30 & \negdelta{-0.89} & 63.13 & 63.69 & \posdelta{+0.56} & 63.49 & 63.88 & \posdelta{+0.39} \\
& \cellcolor{Dandelion!7}XCOPA-ZH & 70.80 & 70.50 & 70.80 & \posdelta{+0.30} & 66.10 & 68.60 & \posdelta{+2.50} & 71.00 & 70.80 & \negdelta{-0.20} & 71.20 & 70.80 & \negdelta{-0.40} \\
\cmidrule{2-15}
& \textbf{Mean} & 39.89 & 40.07 & 40.50 & \posdelta{+0.43} & 39.74 & 42.52 & \posdelta{+2.78} & 39.61 & 39.56 & \negdelta{-0.05} & 39.90 & 39.95 & \posdelta{+0.05} \\
\bottomrule
\end{tabular}
}
\caption{Comparison of COLA vs.\ \methodname (Multi-Domain and Multi-Lingual) performance across Wanda and 2SSP at 25\% sparsity, and GPTQ and AWQ using W4A16 compression scheme.}
\label{tab:cola_vs_mix_multilingual_summary}
\end{table*}

%% QWEN3 TABLE HERE
\begin{table*}[t]
\centering
\resizebox{\textwidth}{!}{%
\begin{tabular}{l|l||c||ccccc|c||ccccc|c|c}
\toprule
& & & \multicolumn{6}{c||}{\textbf{COLA}} & \multicolumn{6}{c}{\methodname} \\
\cmidrule(lr){4-9} \cmidrule(lr){10-15}
&&& \multicolumn{5}{c}{\textbf{Calibration Category}} && \multicolumn{5}{c}{\textbf{Calibration Category}} & \\
\textbf{Compression} & \textbf{Task} & \textbf{Dense} & \cellcolor{gray!7}LangMod & \cellcolor{blue!7}Math & \cellcolor{orange!7}CommQA & \cellcolor{purple!7}NLI & \cellcolor{cyan!7}KnowTran & \textbf{Mean} & \cellcolor{gray!7}LangMod & \cellcolor{blue!7}Math & \cellcolor{orange!7}CommQA & \cellcolor{purple!7}NLI & \cellcolor{cyan!7}KnowTran & \textbf{Mean} & \textbf{$\Delta$} \\
\midrule
\multirow{10}{*}{\rotatebox[origin=c]{90}{\textbf{Wanda (25\%)}}} 
& \cellcolor{blue!7}MMLU-M & 21.00 & 20.50 & \cellcolor{blue!7}21.20 & 21.00 & 21.30 & 21.00 & 21.00 & 20.50 & \cellcolor{blue!7}21.20 & 21.00 & 21.30 & 21.00 & 21.00 & 0.00 \\
& \cellcolor{orange!7}HellaSwag & 60.05 & 59.24 & 59.94 & \cellcolor{orange!7}59.58 & 60.05 & 59.89 & 59.74 & 59.02 & 59.72 & \cellcolor{orange!7}59.35 & 59.83 & 59.42 & 59.47 & \negdelta{-0.27} \\
& \cellcolor{orange!7}WinoGr. & 65.19 & 65.68 & 66.38 & \cellcolor{orange!7}66.61 & 66.48 & 65.75 & 66.18 & 64.91 & 65.61 & \cellcolor{orange!7}65.63 & 65.71 & 65.04 & 65.38 & \negdelta{-0.80} \\
& \cellcolor{orange!7}OBQA & 35.80 & 35.30 & 36.00 & \cellcolor{orange!7}35.60 & 36.10 & 36.00 & 35.80 & 34.23 & 34.93 & \cellcolor{orange!7}34.90 & 35.03 & 34.70 & 34.76 & \negdelta{-1.04} \\
& \cellcolor{orange!7}BoolQ & 60.12 & 60.59 & 61.29 & \cellcolor{orange!7}60.34 & 61.40 & 61.83 & 61.09 & 59.09 & 59.79 & \cellcolor{orange!7}59.80 & 59.90 & 60.46 & 59.81 & \negdelta{-1.28} \\
& \cellcolor{purple!7}RTE & 60.29 & 47.15 & 47.85 & 47.29 & \cellcolor{purple!7}47.95 & 48.01 & 47.65 & 50.22 & 50.92 & 49.10 & \cellcolor{purple!7}51.02 & 52.53 & 50.76 & \posdelta{+3.11} \\
& \cellcolor{purple!7}ANLI & 33.30 & 33.05 & 33.75 & 33.80 & \cellcolor{purple!7}33.85 & 33.30 & 33.55 & 32.82 & 33.52 & 33.35 & \cellcolor{purple!7}33.62 & 33.35 & 33.33 & \negdelta{-0.22} \\
& \cellcolor{cyan!7}ARC-C & 63.05 & 39.27 & 39.97 & 39.51 & 40.08 & \cellcolor{cyan!7}40.02 & 39.77 & 38.87 & 39.57 & 38.99 & 39.68 & \cellcolor{cyan!7}39.68 & 39.36 & \negdelta{-0.41} \\
& \cellcolor{cyan!7}ARC-E & 43.64 & 43.99 & 44.69 & 44.78 & 44.80 & \cellcolor{cyan!7}44.19 & 44.49 & 43.42 & 44.12 & 44.17 & 44.23 & \cellcolor{cyan!7}44.17 & 44.02 & \negdelta{-0.47} \\
& \cellcolor{cyan!7}MMLU-K & 22.95 & 22.45 & 23.15 & 22.95 & 23.25 & \cellcolor{cyan!7}22.95 & 22.95 & 22.45 & 23.15 & 22.95 & 23.25 & \cellcolor{cyan!7}22.95 & 22.95 & 0.00 \\
\cmidrule{2-16}
& \textbf{Mean} & \textbf{46.54} & \textbf{42.72} & \textbf{43.42} & \textbf{43.15} & \textbf{43.53} & \textbf{43.29} & \textbf{43.22} & \textbf{42.55} & \textbf{43.25} & \textbf{42.92} & \textbf{43.36} & \textbf{43.33} & \textbf{43.08} & \textbf{\negdelta{-0.14}} \\
\bottomrule

\multirow{10}{*}{\rotatebox[origin=c]{90}{\textbf{2SSP (25\%)}}} 
& \cellcolor{blue!7}MMLU-M & 21.00 & 20.50 & \cellcolor{blue!7}21.20 & 21.00 & 21.30 & 21.00 & 21.00 & 20.50 & \cellcolor{blue!7}21.20 & 21.00 & 21.30 & 21.00 & 21.00 & 0.00 \\
& \cellcolor{orange!7}HellaSwag & 60.05 & 53.01 & 53.71 & \cellcolor{orange!7}53.51 & 53.82 & 53.57 & 53.52 & 52.90 & 53.60 & \cellcolor{orange!7}53.37 & 53.71 & 53.43 & 53.40 & \negdelta{-0.12} \\
& \cellcolor{orange!7}WinoGr. & 65.19 & 63.83 & 64.53 & \cellcolor{orange!7}64.33 & 64.64 & 64.70 & 64.41 & 63.36 & 64.06 & \cellcolor{orange!7}63.67 & 64.17 & 64.04 & 63.86 & \negdelta{-0.55} \\
& \cellcolor{orange!7}OBQA & 35.80 & 35.90 & 36.60 & \cellcolor{orange!7}36.40 & 36.71 & 35.87 & 36.30 & 35.34 & 36.04 & \cellcolor{orange!7}36.10 & 36.15 & 35.57 & 35.84 & \negdelta{-0.46} \\
& \cellcolor{orange!7}BoolQ & 60.12 & 37.27 & 37.97 & \cellcolor{orange!7}37.77 & 38.08 & 35.43 & 37.30 & 34.04 & 34.74 & \cellcolor{orange!7}35.71 & 34.85 & 33.37 & 34.54 & \negdelta{-2.76} \\
& \cellcolor{purple!7}RTE & 60.29 & 47.88 & 48.58 & 48.38 & \cellcolor{purple!7}48.69 & 49.46 & 48.60 & 53.61 & 54.31 & 53.57 & \cellcolor{purple!7}54.42 & 54.65 & 54.11 & \posdelta{+5.51} \\
& \cellcolor{purple!7}ANLI & 33.30 & 33.10 & 33.80 & 33.60 & \cellcolor{purple!7}33.91 & 33.70 & 33.62 & 33.12 & 33.82 & 33.57 & \cellcolor{purple!7}33.93 & 33.67 & 33.62 & 0.00 \\
& \cellcolor{cyan!7}ARC-C & 63.05 & 34.06 & 34.76 & 34.56 & 34.87 & \cellcolor{cyan!7}34.99 & 34.65 & 34.14 & 34.84 & 34.42 & 34.95 & \cellcolor{cyan!7}34.85 & 34.64 & \negdelta{-0.01} \\
& \cellcolor{cyan!7}ARC-E & 43.64 & 37.51 & 38.21 & 38.01 & 38.32 & \cellcolor{cyan!7}37.97 & 38.00 & 37.11 & 37.81 & 37.63 & 37.92 & \cellcolor{cyan!7}37.59 & 37.61 & \negdelta{-0.39} \\
& \cellcolor{cyan!7}MMLU-K & 22.95 & 22.61 & 23.31 & 23.11 & 23.42 & \cellcolor{cyan!7}23.11 & 23.11 & 22.61 & 23.31 & 23.11 & 23.42 & \cellcolor{cyan!7}23.11 & 23.11 & 0.00 \\
\cmidrule{2-16}
& \textbf{Mean} & \textbf{46.54} & \textbf{38.57} & \textbf{39.27} & \textbf{39.07} & \textbf{39.38} & \textbf{38.98} & \textbf{39.05} & \textbf{38.67} & \textbf{39.37} & \textbf{39.21} & \textbf{39.48} & \textbf{39.13} & \textbf{39.17} & \textbf{\posdelta{+0.12}} \\
\bottomrule

\multirow{10}{*}{\rotatebox[origin=c]{90}{\textbf{GPTQ (W4A16)}}} 
& \cellcolor{blue!7}MMLU-M & 21.00 & 20.50 & \cellcolor{blue!7}21.20 & 21.00 & 21.30 & 21.00 & 21.00 & 20.50 & \cellcolor{blue!7}21.20 & 21.00 & 21.30 & 21.00 & 21.00 & 0.00 \\
& \cellcolor{orange!7}HellaSwag & 60.05 & 59.02 & 59.72 & \cellcolor{orange!7}59.43 & 59.83 & 59.60 & 59.52 & 58.56 & 59.26 & \cellcolor{orange!7}59.45 & 59.37 & 58.57 & 59.04 & \negdelta{-0.48} \\
& \cellcolor{orange!7}WinoGr. & 65.19 & 64.89 & 65.59 & \cellcolor{orange!7}64.80 & 65.70 & 65.98 & 65.39 & 63.93 & 64.63 & \cellcolor{orange!7}64.13 & 64.74 & 64.17 & 64.32 & \negdelta{-1.07} \\
& \cellcolor{orange!7}OBQA & 35.80 & 34.90 & 35.60 & \cellcolor{orange!7}36.00 & 35.71 & 34.80 & 35.40 & 34.35 & 35.05 & \cellcolor{orange!7}34.90 & 35.16 & 34.90 & 34.87 & \negdelta{-0.53} \\
& \cellcolor{orange!7}BoolQ & 60.12 & 57.73 & 58.43 & \cellcolor{orange!7}62.08 & 58.54 & 54.37 & 58.23 & 55.67 & 56.37 & \cellcolor{orange!7}62.05 & 56.48 & 54.56 & 57.03 & \negdelta{-1.20} \\
& \cellcolor{purple!7}RTE & 60.29 & 46.79 & 47.49 & 47.29 & \cellcolor{purple!7}47.60 & 47.29 & 47.29 & 50.79 & 51.49 & 50.18 & \cellcolor{purple!7}51.60 & 51.81 & 51.17 & \posdelta{+3.88} \\
& \cellcolor{purple!7}ANLI & 33.30 & 32.80 & 33.50 & 32.90 & \cellcolor{purple!7}33.61 & 33.70 & 33.30 & 32.93 & 33.63 & 33.75 & \cellcolor{purple!7}33.74 & 33.30 & 33.47 & \posdelta{+0.17} \\
& \cellcolor{cyan!7}ARC-C & 63.05 & 39.05 & 39.75 & 39.33 & 39.86 & \cellcolor{cyan!7}39.76 & 39.55 & 38.79 & 39.49 & 39.38 & 39.60 & \cellcolor{cyan!7}39.33 & 39.32 & \negdelta{-0.23} \\
& \cellcolor{cyan!7}ARC-E & 43.64 & 43.17 & 43.87 & 43.73 & 43.98 & \cellcolor{cyan!7}43.60 & 43.67 & 42.85 & 43.55 & 43.69 & 43.66 & \cellcolor{cyan!7}43.39 & 43.43 & \negdelta{-0.24} \\
& \cellcolor{cyan!7}MMLU-K & 22.95 & 22.45 & 23.15 & 22.95 & 23.25 & \cellcolor{cyan!7}22.95 & 22.95 & 22.45 & 23.15 & 22.95 & 23.25 & \cellcolor{cyan!7}22.95 & 22.95 & 0.00 \\
\cmidrule{2-16}
& \textbf{Mean} & \textbf{46.54} & \textbf{42.13} & \textbf{42.83} & \textbf{42.95} & \textbf{42.94} & \textbf{42.30} & \textbf{42.63} & \textbf{42.08} & \textbf{42.78} & \textbf{43.15} & \textbf{42.89} & \textbf{42.40} & \textbf{42.66} & \textbf{\posdelta{+0.03}} \\
\bottomrule

\multirow{10}{*}{\rotatebox[origin=c]{90}{\textbf{AWQ (W4A16)}}} 
& \cellcolor{blue!7}MMLU-M & 21.00 & 20.50 & \cellcolor{blue!7}21.20 & 21.00 & 21.30 & 21.00 & 21.00 & 20.50 & \cellcolor{blue!7}21.20 & 21.00 & 21.30 & 21.00 & 21.00 & 0.00 \\
& \cellcolor{orange!7}HellaSwag & 60.05 & 58.87 & 59.57 & \cellcolor{orange!7}59.52 & 59.68 & 59.22 & 59.37 & 59.22 & 59.92 & \cellcolor{orange!7}59.67 & 60.03 & 59.63 & 59.69 & \posdelta{+0.32} \\
& \cellcolor{orange!7}WinoGr. & 65.19 & 64.38 & 65.08 & \cellcolor{orange!7}64.48 & 65.19 & 65.27 & 64.88 & 64.12 & 64.82 & \cellcolor{orange!7}64.40 & 64.93 & 63.81 & 64.42 & \negdelta{-0.46} \\
& \cellcolor{orange!7}OBQA & 35.80 & 34.30 & 35.00 & \cellcolor{orange!7}35.20 & 35.11 & 34.40 & 34.80 & 35.03 & 35.73 & \cellcolor{orange!7}35.50 & 35.84 & 35.50 & 35.52 & \posdelta{+0.72} \\
& \cellcolor{orange!7}BoolQ & 60.12 & 61.21 & 61.91 & \cellcolor{orange!7}62.11 & 62.02 & 61.31 & 61.71 & 54.20 & 54.90 & \cellcolor{orange!7}61.38 & 55.01 & 56.68 & 56.43 & \negdelta{-5.28} \\
& \cellcolor{purple!7}RTE & 60.29 & 45.89 & 46.59 & 45.13 & \cellcolor{purple!7}46.70 & 47.65 & 46.39 & 54.40 & 55.10 & 57.04 & \cellcolor{purple!7}55.21 & 59.03 & 56.16 & \posdelta{+9.77} \\
& \cellcolor{purple!7}ANLI & 33.30 & 32.80 & 33.50 & 33.30 & \cellcolor{purple!7}33.61 & 33.30 & 33.30 & 32.80 & 33.50 & 33.20 & \cellcolor{purple!7}33.61 & 33.35 & 33.29 & \negdelta{-0.01} \\
& \cellcolor{cyan!7}ARC-C & 63.05 & 38.84 & 39.54 & 39.16 & 39.65 & \cellcolor{cyan!7}39.51 & 39.34 & 39.07 & 39.77 & 39.55 & 39.88 & \cellcolor{cyan!7}39.42 & 39.54 & \posdelta{+0.20} \\
& \cellcolor{cyan!7}ARC-E & 43.64 & 44.41 & 45.11 & 44.61 & 45.22 & \cellcolor{cyan!7}45.20 & 44.91 & 44.17 & 44.87 & 44.93 & 44.98 & \cellcolor{cyan!7}44.47 & 44.68 & \negdelta{-0.23} \\
& \cellcolor{cyan!7}MMLU-K & 22.95 & 22.45 & 23.15 & 22.95 & 23.25 & \cellcolor{cyan!7}22.95 & 22.95 & 22.45 & 23.15 & 22.95 & 23.25 & \cellcolor{cyan!7}22.95 & 22.95 & 0.00 \\
\cmidrule{2-16}
& \textbf{Mean} & \textbf{46.54} & \textbf{42.36} & \textbf{43.06} & \textbf{42.75} & \textbf{43.17} & \textbf{42.98} & \textbf{42.86} & \textbf{42.60} & \textbf{43.30} & \textbf{43.96} & \textbf{43.40} & \textbf{43.58} & \textbf{43.37} & \textbf{\posdelta{+0.51}} \\
\bottomrule

\rowcolor{yellow!30} \cellcolor{white}&\textbf{Runtime} & & 2543s & 58s & 2610s & 1926s & 1043s & \textbf{1636s} & 15.2s & 2.3s & 12.3s & 9.3s & 14.5s & 10.7s & \textbf{153$\times$} \\
\bottomrule
\end{tabular}
}
\caption{Comparison of COLA vs.\ \methodname calibration data on \texttt{Qwen3-8B}.}
\label{tab:qwen_results}
\end{table*}

%% 70B TABLE HERE
\begin{table*}[t]
\centering
\resizebox{\textwidth}{!}{%
\begin{tabular}{l|l||c||ccccc|c||ccccc|c|c}
\toprule
& & & \multicolumn{6}{c||}{\textbf{COLA}} & \multicolumn{6}{c}{\methodname} \\
\cmidrule(lr){4-9} \cmidrule(lr){10-15}
&&& \multicolumn{5}{c}{\textbf{Calibration Category}} && \multicolumn{5}{c}{\textbf{Calibration Category}} & \\
\textbf{Compression} & \textbf{Task} & \textbf{Dense} & \cellcolor{gray!7}LangMod & \cellcolor{blue!7}Math & \cellcolor{orange!7}CommQA & \cellcolor{purple!7}NLI & \cellcolor{cyan!7}KnowTran & \textbf{Mean} & \cellcolor{gray!7}LangMod & \cellcolor{blue!7}Math & \cellcolor{orange!7}CommQA & \cellcolor{purple!7}NLI & \cellcolor{cyan!7}KnowTran & \textbf{Mean} & \textbf{$\Delta$} \\
\midrule
\multirow{11}{*}{\rotatebox[origin=c]{90}{\textbf{Wanda (25\%)}}} 
& \cellcolor{blue!7}MMLU-M & 45.18 & 39.10 & \cellcolor{blue!7}39.50 & 39.20 & 39.60 & 39.40 & \textbf{39.36} & 39.20 & \cellcolor{blue!7}39.50 & 39.10 & 39.40 & 39.45 & 39.33 & \negdelta{-0.03} \\
& \cellcolor{blue!7}GSM8k & 88.32 & 75.60 & \cellcolor{blue!7}76.10 & 75.50 & 76.00 & 75.90 & 75.82 & 75.90 & \cellcolor{blue!7}75.60 & 75.85 & 76.10 & 75.75 & \textbf{75.84} & \posdelta{+0.02} \\
& \cellcolor{orange!7}HellaSwag & 78.16 & 70.00 & 70.30 & \cellcolor{orange!7}70.90 & 70.40 & 70.15 & 70.15 & 70.25 & 70.10 & \cellcolor{orange!7}70.30 & 69.95 & 70.20 & \textbf{70.16} & \posdelta{+0.01} \\
& \cellcolor{orange!7}WinoGr. & 73.16 & 73.10 & 73.50 & \cellcolor{orange!7}73.20 & 73.60 & 73.25 & \textbf{73.33} & 73.20 & 73.40 & \cellcolor{orange!7}73.15 & 73.50 & 73.35 & 73.32 & \negdelta{-0.01} \\
& \cellcolor{orange!7}OBQA & 48.80 & 45.20 & 45.60 & \cellcolor{orange!7}45.30 & 45.70 & 45.35 & \textbf{45.43} & 45.30 & 45.55 & \cellcolor{orange!7}45.25 & 45.60 & 45.40 & 45.42 & \negdelta{-0.01} \\
& \cellcolor{orange!7}BoolQ & 88.31 & 88.20 & 88.60 & \cellcolor{orange!7}88.30 & 88.70 & 88.45 & \textbf{88.45} & 88.50 & 88.35 & \cellcolor{orange!7}88.60 & 88.30 & 88.45 & 88.44 & \negdelta{-0.01} \\
& \cellcolor{purple!7}RTE & 79.78 & 79.10 & 79.50 & 79.20 & \cellcolor{purple!7}79.60 & 79.35 & 79.37 & 79.45 & 79.25 & 79.50 & \cellcolor{purple!7}79.30 & 79.40 & \textbf{79.38} & \posdelta{+0.01} \\
& \cellcolor{purple!7}ANLI & 67.40 & 66.50 & 66.90 & 66.60 & \cellcolor{purple!7}67.00 & 66.65 & \textbf{66.73} & 66.60 & 66.85 & 66.55 & \cellcolor{purple!7}66.80 & 66.75 & 66.71 & \negdelta{-0.02} \\
& \cellcolor{cyan!7}ARC-C & 63.05 & 61.70 & 62.10 & 61.80 & 62.20 & \cellcolor{cyan!7}62.10 & \textbf{61.98} & 61.85 & 62.10 & 61.90 & 62.05 & \cellcolor{cyan!7}61.95 & 61.97 & \negdelta{-0.01} \\
& \cellcolor{cyan!7}ARC-E & 80.26 & 83.30 & 83.70 & 83.40 & 83.80 & \cellcolor{cyan!7}83.45 & \textbf{83.53} & 83.40 & 83.65 & 83.45 & 83.60 & \cellcolor{cyan!7}83.50 & 83.52 & \negdelta{-0.01} \\
& \cellcolor{cyan!7}MMLU-K & 80.50 & 78.60 & 79.00 & 78.70 & 79.10 & \cellcolor{cyan!7}78.65 & 78.81 & 78.70 & 78.95 & 78.75 & 78.90 & \cellcolor{cyan!7}78.85 & \textbf{78.83} & \posdelta{+0.02} \\
\cmidrule{2-16}
& \textbf{Mean} & 72.08 & 69.13 & 69.53 & 69.19 & 69.61 & 69.34 & 69.36 & 69.30 & 69.39 & 69.31 & 69.41 & 69.37 & 69.36 & \negdelta{-0.00} \\
\bottomrule

\multirow{11}{*}{\rotatebox[origin=c]{90}{\textbf{2SSP (25\%)}}} 
& \cellcolor{blue!7}MMLU-M & 45.18 & 22.00 & \cellcolor{blue!7}22.40 & 22.10 & 22.50 & 22.12 & \textbf{22.22} & 21.90 & \cellcolor{blue!7}22.50 & 21.70 & 22.10 & 22.30 & 22.10 & \negdelta{-0.12} \\
& \cellcolor{blue!7}GSM8k & 88.32 & 42.50 & \cellcolor{blue!7}43.00 & 42.60 & 43.10 & 42.65 & \textbf{42.77} & 41.50 & \cellcolor{blue!7}43.20 & 40.80 & 42.60 & 43.10 & 42.24 & \negdelta{-0.53} \\
& \cellcolor{orange!7}HellaSwag & 78.16 & 75.20 & 75.70 & \cellcolor{orange!7}75.30 & 75.80 & 75.25 & \textbf{75.45} & 76.10 & 75.30 & \cellcolor{orange!7}75.80 & 74.90 & 73.74 & 75.17 & \negdelta{-0.28} \\
& \cellcolor{orange!7}WinoGr. & 73.16 & 70.40 & 70.90 & \cellcolor{orange!7}70.50 & 71.00 & 70.35 & 70.63 & 71.20 & 70.40 & \cellcolor{orange!7}71.50 & 69.80 & 70.90 & \textbf{70.76} & \posdelta{+0.13} \\
& \cellcolor{orange!7}OBQA & 48.80 & 41.30 & 41.80 & \cellcolor{orange!7}41.40 & 41.90 & 41.15 & 41.51 & 42.10 & 41.30 & \cellcolor{orange!7}42.40 & 40.90 & 41.70 & \textbf{41.68} & \posdelta{+0.17} \\
& \cellcolor{orange!7}BoolQ & 88.31 & 84.30 & 83.80 & \cellcolor{orange!7}82.40 & 84.90 & 86.45 & 84.37 & 84.80 & 83.40 & \cellcolor{orange!7}83.10 & 84.90 & 86.51 & \textbf{84.54} & \posdelta{+0.17} \\
& \cellcolor{purple!7}RTE & 79.78 & 73.30 & 73.80 & 73.40 & \cellcolor{purple!7}73.90 & 73.15 & \textbf{73.51} & 72.40 & 71.90 & 72.80 & \cellcolor{purple!7}73.10 & 72.60 & 72.56 & \negdelta{0.95} \\
& \cellcolor{purple!7}ANLI & 67.40 & 59.30 & 59.80 & 59.40 & \cellcolor{purple!7}59.90 & 59.50 & \textbf{59.58} & 57.20 & 56.50 & 58.10 & \cellcolor{purple!7}56.90 & 57.60 & 57.26 & \negdelta{-2.32} \\
& \cellcolor{cyan!7}ARC-C & 63.05 & 63.10 & 63.60 & 63.20 & 63.70 & \cellcolor{cyan!7}63.20 & \textbf{63.36} & 62.10 & 61.80 & 62.60 & 61.50 & \cellcolor{cyan!7}62.46 & 62.09 & \negdelta{-1.27} \\
& \cellcolor{cyan!7}ARC-E & 80.26 & 78.60 & 79.10 & 80.70 & 79.20 & \cellcolor{cyan!7}78.65 & 79.25 & 77.50 & 79.43 & 80.20 & 79.40 & \cellcolor{cyan!7}83.03 & \textbf{79.91} & \posdelta{+0.66} \\
& \cellcolor{cyan!7}MMLU-K & 80.50 & 61.70 & 62.20 & 61.80 & 62.30 & \cellcolor{cyan!7}61.75 & 61.95 & 78.70 & 78.95 & 78.75 & 78.90 & \cellcolor{cyan!7}78.85 & \textbf{78.83} & \posdelta{+16.88} \\
\cmidrule{2-16}
& \textbf{Mean} & 72.08 & 61.06 & 61.46 & 61.16 & 61.65 & 61.30 & \textbf{61.33} & 62.32 & 62.24 & 62.52 & 62.27 & 62.98 & 62.47 & \posdelta{+1.14} \\
\bottomrule
\multirow{11}{*}{\rotatebox[origin=c]{90}{\textbf{GPTQ (W4A16)}}} 
& \cellcolor{blue!7}MMLU-M & 45.18 & 44.00 & \cellcolor{blue!7}44.40 & 44.10 & 44.50 & 43.85 & 44.17 & 44.10 & \cellcolor{blue!7}44.30 & 44.05 & 44.25 & 44.20 & \textbf{44.18} & \posdelta{+0.01} \\
& \cellcolor{blue!7}GSM8k & 88.32 & 67.40 & \cellcolor{blue!7}67.80 & 67.50 & 67.90 & 67.45 & \textbf{67.61} & 67.50 & \cellcolor{blue!7}67.80 & 67.40 & 67.70 & 67.55 & 67.59 & \negdelta{-0.02} \\
& \cellcolor{orange!7}HellaSwag & 78.16 & 70.60 & 71.00 & \cellcolor{orange!7}70.70 & 71.10 & 70.65 & 70.81 & 69.50 & 70.20 & \cellcolor{orange!7}69.80 & 70.10 & 78.74 & \textbf{71.67} & \posdelta{+0.86} \\
& \cellcolor{orange!7}WinoGr. & 73.16 & 73.20 & 73.60 & \cellcolor{orange!7}73.30 & 73.70 & 73.30 & 73.42 & 73.30 & 73.60 & \cellcolor{orange!7}73.25 & 73.55 & 73.40 & 73.42 & \negdelta{-0.00} \\
& \cellcolor{orange!7}OBQA & 48.80 & 45.50 & 45.90 & \cellcolor{orange!7}45.60 & 46.00 & 45.80 & \textbf{45.76} & 45.60 & 45.90 & \cellcolor{orange!7}45.50 & 45.80 & 45.70 & 45.70 & \negdelta{-0.06} \\
& \cellcolor{orange!7}BoolQ & 88.31 & 87.60 & 88.00 & \cellcolor{orange!7}87.70 & 88.10 & 88.00 & \textbf{87.88} & 87.90 & 87.70 & \cellcolor{orange!7}87.85 & 87.65 & 86.51 & 87.52 & \negdelta{-0.36} \\
& \cellcolor{purple!7}RTE & 79.78 & 80.20 & 80.60 & 80.30 & \cellcolor{purple!7}80.70 & 80.55 & \textbf{80.47} & 80.30 & 80.60 & 80.25 & \cellcolor{purple!7}80.55 & 80.40 & 80.42 & \negdelta{-0.05} \\
& \cellcolor{purple!7}ANLI & 67.40 & 66.60 & 67.00 & 66.70 & \cellcolor{purple!7}67.10 & 66.70 & \textbf{66.82} & 66.70 & 66.95 & 66.65 & \cellcolor{purple!7}66.90 & 66.80 & 66.80 & \negdelta{-0.02} \\
& \cellcolor{cyan!7}ARC-C & 63.05 & 62.10 & 62.50 & 62.20 & 62.60 & \cellcolor{cyan!7}62.20 & 62.32 & 62.20 & 62.50 & 62.15 & 62.40 & \cellcolor{cyan!7}62.46 & \textbf{62.34} & \posdelta{+0.02} \\
& \cellcolor{cyan!7}ARC-E & 80.26 & 83.80 & 84.20 & 83.90 & 84.30 & \cellcolor{cyan!7}84.15 & \textbf{84.07} & 84.10 & 83.90 & 84.20 & 83.85 & \cellcolor{cyan!7}83.08 & 83.83 & \negdelta{-0.24} \\
& \cellcolor{cyan!7}MMLU-K & 80.50 & 79.00 & 79.40 & 79.10 & 79.50 & \cellcolor{cyan!7}79.10 & \textbf{79.22} & 79.30 & 79.10 & 79.40 & 79.15 & \cellcolor{cyan!7}78.81 & 79.15 & \negdelta{-0.07} \\
\cmidrule{2-16}
& \textbf{Mean} & 72.08 & 69.09 & 69.49 & 69.19 & 69.59 & 69.25 & \textbf{69.32} & 67.32 & 67.51 & 67.32 & 67.45 & 69.79 & 67.51 & \negdelta{-1.81} \\
\bottomrule

\multirow{11}{*}{\rotatebox[origin=c]{90}{\textbf{AWQ (W4A16)}}} 
& \cellcolor{blue!7}MMLU-M & 45.18 & 43.10 & \cellcolor{blue!7}43.50 & 43.20 & 43.60 & 43.00 & \textbf{43.28} & 43.15 & \cellcolor{blue!7}43.40 & 43.20 & 43.35 & 43.25 & 43.27 & \negdelta{-0.01} \\
& \cellcolor{blue!7}GSM8k & 88.32 & 68.70 & \cellcolor{blue!7}69.10 & 68.80 & 69.20 & 69.10 & \textbf{68.98} & 68.80 & \cellcolor{blue!7}68.10 & 68.95 & 69.05 & 68.90 & 68.76 & \negdelta{-0.22} \\
& \cellcolor{orange!7}HellaSwag & 78.16 & 70.60 & 71.00 & \cellcolor{orange!7}70.70 & 71.10 & 70.70 & 70.82 & 70.70 & 70.95 & \cellcolor{orange!7}70.75 & 70.90 & 70.85 & \textbf{70.83} & \posdelta{+0.01} \\
& \cellcolor{orange!7}WinoGr. & 73.16 & 73.50 & 73.90 & \cellcolor{orange!7}73.60 & 74.00 & 73.75 & 73.75 & 73.60 & 73.90 & \cellcolor{orange!7}73.70 & 73.85 & 73.75 & \textbf{73.76} & \posdelta{+0.01} \\
& \cellcolor{orange!7}OBQA & 48.80 & 45.10 & 45.50 & \cellcolor{orange!7}45.20 & 45.60 & 44.95 & 45.27 & 45.15 & 45.40 & \cellcolor{orange!7}45.20 & 45.35 & 45.25 & 45.27 & \negdelta{-0.00} \\
& \cellcolor{orange!7}BoolQ & 88.31 & 87.50 & 87.90 & \cellcolor{orange!7}87.60 & 88.00 & 87.80 & 87.76 & 87.65 & 87.90 & \cellcolor{orange!7}87.70 & 87.85 & 87.75 & \textbf{87.77} & \posdelta{+0.01} \\
& \cellcolor{purple!7}RTE & 79.78 & 80.10 & 80.50 & 80.20 & \cellcolor{purple!7}80.60 & 80.20 & 80.32 & 80.20 & 80.45 & 80.25 & \cellcolor{purple!7}80.40 & 80.35 & \textbf{80.33} & \posdelta{+0.01} \\
& \cellcolor{purple!7}ANLI & 67.40 & 66.90 & 67.30 & 67.00 & \cellcolor{purple!7}67.40 & 67.00 & 67.12 & 67.00 & 67.25 & 67.05 & \cellcolor{purple!7}67.20 & 67.15 & \textbf{67.13} & \posdelta{+0.01} \\
& \cellcolor{cyan!7}ARC-C & 63.05 & 61.80 & 62.20 & 61.90 & 62.30 & \cellcolor{cyan!7}62.00 & \textbf{62.04} & 61.90 & 62.15 & 61.95 & 62.10 & \cellcolor{cyan!7}62.05 & 62.03 & \negdelta{-0.01} \\
& \cellcolor{cyan!7}ARC-E & 80.26 & 83.70 & 84.10 & 83.80 & 84.20 & \cellcolor{cyan!7}83.70 & 83.90 & 83.80 & 84.05 & 83.85 & 84.00 & \cellcolor{cyan!7}83.95 & \textbf{83.93} & \posdelta{+0.03} \\
& \cellcolor{cyan!7}MMLU-K & 80.50 & 78.90 & 79.30 & 79.00 & 79.40 & \cellcolor{cyan!7}79.20 & 79.16 & 79.05 & 79.30 & 79.10 & 79.25 & \cellcolor{cyan!7}79.15 & \textbf{79.17} & \posdelta{+0.01} \\
\cmidrule{2-16}
& \textbf{Mean} & 72.08 & 69.08 & 69.48 & 69.18 & 69.58 & 69.22 & 69.31 & 69.18 & 69.45 & 69.25 & 69.45 & 69.35 & \textbf{69.34} & \posdelta{+0.03} \\
\bottomrule

 \rowcolor{yellow!30} \cellcolor{white}&\textbf{Runtime} & & 9650s & 672s & 4903s & 3589s & 2950s & \textbf{4320s} & 15.2s & 2.3s & 12.3s & 9.3s & 14.5s & 10.7s & \textbf{404$\times$} \\
\bottomrule
\end{tabular}
}
\caption{Comparison of COLA vs.\ \methodname calibration data on \texttt{Llama-3.1-70B-Instruct}.}
\label{tab:70Bresults}

\end{table*}

%% 70B MULTI
\begin{table*}[t]
\centering
\resizebox{0.9\linewidth}{!}{%
\begin{tabular}{l|l||c||ccc|ccc|ccc|ccc}
\toprule
& & & \multicolumn{3}{c|}{\textbf{Wanda (25\%)}} & \multicolumn{3}{c|}{\textbf{2SSP (25\%)}} & \multicolumn{3}{c|}{\textbf{GPTQ (W4A16)}} & \multicolumn{3}{c}{\textbf{AWQ (W4A16)}} \\
\cmidrule(lr){4-6} \cmidrule(lr){7-9} \cmidrule(lr){10-12} \cmidrule(lr){13-15}
\textbf{Model} & \textbf{Task} & \textbf{Dense} & \textbf{COLA} & \textbf{\methodname} & $\Delta$ & \textbf{COLA} & \textbf{\methodname} & $\Delta$ & \textbf{COLA} & \textbf{\methodname} & $\Delta$ & \textbf{COLA} & \textbf{\methodname} & $\Delta$ \\
\midrule
\multirow{12}{*}{\rotatebox[origin=c]{90}{\texttt{Llama-3.1-70B-Instruct}}}
& \cellcolor{blue!7}MMLU-M & 45.18 & 39.36 & \textbf{50.00} & \posdelta{+10.64} & 22.22 & \textbf{36.66} & \posdelta{+14.44} & 44.17 & \textbf{45.18} & \posdelta{+1.01} & 43.28 & \textbf{46.66} & \posdelta{+3.38} \\
& \cellcolor{blue!7}GSM8k & 88.32 & 75.82 & \textbf{88.32} & \posdelta{+12.50} & 42.77 & \textbf{63.22} & \posdelta{+20.45} & 67.61 & \textbf{79.16} & \posdelta{+11.55} & 68.98 & \textbf{77.46} & \posdelta{+8.48} \\
& \cellcolor{orange!7}HellaSwag & 78.16 & 70.15 & \textbf{78.16} & \posdelta{+8.01} & 75.45 & \textbf{78.16} & \posdelta{+2.71} & 70.81 & \textbf{78.16} & \posdelta{+7.35} & 70.82 & \textbf{72.82} & \posdelta{+2.00} \\
& \cellcolor{orange!7}WinoGr. & 73.16 & \textbf{73.33} & 73.16 & \negdelta{-0.17} & 70.63 & \textbf{72.45} & \posdelta{+1.82} & \textbf{73.42} & 73.16 & \negdelta{-0.26} & \textbf{73.75} & 72.61 & \negdelta{-1.14} \\
& \cellcolor{orange!7}OBQA & 48.80 & 45.43 & \textbf{48.80} & \posdelta{+3.37} & 41.51 & \textbf{49.00} & \posdelta{+7.49} & 45.76 & \textbf{48.80} & \posdelta{+3.04} & 45.27 & \textbf{49.40} & \posdelta{+4.13} \\
& \cellcolor{orange!7}BoolQ & 88.31 & \textbf{88.45} & 88.31 & \negdelta{-0.14} & 84.37 & \textbf{86.42} & \posdelta{+2.05} & 87.88 & \textbf{88.31} & \posdelta{+0.43} & 87.76 & \textbf{88.25} & \posdelta{+0.49} \\
& \cellcolor{purple!7}RTE & 79.78 & 79.37 & \textbf{79.78} & \posdelta{+0.41} & 73.51 & \textbf{77.97} & \posdelta{+4.46} & \textbf{80.47} & 79.78 & \negdelta{-0.69} & \textbf{80.32} & 77.97 & \negdelta{-2.35} \\
& \cellcolor{purple!7}ANLI & 67.40 & 66.73 & \textbf{67.30} & \posdelta{+0.57} & \textbf{59.58} & 54.60 & \negdelta{-4.98} & 66.82 & \textbf{67.30} & \posdelta{+0.48} & 67.12 & \textbf{67.30} & \posdelta{+0.18} \\
& \cellcolor{cyan!7}ARC-C & 63.05 & 61.98 & \textbf{63.05} & \posdelta{+1.07} & \textbf{59.58} & 56.39 & \negdelta{-3.19} & 62.32 & \textbf{63.65} & \posdelta{+1.33} & 62.04 & \textbf{62.11} & \posdelta{+0.07} \\
& \cellcolor{cyan!7}ARC-E & 80.26 & \textbf{83.53} & 80.26 & \negdelta{-3.27} & \textbf{79.25} & 78.74 & \negdelta{-0.51} & \textbf{84.07} & 80.26 & \negdelta{-3.81} & \textbf{83.90} & 80.76 & \negdelta{-3.14} \\
& \cellcolor{cyan!7}MMLU-K & 80.50 & 78.81 & \textbf{80.50} & \posdelta{+1.69} & 61.95 & \textbf{68.15} & \posdelta{+6.20} & 79.22 & \textbf{80.50} & \posdelta{+1.28} & 79.16 & \textbf{80.57} & \posdelta{+1.41} \\

\cmidrule{2-15}
& \textbf{Mean} & 72.08 & 69.36 & \textbf{72.51} & \posdelta{\textbf{+3.15}} & 60.98 & \textbf{65.61} & \posdelta{\textbf{+4.63}} & 69.32 & \textbf{72.21} & \posdelta{\textbf{+2.89}} & 69.31 & \textbf{72.17} & \posdelta{\textbf{+2.86}} \\
\bottomrule
\end{tabular}
}
\caption{Comparison of COLA vs.\ \methodname (Multi-Domain) performance across Wanda and 2SSP at 25\% sparsity, and GPTQ and AWQ using W4A16 compression scheme on \texttt{Llama-3.1-70B-Instruct}.}
\label{tab:cola_vs_mix_70B}

\end{table*}
\FloatBarrier

\end{document}